\documentclass[sigconf,natbib=true,screen]{acmart}
\usepackage{xspace,balance,tabularx,multirow}
\usepackage{flushend}
\usepackage{tikz}
\usepackage{pgfplots}
\pgfplotsset{compat=1.16}
\usetikzlibrary{patterns}
\usepackage{subfig}
\usepackage[ruled, vlined, linesnumbered]{algorithm2e}
\usepackage{colortbl}
\usepackage{bbold}
\SetKwComment{Comment}{$\triangleright$\ }{}
\usepackage{enumitem}
\usepackage{tablefootnote}
\usepackage{upgreek,textgreek}
\usepackage{pifont}
\usepackage[noabbrev]{cleveref}
\usepackage{titlecaps}
\usepackage{lipsum}
\usepackage{multirow}
\pgfplotsset{every tick label/.append style={font=\tiny}}

\newlength{\starsize}
\newlength{\starspread}
\tikzset{starsize/.code={\setlength{\starsize}{#1}},
         starspread/.code={\setlength{\starspread}{#1}}}
\tikzset{starsize=1mm,
         starspread=3mm}
\pgfdeclarepatternformonly[\starspread,\starsize]
  {my fivepointed stars}
  {\pgfpointorigin}
  {\pgfqpoint{\starspread}{\starspread}}
  {\pgfqpoint{\starspread}{\starspread}}
  {
   \pgftransformshift{\pgfqpoint{\starsize}{\starsize}}
   \pgfpathmoveto{\pgfqpointpolar{18}{\starsize}}
   \pgfpathlineto{\pgfqpointpolar{162}{\starsize}}
   \pgfpathlineto{\pgfqpointpolar{306}{\starsize}}
   \pgfpathlineto{\pgfqpointpolar{90}{\starsize}}
   \pgfpathlineto{\pgfqpointpolar{234}{\starsize}}
   \pgfpathclose%
   \pgfusepath{fill}
  }

\newcommand{\renchi}[1]{{\color{red}{[Renchi: #1]}}}

\makeatletter
\newcommand*\bigcdot{\mathpalette\bigcdot@{.5}}
\newcommand*\bigcdot@[2]{\mathbin{\vcenter{\hbox{\scalebox{#2}{$\m@th#1\bullet$}}}}}
\makeatother

\newcommand{\stitle}[1]{\vspace*{0.5em}\noindent{\underline{\bf #1.\/}}}

\newcommand{\V}{\mathcal{V}\xspace}
\newcommand{\G}{\mathcal{G}\xspace}

\newcommand{\EDG}{\mathcal{E}\xspace}
\newcommand{\PP}{\mathcal{P}\xspace}
\newcommand{\LL}{\mathcal{L}\xspace}
\newcommand{\NN}{\mathcal{N}\xspace}
\newcommand{\C}{\mathcal{C}\xspace}

\newcommand{\WM}{\boldsymbol{W}\xspace}
\newcommand{\EM}{\boldsymbol{E}\xspace}
\newcommand{\AM}{\boldsymbol{A}\xspace}
\newcommand{\DM}{\boldsymbol{D}\xspace}
\newcommand{\IM}{\boldsymbol{I}\xspace}
\newcommand{\SM}{\boldsymbol{S}\xspace}
\newcommand{\MM}{\boldsymbol{M}\xspace}

\newcommand{\PM}{\boldsymbol{P}\xspace}
\newcommand{\CM}{\boldsymbol{C}\xspace}
\newcommand{\YM}{\boldsymbol{Y}\xspace}
\newcommand{\XM}{\boldsymbol{X}\xspace}

\newcommand{\UM}{\boldsymbol{U}\xspace}
\newcommand{\VM}{\boldsymbol{V}\xspace}
\newcommand{\HM}{\boldsymbol{H}\xspace}
\newcommand{\ZM}{\boldsymbol{Z}\xspace}
\newcommand{\RM}{\boldsymbol{R}\xspace}

\newcommand{\NAM}{\boldsymbol{\hat{A}}\xspace}

\newcommand{\BM}{\boldsymbol{B}\xspace}

\newcommand{\NR}{\mathbb{N}_{\text{row}}\xspace}
\newcommand{\NC}{\mathbb{N}_{\text{col}}\xspace}

\newcommand{\QM}{\boldsymbol{Q}\xspace}

\newcommand{\algo}{\texttt{DEMM}\xspace}
\newcommand{\algoplus}{\texttt{DEMM}{+}\xspace}
\newcommand{\algoal}{\texttt{DEMM}-\texttt{NA}\xspace}
\newcommand{\embalgo}{\texttt{FAAO}\xspace}
\newcommand{\clustalgo}{\texttt{SSKC}\xspace}

\newenvironment{customlegend}[1][]{%
    \begingroup
    \csname pgfplots@init@cleared@structures\endcsname
    \pgfplotsset{#1}%
}{%
    \csname pgfplots@createlegend\endcsname
    \endgroup
}%

\def\addlegendimage{\csname pgfplots@addlegendimage\endcsname}

\makeatletter
\newcommand\footnoteref[1]{\protected@xdef\@thefnmark{\ref{#1}}\@footnotemark}
\makeatother

\let\oldnl\nl
\newcommand{\nonl}{\renewcommand{\nl}{\let\nl\oldnl}}

\DeclareMathOperator{\Tr}{Tr}

\SetKwComment{Comment}{/* }{ */}

\SetKwRepeat{Repeat}{do}{until}

\makeatletter 
\g@addto@macro{\@algocf@init}{\SetKwInOut{Parameter}{Parameters}} 
\makeatother

\definecolor{myred}{HTML}{fd7f6f}
\definecolor{myred_new}{HTML}{D8D8D8}
\definecolor{myred_new2}{HTML}{D7191C}
\definecolor{myblue}{HTML}{7eb0d5}
\definecolor{mygreen}{HTML}{b2e061}
\definecolor{mypurple}{HTML}{bd7ebe}
\definecolor{myorange}{HTML}{ffb55a}
\definecolor{myyellow}{HTML}{ffee65}
\definecolor{mypurple2}{HTML}{beb9db}
\definecolor{mypink}{HTML}{fdcce5}
\definecolor{mycyan}{HTML}{8bd3c7}

\definecolor{myblue2}{HTML}{115f9a}
\definecolor{myred2}{HTML}{c23728}

\newcommand{\eat}[1]{}

\definecolor{NSCcol1}{HTML}{1f77b4}
\definecolor{NSCcol2}{HTML}{aec7e8}
\definecolor{NSCcol3}{HTML}{ff7f0e}
\definecolor{NSCcol4}{HTML}{ffbb78}
\definecolor{NSCcol5}{HTML}{98df8a}

\definecolor{ggreen}{HTML}{9BBB59}
\definecolor{olive}{RGB}{128,128,0}        
\definecolor{magenta}{RGB}{255,0,255}

\AtBeginDocument{%
  \providecommand\BibTeX{{%
    \normalfont B\kern-0.5em{\scshape i\kern-0.25em b}\kern-0.8em\TeX}}}

\setcopyright{acmcopyright}
\copyrightyear{2018}
\acmYear{2018}
\acmDOI{XXXXXXX.XXXXXXX}

\acmPrice{15.00}
\acmISBN{978-1-4503-XXXX-X/18/06}

\begin{document}


\title{Effective Clustering for Large Multi-Relational Graphs}
\subtitle{Technical Report}

\author{Xiaoyang Lin}
\affiliation{%
  \institution{Hong Kong Baptist University}
  \country{}
  \country{Hong Kong SAR, China}
}
\email{csxylin@comp.hkbu.edu.hk}

\author{Runhao Jiang}
\affiliation{%
  \institution{Hong Kong Baptist University}
  \country{Hong Kong SAR, China}
}
\email{csrhjiang@comp.hkbu.edu.hk}
\orcid{0009-0000-4841-9175}

\author{Renchi Yang}
\authornote{Corresponding Author}
\affiliation{%
  \institution{Hong Kong Baptist University}
  \country{Hong Kong SAR, China}
}
\email{renchi@hkbu.edu.hk}
\orcid{0000-0002-7284-3096}

\settopmatter{printfolios=true}

\renewcommand{\shortauthors}{Lin et al.}

\begin{abstract}
Multi-relational graphs (MRGs) are an expressive data structure for modeling diverse interactions/relations among real objects (i.e., nodes), which pervade extensive applications and scenarios. Given an MRG $\G$ with $N$ nodes, partitioning the node set therein into $K$ disjoint clusters (referred to as MRGC) is a fundamental task in analyzing MRGs, which has garnered considerable attention.
However, the majority of existing solutions towards MRGC either yield severely compromised result quality by ineffective fusion of heterogeneous graph structures and attributes, or struggle to cope with sizable MRGs with millions of nodes and billions of edges due to the adoption of sophisticated and costly deep learning models.

In this paper, we present \algo{} and \algoplus{}, two effective MRGC approaches to address the aforementioned limitations. Specifically, our algorithms are built on novel two-stage optimization objectives, where the former seeks to derive high-caliber node feature vectors by optimizing the {\em multi-relational Dirichlet energy} specialized for MRGs, while the latter minimizes the {\em Dirichlet energy} of clustering results over the node affinity graph.
In particular, \algoplus{} achieves significantly higher scalability and efficiency over our based method \algo{} through a suite of well-thought-out optimizations. Key technical contributions include (i) a highly efficient approximation solver for constructing node feature vectors, and (ii) a judicious and theoretically-grounded problem transformation together with carefully-crafted techniques that enable the linear-time clustering without explicitly materializing the $N\times N$ dense affinity matrix.
Further, we extend \algoplus{} to handle attribute-less MRGs through non-trivial adaptations.
Extensive experiments, comparing \algoplus{} against 20 baselines over 11 real MRGs, exhibit that \algoplus{} is consistently superior in terms of clustering quality measured against ground-truth labels, while often being remarkably faster.
\end{abstract}




\maketitle

\section{Introduction}
{\em Multi-relational graphs} (MRGs) are data structures composed of nodes interconnected via multiple types of relations, which excel in modeling and capturing complex relations and associations among real-world entities. 
Practical MRGs include social networks, whose users are connected via friendships and varied interactive activities, biological graphs where biological entities (proteins or genes) are associated by interactions, regulatory relationships, or metabolic pathways, as well as financial networks that encompass diverse edges, such as transactions, ownerships, and contractual relationships.
Due to the omnipresence of such multi-relational data structures, MRGs find broad applications across various domains, including recommendation systems~\cite{GNNRECOfan2022,KEIGLi2024},
biomedicine~\cite{Gebnyue2019,li2021graph}, financial risk control~\cite{wang2024multi,tang2025fraud}, academic network mining~\cite{du2022academic,yang2014recommendation}, social network analysis~\cite{guesmi2019community,lin2021multi}, etc. 

As a fundamental analytical task, the goal of {\em multi-relational graph clustering} (MRGC) is to divide the MRG $\G$ into $K$ disjoint groups of nodes that are internally tightly-knit and similar, where the number $K$ of clusters is specified a priori.
Two real-world application examples (depicted in Fig~\ref{fig:youtube}) of MRGC are as follows:
\begin{itemize}[leftmargin=*]
    \item \textbf{Detecting Social Communities:} 
    On the video sharing website YouTube, as shown in Fig.~\ref{fig:youtube}, active users can connect via contact, co-subscription, co-subscribed, sharing favorite videos, and commenting, which form a multiple relational graph (MRG). Through MRGC, we can extract high-quality communities of users sharing similar interests by integrating such heterogeneous interactions/relations~\cite{tang2012community}, thereby facilitating video/YouTuber recommendations and advertising.
    \item \textbf{Neuroscience:} In brain networks, there are structural (e.g., axonal pathways) and functional (e.g., correlations in activity) connections among brain regions (e.g., neurons or cortical areas). The clustering over such multi-relational structures can help identify functional modules and offer valuable insights into brain structures and functions~\cite{ashourvan2019multi,crofts2022structure}.
\end{itemize}
Despite being superior in practical applications, compared to traditional graph clustering, MRGC poses unique challenges in fusing rich structures underlying heterogeneous relations, as well as exploiting nodal attributes that are widely present in real MRGs.

A straightforward treatment for MRGC is to simply convert the MRG $\G$ into a single-relational graph $\overline{\G}$ through an equal weighting of multiple-typed relations therein, followed by applying {\em attributed graph clustering} techniques~\cite{yang2021effective,liu2022survey} over $\overline{\G}$. This paradigm overlooks the specific nuances and importance of different relation types, engendering biased results and subpar clustering quality. 
For instance, on social networks, treating relationships, including friendships, family ties, and professional connections equally will obscure the distinction between close family members and distant acquaintances.

Over the past few years, there has been a surge of interest in designing approaches specially catered for MRGC~\cite{Hassani2020ContrastiveMR,Lin2021GraphFM,Pan2021MultiviewCG,Nie2017SelfweightedMC} (detailed in Section~\ref{related-work}).
The majority of them can be categorized into two groups: {\em Multi-Relational Structure} (MRS)-based and {\em Multi-View Embedding} (MVE)-based methods.
Specifically, as depicted in Figure~\ref{fig:workflows}, the MRS-based methodology~\cite{Lin2021GraphFM,Ling2023DualLG,Park2019UnsupervisedAM} focuses on automatically adjusting weights for the integration of graph structures $\{\AM^{(r)}\}$ under heterogeneous relation types in MRGs, before incorporating node attributes $\XM$ for subsequent clustering.
However, this category of methods primarily hinges on graph structures for weight adjustment, which disregards or underexploits the attribute information. Such an oversight results in inaccurate weights and severe misalignment between graph structures and node attributes.
In contrast, the MVE-based models~\cite{Liu2021MultilayerGC,Xia2021MultiviewGE,Pan2023BeyondHR,Shen2024BeyondRI,Qian2023UpperBB, Mo2023DisentangledMG, Shen2024BalancedMG, Peng2023UnsupervisedMG} reverse the above two steps (see Figure~\ref{fig:workflows}), where the former step turns to encoding attributes $\XM$ on each single-relational graph $\AM^{(r)}$ into node feature vectors $\HM^{(r)}$ severally, whilst the latter step attends to unifying these multi-typed feature vectors $\{\HM^{(r)}\}$ into the final representations $\HM$ for node clustering.
Although this post-fusion scheme enjoys better result effectiveness, it 
fails to adequately capture the structural consistencies, disparities, and complementaries of varied types of relations~\cite{Zhuang2024EnhancingMG}.

\begin{figure}[!t]
\centering
\includegraphics[width=0.99\columnwidth]{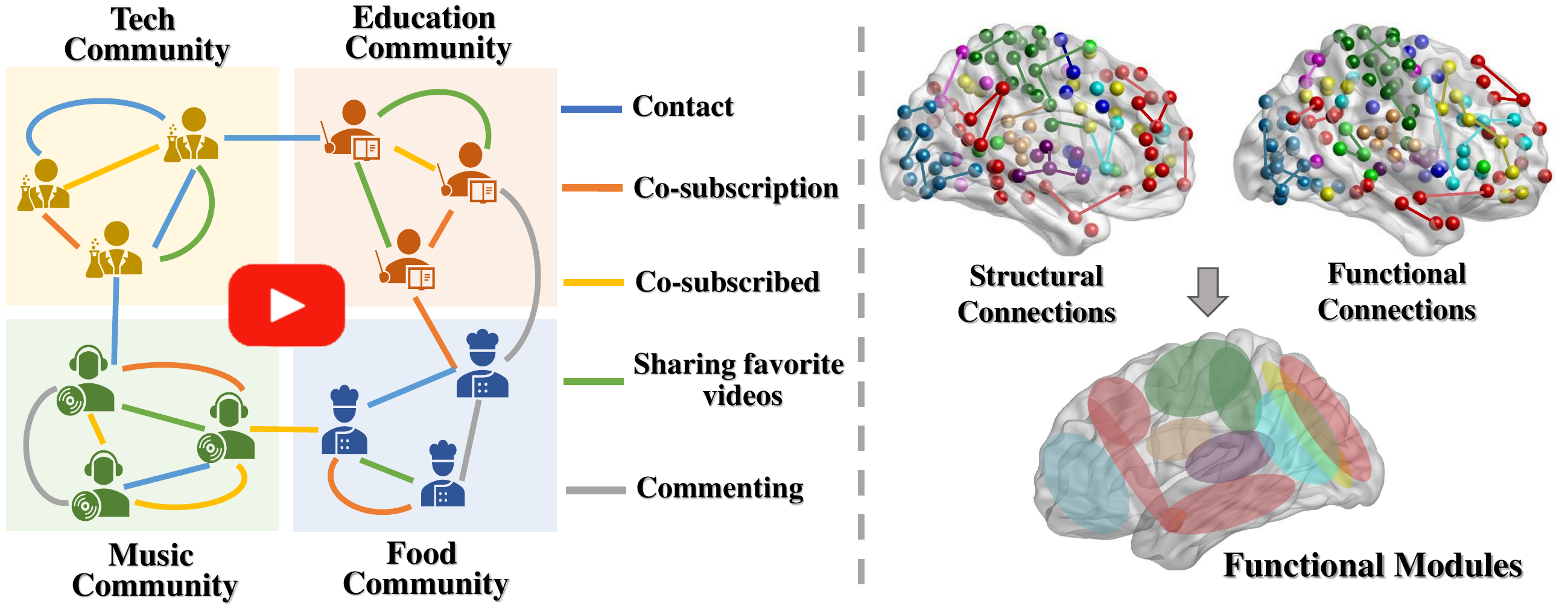}
\vspace{-2ex}
\caption{\textcolor{black}{Real application examples of MRGC.}}\label{fig:youtube}
\vspace{-2ex}
\end{figure}

In summary, extant MRGC studies still have flaws in reconciling multiplex relations and fusing information from heterogeneous structures and attributes, and thus, incur sub-optimal performance. On top of that, most solutions rely on sophisticated matrix solvers or deep learning models that entail substantial memory and compute-intensive operations, which are rather expensive for even medium-sized MRGs. 

\eat{
Overall, the current MRGC methods have following problems: first, most methods falls to reconciling heterogeneous relational dependencies across multiplex interactions and synchronizing the structural-attribute consistency alignment; second, the introduction of multi-relational graphs inherently elevates the time complexity of all MRGC algorithms. Although certain methods attempt to mitigate this by incorporating anchor graphs~\cite{kang2020large} to reduce computational  complexity overhead, such strategies inevitably compromise clustering quality to some extent. To address these challenges, we come up with
\algo{} (\underline{D}irichlet \underline{E}nergy \underline{M}inimization on \underline{M}ulti-Relational Graphs)
}

To overcome the deficiencies of existing methods, we propose \algo{} and \algoplus{} that achieve superb performance for MRGC over multiple real MRG datasets, through the optimization of our novel two-stage objective functions formulated based on the {\em Dirichlet energy} (DE)~\cite{zhou2005regularization} in a principled way.
As overviewed in Figure~\ref{fig:workflows}, distinct from MRS- and MVE-based approaches, \algo{} follows a two-stage pipeline, in which the first stage iteratively refines the node feature vectors $\HM$ by injecting information from node attributes $\XM$ and multiplex graph structures $\{\AM^{(r)}\}$, while the second phase constructs an affinity graph $\SM$ from $\HM$ and derives clusters therefrom. 
More concretely, in the first stage, the feature vectors $\HM$ and weights for integrating $\{\AM^{(r)}\}$ are alternatively updated towards optimizing the notion of {\em multi-relational Dirichlet energy} (MRDE) and ancillary terms, which is a new extension of the DE to MRGs dedicated to enforcing features of adjacent nodes of important relation types to be close.
In the same vein, \algo{} obtains clusters by minimizing their DE on $\SM$ such that cluster assignments of nodes with high affinity in $\SM$ are similar.
Unfortunately, \algo{} suffers from a quadratic complexity for the computation of $\HM$ and materialization of $\SM$, rendering it incompetent for large MRGs.

\begin{figure}[!t]
\centering
\includegraphics[width=0.9\columnwidth]{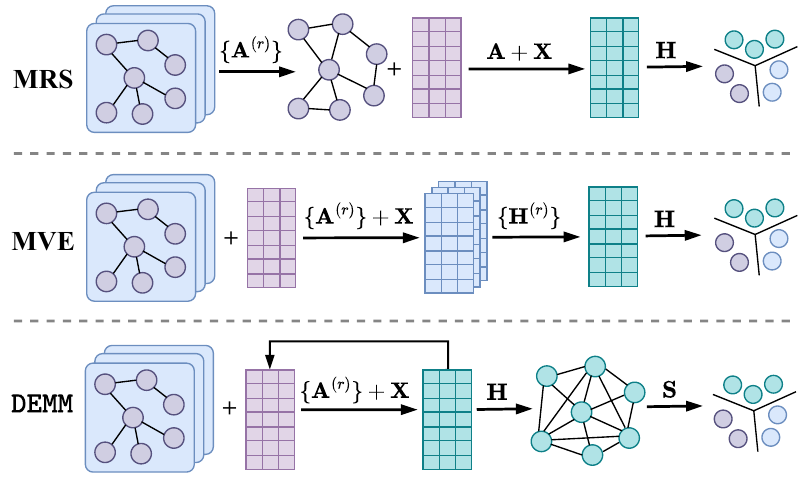}
\vspace{-2ex}
\caption{Workflows of existing MRGC methods and \algo{}.}
\label{fig:workflows}
\vspace{-3ex}
\end{figure}

To this end, we upgrade \algo{} to a linear-time method \algoplus{}, which obtains high efficiency without degrading result utility, via a series of novel algorithmic designs, optimization tricks, and theoretical analyses. Under the hood, \algoplus{} includes a carefully-designed approximate solver \embalgo{} for alternative updating of feature vectors $\HM$ and fusion weights, by uncovering computation bottlenecks and capitalizing on their mathematical properties for fast estimation.
In addition, through theoretically-grounded problem transformations along with our \clustalgo{} algorithm empowered by mathematical apparatus {\em random Fourier features}~\cite{rahimi2007random} and {\em Sinkhorn-Knopp normalization}~\cite{sinkhorn1967concerning}, \algoplus{} judiciously eliminates the need to materialize a quadratic-sized affinity graph and its rear-mounted arduous eigendecomposition in \algo{}.
Furthermore, we enable \algoplus{} over {\em attribute-less} MRGs that are under-explored in previous works with an additional orthogonality constraint.
Our empirical studies evaluating \algoplus{} against 20 competitors on 11 real MRG datasets demonstrate that \algoplus{} consistently and conspicuously outperforms the state-of-the-art solutions for MRGC in terms of clustering quality at a fraction of their computational expenses.

The contributions of this paper can be summarized as follows:
\begin{itemize}[leftmargin=*]
\item Conceptually, we introduce the new notion of MRDE on MRGs and formulate the MRGC task as a two-stage optimization problem based on the MRDE and DE.
\item Methodologically, we develop a brute-force algorithm \algo{} to solve the above objectives for effective MRGC, and a computationally tractable solver \algoplus{} for practical scalability with non-trivial theories and techniques \embalgo{} and \clustalgo{}. \algoplus{} is further extended as \algoal{} to attribute-less MRGs.
\item Empirically, we conduct extensive experiments on 9 real datasets of various sizes to validate the effectiveness and efficiency of proposed methods.
\end{itemize}


\section{Problem Formulation}
In this section, we set up the necessary preliminaries and provide a formalization of the MRGC problem.

\subsection{Symbol and Terminology}
\stitle{Matrix Notation} Throughout this paper, sets are denoted by calligraphic letters, e.g., $\V$. Matrices (resp. vectors) are written in bold uppercase (resp. lowercase) letters, e.g., $\MM$ (resp. $\mathbf{x}$). We use $\MM_{i}$ and $\MM_{\cdot,i}$ to represent the $i^{\textnormal{th}}$ row and column of $\MM$, respectively. $\|\MM\|_F$ denotes the Frobenius norm of matrix $\MM$ and $\texttt{nnz}(\MM)$ is the number of non-zero entries in $\MM$.
A matrix $\MM$ is said to be row-normalized (resp. column-normalized) if each $i^{\textnormal{th}}$ row (resp. column) is $L_2$ normalized, i.e., \(\|\MM_i\|_2\)=1 (resp. \(\|\MM_{\cdot,i}\|_2=1\)). For ease of exposition, we say $\MM\in \NR$ if $\MM$ is row-normalized.
By ``the first $K$ eigenvectors'', we refer to the eigenvectors corresponding to the $K$ largest eigenvalues of a matrix.

\begin{table}[!t]
\centering
\renewcommand{\arraystretch}{0.9}
\begin{footnotesize}
\caption{Frequently used symbols.}\vspace{-3mm} \label{tbl:symbol}
\resizebox{\columnwidth}{!}{%
\begin{tabular}{|p{0.51in}|p{2.5in}|}
\hline
{\bf Symbol}  &  {\bf Description}\\
\hline
$\V, \EDG^{(r)}$   & The node set and edge set of $r^{\textnormal{th}}$ relation type.\\ \hline
$N, M^{(r)}, M$ & The numbers of nodes, edges in $\EDG^{(r)}$, and all the edges.\\ \hline
$R, K$ & The numbers of relation types and desired clusters.\\ \hline
$D, d$ & The dimensions of the input attribute and feature vectors. \\ \hline
$\XM, \HM$ & Initial and target feature vectors of nodes.\\ \hline
$\DM^{(r)}$ & The diagonal degree matrix of $\EDG^{(r)}$. \\ \hline
$\AM^{(r)},\NAM^{(r)}$ & The adjacency matrix of $\EDG^{(r)}$ and its normalized version. \\ \hline
$\omega_r$ & The importance weight for $r^{\textnormal{th}}$ relation type. \\ \hline
$\YM, \SM$ & The NCI and affinity matrix defined in Eq.~\eqref{eq:NCI} and Eq.~\eqref{eq:edge-weight}.\\ \hline
$\mathcal{D}(\HM, \AM^{(r)})$ & The DE of $\HM$ on $\AM^{(r)}$ defined in Eq.~\eqref{eq:DE}.\\ \hline
$\alpha, \beta$ & The coefficients for terms $\mathcal{L}_{\textnormal{MRDE}}$ and $\mathcal{L}_{\textnormal{reg}}$ in Eq.~\eqref{eq:obj-emb}.\\ \hline
$L, m$ & The number of hops and sketching dimension in \embalgo{}.\\ \hline
\end{tabular}%
}
\end{footnotesize}
\vspace{-2ex}
\end{table}

\stitle{Graph Nomenclature}
A {\em multi-relational graph} (MRG) is defined as $\G=(\V,\{\EDG^{(r)}\}_{r=1}^R)$, where $\V$ denotes the set of $N$ distinct nodes and $\EDG^{(r)}$ contains a set of $M^{(r)}$ edges (or relations) between nodes in $\V$ in the $r^{\textnormal{th}}$ ($1\le r\le R$) type of relation.
The total number of edges in $\G$ is denoted by $M=\sum_{r=1}^R{M^{(r)}}$.
For each edge $(v_i,v_j)\in \EDG^{(r)}$ connecting nodes $v_i$ and $v_j$, we say $v_i$ and $v_j$ are neighbors to each other under $r^{\textnormal{th}}$ relation type. 
The degree of $v_i$ (i.e., the neighbors of $v_i$) in $\EDG^{(r)}$ is symbolized by $d^{(r)}_i$.
In particular, we refer to $\G$ as an {\em attributed} MRG if each node $v_i\in \V$ is endowed with a $D$-dimensional attribute vector $\XM_i$, and otherwise an {\em attribute-less MRG}. Unless specified otherwise, an MRG $\G$ is assumed to be attributed by default.

We denote by $\AM^{(r)}\in \{0,1\}^{N\times N}$ the adjacency matrix constructed from the edges in $\EDG^{(r)}$ and by $\DM^{(r)}$ the degree matrix whose diagonal entry $\DM^{(r)}_{i,i}=d^{(r)}_i$. Accordingly, the normalized adjacency matrix $\NAM^{(r)}$ is defined as $\NAM^{(r)}=\DM^{{(r)}^{-\frac{1}{2}}}\AM^{(r)}\DM^{{(r)}^{-\frac{1}{2}}}$ and the normalized Laplacian matrix is $\IM-\NAM^{(r)}$.
Additionally, the oriented incidence matrix of $\EDG^{(r)}$ is symbolized by $\EM^{(r)}\in \mathbb{R}^{N\times M^{(r)}}$, and $\EM^{(r)}{\EM^{(r)}}^\top=\DM^{(r)}-\AM^{(r)}$.
In Definition~\ref{def:mix-gap}, we define the {\em $(\ell_1,\ell_2)$-order maximum eigengap} (OME) of normalized adjacency matrix $\NAM$.
Table \ref{tbl:symbol} lists the frequently used symbols in this paper.

\begin{definition}[$(\ell_1,\ell_2)$-Order Maximum Eigengap]\label{def:mix-gap}
Let $\lambda_{i}(\NAM)$ be the $i^{\textnormal{th}}$ eigenvalue of $\NAM$. The $(\ell_1,\ell_2)$-order maximum eigengap  is $\mu_{\ell_1,\ell_2} = \underset{1\le i\le N}{\max}{|{\lambda_{i}(\NAM)}^{\ell_1}-{\lambda_{i}(\NAM)}^{\ell_2}|}$.
\end{definition}


\stitle{Multi-Relational Graph Clustering (MRGC)} Given an MRG $\G$ and the number $K$ of clusters, the overreaching goal of MRGC is to partition the node set $\V$ into $K$ disjoint groups $\{\mathcal{C}_1,\ldots, \mathcal{C}_K\}$ (i.e., $\bigcup_{k=1}^{K}\mathcal{C}_k=\V$ and $\mathcal{C}_i\cap \mathcal{C}_j=\varnothing$ for $i\neq j$), such that nodes with high attribute homogeneity and strong connectivity under $R$ relation types are in the same group, while dissimilar and distant ones fall into distinct clusters. 

\textcolor{black}{
This goal can typically be achieved through two subtasks. Firstly, the task is to construct a feature matrix $\HM$ that can accurately capture the affinity between nodes in terms of attribute similarity and multiplex structural connectivity in MRGs. Subsequently, clusters $\{\mathcal{C}_1,\ldots, \mathcal{C}_K\}$ can be derived from $\HM$ such that similar feature vectors in $\HM$ are grouped into the same clusters.
}
Particularly, clusters $\{\mathcal{C}_1,\ldots, \mathcal{C}_K\}$ can be represented in matrix form using an $N\times K$ {\em node-cluster indicator} (NCI) $\YM$ in which
\begin{equation}\label{eq:NCI}
\YM_{i,k} = \begin{cases}
\frac{1}{\sqrt{|\mathcal{C}_k|}}, & \text{if $v_i\in \mathcal{C}_k$}, \\
\quad 0, & \text{otherwise}.
\end{cases}
\end{equation}


\subsection{Multi-Relational Dirichlet Energy}\label{sec:MRDE}
\begin{figure}[!t]
\centering
\includegraphics[width=1.0\columnwidth]{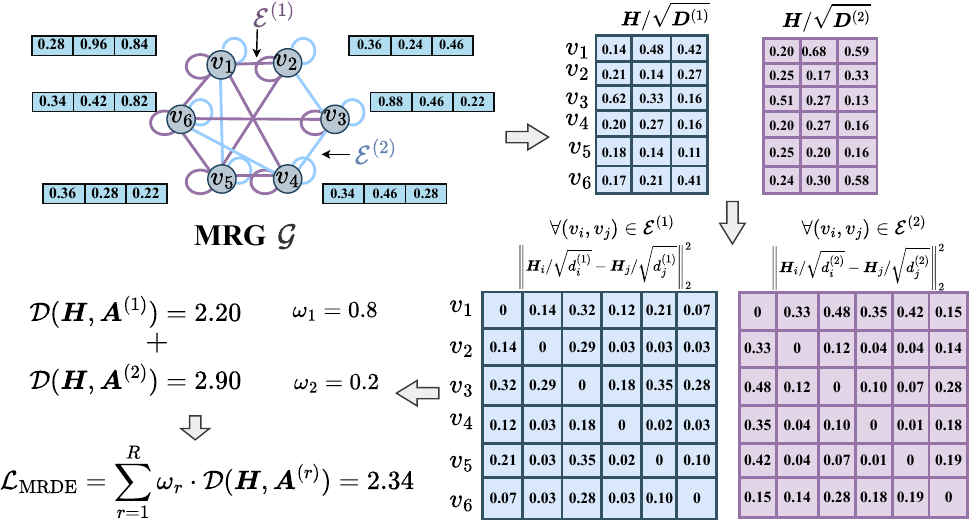}
\vspace{-4ex}
\caption{\textcolor{black}{A running example for MRDE.}}
\label{fig:mrde}
\vspace{-4ex}
\end{figure}

The {\em Dirichlet energy} (DE)~\cite{zhou2005regularization} of feature matrix $\HM\in \mathbb{R}^{N\times d}$ over a graph with edges $\EDG^{(r)}$ is defined by
\begin{footnotesize}
\begin{align}
\mathcal{D}(\HM, \AM^{(r)}) &= \frac{1}{2}\sum_{v_i,v_j\in \V}{\AM^{(r)}_{i,j}\cdot\left\|{\HM_i}/{\sqrt{\|\AM^{(r)}_i\|_1}}-{\HM_j}/{\sqrt{\|\AM^{(r)}_j\|_1}}\right\|^2_2} \notag \\
& =\frac{1}{2}\sum_{(v_i,v_j)\in \EDG^{(r)}}{\left\|{\HM_i}/{\sqrt{d^{(r)}_i}}-{\HM_j}/{\sqrt{d^{(r)}_j}}\right\|^2_2}\label{eq:DE},
\end{align}
\end{footnotesize}
where $\textstyle \left\|{\HM_i}/{\sqrt{d^{(r)}_i}}-{\HM_j}/{\sqrt{d^{(r)}_j}}\right\|^2_2$ measures the dissimilarity of the features of two adjacent nodes $v_i,v_j$ in $\EDG^{(r)}$. Intuitively, $\mathcal{D}(\HM, \AM^{(r)})$ assesses the overall {\em smoothness} of $\HM$ over $\EDG^{(r)}$, indicating whether node features in $\HM$ are similar across adjacent nodes. 

To quantify the smoothness of $\HM$ over the MRG $\G$, we extend the { Dirichlet energy} to the {\em multi-relational Dirichlet energy} (MRDE), which is formulated as follows:
\begin{small}
\begin{equation}\label{eq:MRDE}
\mathcal{L}_{\textnormal{MRDE}}= \sum_{r=1}^R{\omega_r\cdot\mathcal{D}(\HM, \AM^{(r)})}.
\end{equation}
\end{small}
$\omega_1,\ldots,\omega_R$ represents the {\em relation type weights} (hereafter RTWs), which specify the importance of the edges under $R$ relation types, respectively. Particularly, a low MRDE $\mathcal{L}_{\textnormal{MRDE}}$ reflects a high smoothness of $\HM$ over $\G$, while a high MRDE connotes a large divergence in features of adjacent nodes. \textcolor{black}{In other words, this implies that MRDE can be used to measure the quality of feature matrix $\HM$ in fusing multiplex structural connectivity in MRG $\G$.}

\textcolor{black}{
\begin{example}
Figure~\ref{fig:mrde} presents an MRG $\G$ that contains two types of relations ($\EDG^{(1)}$ and $\EDG^{(2)}$) and six nodes (i.e., $v_1$-$v_6$). The first (resp. second) type of relations is colored in purple (resp. blue). Each node $v_i$ in $v_1$-$v_6$ is associated with a 3-dimensional attribute vector $\HM_i$. 
By normalizing the attribute vectors by their respective node degrees in two types of relations, i.e., ${\HM_i}/{\sqrt{d^{(1)}_i}}$ and ${\HM_i}/{\sqrt{d^{(2)}_i}}$, we obtain two new node feature matrices $\HM/\sqrt{\DM^{(1)}}$ and $\HM/\sqrt{\DM^{(2)}}$.
For each edge $(v_i,v_j)\in \EDG^{(1)}$ (resp. $\EDG^{(2)}$), we calculate $\|{\HM_i}/{\sqrt{d^{(1)}_i}}-{\HM_j}/{\sqrt{d^{(1)}_j}}\|^2_2$ (resp. $\|{\HM_i}/{\sqrt{d^{(2)}_i}}-{\HM_j}/{\sqrt{d^{(2)}_j}}\|^2_2$). Summing up these values, respectively, leads to DE $\mathcal{D}(\HM, \AM^{(1)})=2.2$ and $\mathcal{D}(\HM, \AM^{(2)})=2.9$. Suppose that the RTWs are $\omega_1=0.8$ and $\omega_2=0.2$. The MRDE is then $\mathcal{L}_{\textnormal{MRDE}}= 0.8\times\mathcal{D}(\HM, \AM^{(1)})+0.2\times \mathcal{D}(\HM, \AM^{(2)}) = 2.34$.
\end{example}
}

\begin{table}[!h]
    \centering
    \vspace{-3mm}
    \small
    {
    \caption{\textcolor{black}{The MRDE and ACC values by \algoplus{} and \texttt{BMGC}~\cite{Shen2024BalancedMG}}.}
    \label{tab:mrde-com}
    \vspace{-2mm}
    \begin{tabular}{c|c|c|c|c|c|c}
        \hline
        {\bf Method} & {\bf Metric} & \bf{ {\em ACM}} & \bf{ {\em DBLP}} & \bf{ {\em ACM2}} & \bf{ {\em Yelp}} & \bf{ {\em IMDB}} \\
        \hline
        \multirow{2}{*}{\texttt{BMGC}} & MRDE & 1576.6 & 2837.6 & 2765.4 & 2164.5 & 1456.8 \\
            & ACC & 93.0 & 93.4 & 91.3 & 91.5 & 51.0 \\ \hline
       \multirow{2}{*}{\algoplus{}} & MRDE & 1380.6 & 2635.6 & 2505.8 & 2072.1 & 1296.4 \\
         & ACC & 93.6 & 93.7 & 91.3 & 92.7 & 67.6 \\
        \hline
    \end{tabular}
    \arrayrulecolor{black} 
    }
    \vspace{-2mm} 
\end{table}

\textcolor{black}{
Table~\ref{tab:mrde-com} reports the MRDE values of feature matrices obtained by a state-of-the-art MRGC approach \texttt{BMGC}~\cite{Shen2024BalancedMG} and our proposed \algoplus{}, as well as the final clustering accuracies (ACC) on five real datasets, respectively. The empirical results indicate that a smaller MRDE yields a better clustering quality on MRGs.
}

\begin{figure}[!t]
\centering
\includegraphics[width=\columnwidth]{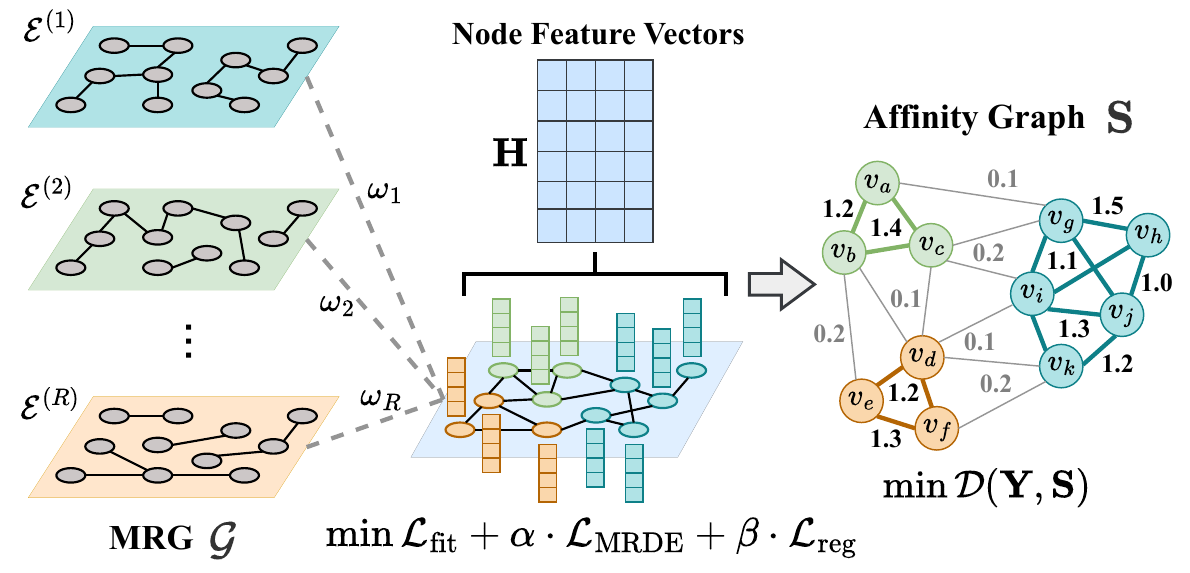}
\vspace{-5ex}
\caption{Two-Stage Optimization Objectives for MRGC.}
\label{fig:idea}
\vspace{-2ex}
\end{figure}

\subsection{Two-Stage Optimization Objectives}\label{sec:objtives}
Next, we define our two-stage objective functions schematized in Figure~\ref{fig:idea} for MRGC, based on the notions of DE and MRDE defined in Eq.~\eqref{eq:DE} and Eq.~\eqref{eq:MRDE}. 

\stitle{Stage I Objective}
{ As shown in Figure~\ref{fig:idea}, the first task is to fuse the attribute information in $\XM$ and the graph structures underlying $R$ types of relations $\{\EDG^{(1)},\EDG^{(1)},\ldots,\EDG^{(R)}\}$ into node feature vectors $\HM\in \NR$ by optimizing the following objective:}
\begin{small}
\begin{equation}\label{eq:obj-emb}
\begin{gathered}
\min_{\HM\in \NR,\ \omega_r\in \mathbb{R}} {\mathcal{L}_{\textnormal{fit}} + \alpha\cdot\mathcal{L}_{\textnormal{MRDE}}} + \beta\cdot \mathcal{L}_{\textnormal{reg}}  \quad \text{s.t.}  \quad \sum_{r=1}^R {\omega_r}=1,
\end{gathered}
\end{equation}
\end{small}
where the fitting and regularization terms $\mathcal{L}_{\textnormal{fit}}$, $\mathcal{L}_{\textnormal{reg}}$ are defined by
\begin{small}
\begin{equation*}\label{eq:regular-term}
\mathcal{L}_{\textnormal{fit}} = \|\HM-\XM\|_F^2, \quad \mathcal{L}_{\textnormal{reg}} = \sum_{r=1}^{R}\omega_r \cdot \|\NAM^{(r)}\|_F^2,
\end{equation*}
\end{small}
and $\alpha, \beta$ are their respective coefficients. The constraint $\sum_{r=1}^R {\omega_r}=1$ enforces a normalization on the $R$ RTWs.

{More specifically, the fitting term $\mathcal{L}_{\textnormal{fit}}$ seeks to reduce the discrepancy between the target node feature vectors $\HM$ and initial features\footnote{For notational convenience, we henceforth refer to the node attribute matrix denoised via a principal component analysis as initial features $\XM$.} $\XM\in \mathbb{R}^{N\times d}$, whereas the MRDE term $\mathcal{L}_{\textnormal{MRDE}}$ renders feature vectors $\HM_i$ and $\HM_j$ of nodes $v_i, v_j$ close to each other when they are connected via an edge of important types, i.e., its RTW $\omega_r$ is large. By minimizing MRDE, this stage seeks to obtain node feature vectors $\HM$ that are consistently smooth over the $R$ types of structural connectivity $\{\EDG^{(1)},\EDG^{(1)},\ldots,\EDG^{(R)}\}$ in MRGs.}
Notably, we additionally incorporate $\mathcal{L}_{\textnormal{reg}}$ to regularize RTWs $\{\omega_r\}_{r=1}^R$ with the consideration of the volumes of their associated edges, thereby preventing over-weighting (resp. under-weighting) the large (resp. small) edge set $\EDG^{(r)}$ (i.e., $\NAM^{(r)}$).
\textcolor{black}{In a nutshell, the main goal of Stage I is to compute RTWs $\{\omega_r\}_{r=1}^R$ automatically by optimizing the objective function to fuse $\{\EDG^{(r)}\}_{r=1}^R$, thereby obtaining node feature vectors $\HM$ while minimizing MRDE.}


\stitle{Stage II Objective}
In the second stage, the goal is to minimize the DE of NCI $\YM$ over an affinity graph $\SM$ constructed from node feature vectors $\HM$, i.e., 
\begin{equation}\label{eq:obj-clust}
\min_{\C_1,\ldots,\C_K}\mathcal{D}(\YM, \SM).
\end{equation}
Under certain assumptions on $\SM$, it can be transformed into
\begin{small}
\begin{equation}\label{eq:Ncut}
\min_{\C_1,\ldots,\C_K} \sum_{k=1}^{K}\sum_{v_i\in \C_k,v_j\in \V\setminus \C_k}{\frac{\SM_{i,j}}{|\C_k|}},
\end{equation}
\end{small}
which is to identify a set $\{\C_1,\ldots,\C_K\}$ of $K$ clusters that minimize the external connectivity of clusters.
{As exemplified in Figure~\ref{fig:idea}, clusters $v_a$-$v_c$, $v_d$-$v_f$, and $v_g$-$v_k$ are an ideal partitioning of $\V$ over $\SM$ since the affinity values of inter-partition nodes are merely $0.1$ or $0.2$, while those of intra-partition nodes are mostly more than $1.0$.
}

In particular, following the conventional choice for the affinity matrix of feature vectors in Euclidean space~\cite{shawe2004kernel,shi2000normalized}, we employ the {\em Gaussian kernel} with pairwise distance to measure the affinity of node pair $(v_i,v_j)$:
\begin{small}
\begin{equation}\label{eq:edge-weight}
\SM_{i,j} = \exp{\left(-\frac{\|{\HM}_i-{\HM}_j\|_2^2}{\sigma}\right)},
\end{equation}
\end{small}
where $\sigma$ is the {\em kernel width} parameter (typically 1 or 2). To accurately discriminate similar and dissimilar node pairs, node feature vectors $\HM$ is normalized such that $-1\le {\HM}_i\cdot {\HM}_j \le 1\ \forall{v_i,v_j \in \V}$ before constructing $\SM$.
\textcolor{black}{
Intuitively, minimizing $\mathcal{D}(\YM, \SM)$ is to minimize the Euclidean distances of feature vectors of nodes in the same clusters.
}

\section{The \algo{} Method}\label{sec:DEMM}

This section presents our first-cut solution \algo{} for MRGC, shown in Algorithm~\ref{alg:demm}. 
\textcolor{black}{At a high level, \algo{} is an approximate method towards optimizing our two-stage objective functions in Eq.~\eqref{eq:obj-emb} and~\eqref{eq:obj-clust} using an alternative optimization and spectral clustering under constraint relaxation, respectively.
}
More concretely, \algo{} takes as input an MRG $\G$, coefficients $\alpha$, $\beta$, and the number $K$ of clusters, and runs in two phases.
In the following, Section~\ref{sec:BFAO} details our brute-force alternative optimization method for our first objective in Eq.~\eqref{eq:obj-emb} to construct feature vectors $\HM$ (Stage I).
In Section~\ref{sec:SAGC}, we transform our clustering objective in Eq.~\eqref{eq:obj-clust} to its theoretically equivalent problem and apply a spectral approach to generate clusters $\{\C_1,\ldots,\C_k\}$ based on $\HM$ (Stage II).
Section~\ref{sec:analysis-1} provides theoretical analyses of \algo{} in terms of its correctness and computational complexity.

\begin{algorithm}[!t]
\caption{\algo{} Algorithm}\label{alg:demm}
\KwIn{ An MRG $\G$, parameters $\alpha$, $\beta$, and $K$.}
\KwOut{A set $\{\C_1,\C_2,\ldots,\C_K\}$ of $K$ clusters.}
\tcc{Brute-Force Alternating Optimization}
${\omega_r} \gets \frac{1}{R}\ \forall{1\le r\le R}$\;
\Repeat{$\HM$ converges}{
Compute $\NAM$ according to Eq.~\eqref{eq:A-sum}\;
Compute $\HM$ according to Eq.\eqref{eq:H-inv-comp}\;
Normalize $\HM$ such that $\HM\in \NR$\;
Update $\omega_r$ according to Eq.~\eqref{eq:update-w} $\forall{1\le r\le R}$\;
}
\tcc{Spectral Affinity Graph Clustering}
Normalize $\HM$ according to Eq.\eqref{eq:PCC}\;
Construct affinity matrix $\SM$ according to Eq.~\eqref{eq:edge-weight}\;
$\UM\gets$ the first $K$ eigenvectors of $\SM$\;
Run $K$-Means over $\UM$ to generate $\{\C_1,\ldots,\C_K\}$\;
\end{algorithm}

\subsection{Brute-Force Alternating Optimization}\label{sec:BFAO}
{ Given the hardness of Eq.~\eqref{eq:obj-emb}, we resort to an alternative optimization strategy to {\em approximately} solve this problem.}
Specifically, we update two variables, i.e., node feature vector $\HM$ and relation type weights $\{\omega_r\}_{r=1}^R$, alternatively, each time fixing one of them and updating the other, using the following rules.

\stitle{Update $\HM$ with $\{\omega_r\}_{r=1}^R$ fixed} Firstly, for any relation type $r$, we have the following fact: $\mathcal{D}(\HM, \EDG^{(r)}) = \texttt{trace}(\HM^{\top}(\IM-\NAM^{(r)})\HM)$.
{
Given fixed RTWs $\{\omega\}_{r=1}^R$, the original optimization objective in Eq.~\eqref{eq:obj-emb} can be simplified as the following partial objective function: $\min_{\HM\in \NR} \|\HM-\XM\|_F^2 + \alpha\cdot\mathcal{L}_{\textnormal{MRDE}}$,
which is equivalent to optimizing
}
\begin{equation}\label{eq:obj-H}
\min_{\HM\in \NR} \|\HM-\XM\|_F^2 + \alpha\cdot \texttt{trace}(\HM^{\top}(\IM-\NAM)\HM),
\end{equation}
where $\NAM$ is the weighted average of $\{\NAM^{(r)}\}_{r=1}^R$ defined in Eq.~\eqref{eq:A-sum}, henceforth referred to as the {\em unified normalized adjacency matrix}.
\begin{small}
\begin{equation}\label{eq:A-sum}
\NAM=\sum_{r=1}^{R}\omega_r\cdot\NAM^{(r)}
\end{equation}
\end{small}

\begin{lemma}\label{lem:opt-H}
The closed-form solution to Eq.~\eqref{eq:obj-H} is 
\begin{small}
\begin{equation}\label{eq:H-inv-comp}
\HM= \frac{1}{1+\alpha}\cdot \left(\IM-\frac{\alpha}{1+\alpha}\NAM\right)^{-1} \XM.
\end{equation}
\end{small}
\end{lemma}

{Our Lemma~\ref{lem:opt-H}\footnote{All proofs appear in Appendix~\ref{sec:proof}.} reveals that the optimal $\HM$ in Eq.~\eqref{eq:obj-H} (intermediate partial optimum to Eq.~\eqref{eq:obj-emb}) can be obtained through a matrix inverse as in Eq.~\eqref{eq:H-inv-comp}.
}

\stitle{Update $\{\omega_r\}_{r=1}^R$ with $\HM$ fixed} { When $\HM$ is at hand, the partial objective function of Eq.~\eqref{eq:obj-emb} can be rewritten as}
\begin{small}
\begin{equation*}
\begin{gathered}
\min_{\{\omega_r\}_{r=1}^R}\alpha\sum_{r=1}^{R}\omega_r \cdot\texttt{trace}\left(\HM^{\top}(\IM-\NAM^{(r)})\HM\right) + \beta \sum_{r=1}^{R}\omega_r \cdot \|\NAM^{(r)}\|_F^2
\end{gathered}
\end{equation*}
\end{small}
such that \(\sum_{r=1}^R{\omega_r}=1\).
By leveraging the Cauchy–Schwarz inequality, we can prove that the above \textcolor{black}{partial} objective is optimized when we set the RTW
\begin{small}
\begin{equation}\label{eq:update-w}
{\omega}_r=\frac{\left(\beta \cdot\|\NAM^{(r)}\|_F^2+\alpha\cdot \texttt{\textnormal{trace}}\left(\HM^{\top}(\IM-\NAM^{(r)})\HM\right)\right)^{-2}}{\sum_{r=1}^{R}\left(\beta\cdot \|\NAM^{(r)}\|_F^2+\alpha\cdot\texttt{\textnormal{trace}}\left(\HM^{\top}(\IM-\NAM^{(r)})\HM\right)\right)^{-2}}
\end{equation}
\end{small}
for each relation type $1\le r \le R$. Notice that $\{\|\NAM^{(r)}\|_F^2\}_{r=1}^R$ can be precomputed and reused in each iteration. We defer the detailed derivative steps to Appendix~\ref{sec:proof} for the sake of space.

Based on the above rules for updating $\HM$ and $\{\omega_r\}_{r=1}^R$, \algo{} (Algorithm~\ref{alg:demm}) begins by initializing RTWs $\omega_r=\frac{1}{R}$ $\forall{1\le r\le R}$ at Line 1. Continuing forth, Algorithm~\ref{alg:demm} starts an iterative process to update $\HM$ and $\{\omega_r\}_{r=1}^R$ in an alternating fashion (Lines 2-7). To be specific, \algo{} first fuses the normalized adjacency matrices of $R$ relation types into the unified normalized adjacency matrix $\NAM$ by Eq.~\eqref{eq:A-sum}, followed by an inverse of matrix \(\IM-\frac{\alpha}{1+\alpha}\NAM\) to get updated node feature vectors $\HM$ in Eq.~\eqref{eq:H-inv-comp} (Lines 3-4). $\HM$ is further row-normalized such that $\HM \in \NR$ at Line 5. After that, Algorithm~\ref{alg:demm} updates each relation type weight $\omega_r$ with the latest $\HM$ by Eq.~\eqref{eq:update-w} at Line 6, and repeats the above procedure until $\HM$ stabilizes.

\subsection{Spectral Affinity Graph Clustering}\label{sec:SAGC}
\begin{lemma}\label{lem:obj-clust}
If $\YM$ is required to be an $N\times K$ NCI as in Eq.~\eqref{eq:NCI}, then
\begin{equation}\label{eq:clust-obj}
\min_{\YM}\mathcal{D}(\YM, \SM) \Leftrightarrow \max_{\YM}{\texttt{trace}(\YM^{\top}\SM\YM)}.
\end{equation}
\end{lemma}
According to Lemma~\ref{lem:obj-clust}, our second optimization objective in Eq.~\eqref{eq:obj-clust} can be equivalently transformed to Eq.~\eqref{eq:clust-obj}, which is essentially an Ncut problem~\cite{shi2000normalized}.
Note that the N-cut problem has been proven to be NP-hard~\cite{goldschmidt1988polynomial,wagner1993between}. {We resort to a standard way of {\em spectral clustering}~\cite{von2007tutorial} to {\em approximately} solve it by first relaxing the discrete constraint in Eq.\eqref{eq:NCI} on $\YM$}, leading to the following objective function:
\begin{equation*}\label{eq:clust-obj-relax}
\max_{\tilde{\YM}\in \mathbb{R}^{N\times K}}{\texttt{trace}(\tilde{\YM}^{\top}\SM\tilde{\YM})}\quad \text{s.t.}\ \tilde{\YM}^\top\tilde{\YM}=\IM,
\end{equation*}
where $\tilde{\YM}$ is a continuous version of NCI $\YM$. 
According to Ky Fan's trace maximization principle~\cite{fan1949theorem}, the optimal solution is $\UM$ that contains the first $K$ eigenvectors of the affinity matrix $\SM$ as columns. 
The remaining task is then the conversion from $\UM$ into NCI $\YM$ by minimizing their {\em distance}, which typically can be done using rounding techniques~\cite{yang2024efficient,shi2003multiclass} or $K$-Means.

As illustrated at Lines 8-11 in Algorithm~\ref{alg:demm}, \algo{} proceeds to derive clusters from node feature vectors $\HM$ by first constructing the affinity matrix $\SM$ according to Eq.~\eqref{eq:edge-weight} (Lines 8-9). Particularly, before computing $\SM$, for each node $v_i\in \V$, Algorithm~\ref{alg:demm} applies a standardization $\HM_i - \overline{h}_i$, followed by an $L_2$ normalization, i.e.,
\begin{small}
\begin{equation}\label{eq:PCC}
{\HM}_i = \frac{\HM_i - \overline{h}_i}{\|\HM_i - \overline{h}_i\|_2},
\end{equation}
\end{small}
where $\overline{h}_i$ is the mean of $\HM_i$, i.e., $\frac{1}{d}\sum_{\ell=1}^d{\HM_{i,\ell}}$. As stated in Theorem 1 in \cite{tan2023metric}, this operation ensures the affinity $\HM_i\cdot \HM_j\in [-1,1]$ for any two nodes $v_i,v_j\in \V$.

Afterwards, the first $K$ eigenvectors $\UM$ of $\SM$ are then calculated through the popular {\em Arnoldi iterative solver} for partial eigendecomposition~\cite{lehoucq1996deflation} at Line 10. Following common practice in spectral clustering, we run $K$-Means over $\UM$ to produce NCI $\YM$, i.e., the $K$ clusters $\{\C_1,\C_2,\ldots,\C_K\}$ at Line 11. 
\eat{
Unlike classic spectral clustering that directly creates clusters by applying $K$-Means over eigenvectors $\UM$, Algorithm~\ref{alg:demm} leverages the \texttt{SNEM-Rounding} algorithm\footnote{We refer interested readers to Appendix~\ref{sec:algo-detail} for the algorithmic details.}~\cite{yang2024efficient} to derive clusters $\{\C_1,\C_2,\ldots,\C_K\}$ from $\UM$ (Line 11). Particularly, \texttt{SNEM-Rounding} iteratively seeks an orthogonal transformation matrix $\RM$ that most closely maps $\UM$ to NCI $\YM$ in terms of spectral norm:
\begin{equation*}
\min_{\RM\in\mathbb{R}^{K\times K}, \YM}{\|\UM\RM-\YM\|_2}\quad \text{s.t.}\ \RM\RM^\top = \IM.
\end{equation*}
As such, \({\texttt{trace}(\YM^{\top}\SM\YM)}\approx {\texttt{trace}(\RM^\top\UM^{\top}\SM\UM\RM)} ={\texttt{trace}(\UM^{\top}\SM\UM)}\), meaning that the loss \({\texttt{trace}(\YM^{\top}\SM\YM)}\) is close to the maximum since $\UM$ is the optimal solution to Eq.~\eqref{eq:clust-obj-relax}, and thus, \({\texttt{trace}(\UM^{\top}\SM\UM)}\) is the upper bound for \({\texttt{trace}(\YM^{\top}\SM\YM)}\).
}

\begin{figure}[!t]
\centering
\includegraphics[width=\columnwidth]{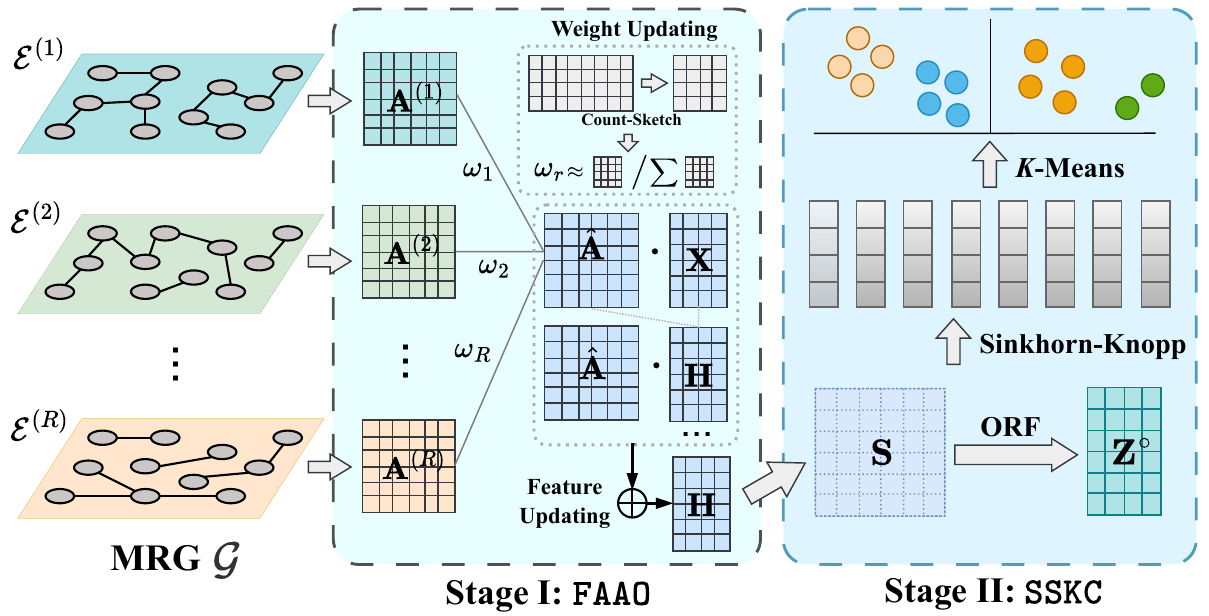}
\vspace{-4ex}
\caption{Overview of \algoplus{}.}
\label{fig:overview}
\vspace{-2ex}
\end{figure}

\subsection{Complexity Analysis}\label{sec:analysis-1}
Since Lines 3, 5, and 8 of Algorithm~\ref{alg:demm} merely involve summation of matrices and matrix normalizations, we focus on analyzing the complexities of computationally intensive operations. 
Particularly, inverting an $N \times N$ matrix followed by the multiplication with $\XM$ at Line 4 incurs a time cost of $O(MN+N^2 d)$. Line 6 calculates $\texttt{\textnormal{trace}}\left(\HM^{\top}(\IM-\NAM^{(r)})\HM\right)$ when updating each relation type weight $\omega_r$, leading to a total of $O(M d+N d^2R)$ time for $R$ relation type weights. 
In the second stage, Line 9 requires materializing the affinity matrix $\SM$ in Eq.~\eqref{eq:edge-weight} for all node pairs, consuming $O(N^2 d)$ time cost, whereas extracting the first $K$ eigenvectors of $\SM$ at Line 10 can be done in $O(N^2 K)$ time~\cite{lehoucq1996deflation}. Therefore, the overall time complexity of \algo{} is bounded by $O(MN+N^2d+Nd^2R)$.

Regarding space overhead, since the matrix inversion in Eq.~\eqref{eq:update-w} yields an $N\times N$ dense matrix and Line 9 materializes an $N \times N$ affinity matrix $\SM$, the total space complexity of \algo{} is $O(N^2)$.

\section{The \algoplus{} Algorithm}\label{sec:demm+}

Despite achieving high clustering quality as exhibited in experiments (Section~\ref{sec:exp}), \algo{} incurs quadratic computational cost and space overhead, and thus, is incompetent for large MRGs. As pinpointed in the preceding section, the colossal time and storage space are ascribed to the materialization of $N\times N$ {\em dense} matrices and expensive matrix operations, including inversion, multiplication, and eigendecomposition, in either the construction of node feature vectors $\HM$ or the generation of clusters $\{\C_1,\C_2,\ldots,\C_K\}$.
To alleviate such issues, this section further proposes \algoplus{} for MRGC, which is able to advance MRG clustering performance in efficiency without compromising the effectiveness. 

{Figure~\ref{fig:overview} depicts an overview of \algoplus{}. Akin to \algo{}, \algoplus{} consists of two secondary algorithms, \embalgo{} and \clustalgo{}, for the constructions of $\HM$ and $\{\C_1,\C_2,\ldots,\C_K\}$, respectively. At a high level, \algoplus{} develops a truncated approximation for $\HM$ and sketching-based estimations for RTWs in the first stage. Subsequently, it transforms the costly spectral clustering in Stage II to a cheap $K$-Means by adjusting $\SM$.} In Section~\ref{sec:FAAO}, we first elucidate the algorithmic details of \embalgo{}, which approximately updates $\HM$ and RTWs $\{\omega_r\}_{r=1}^R$ alternatively towards optimizing our objective in Eq.~\eqref{eq:obj-emb} using linear time and space. In lieu of optimizing Eq.~\eqref{eq:clust-obj} to get clusters $\{\C_1,\C_2,\ldots,\C_K\}$ via the explicit construction of the $N\times N$ affinity graph $\SM$ and costly spectral clustering, Section~\ref{sec:SSKC} presents our \clustalgo{} method that achieves a linear computational time complexity through a theoretically-grounded problem transformation and innovative adoption of mathematical apparatus, i.e., orthogonal random features and Sinkhorn-Knopp normalization.
\textcolor{black}{
Lastly, we further extend \algoplus{} to handle attribute-less MRGs (dubbed as \algoal{}). The algorithmic details are deferred to Appendix~\ref{sec:extend} for the interest of space. 
}

\eat{
\begin{table}[!t]
\centering
\renewcommand{\arraystretch}{1.0} 
\begin{small}
\caption{$(1-\mu(\NAM))^L$ of real MRGs. \renchi{so tiny?}}\label{tbl:sepctral-gap}
\vspace{-3mm}
\begin{tabular}{c|c|c|c|c|c}
\hline
{\bf $L$} & {\em ACM} & {\em DBLP} & {\em ACM2} & {\em Yelp} & {\em IMDB}  \\
\hline
$5$  &$1.2e^{-8}$&$1.5e^{-3}$ &$4.5e^{-8}$ &$1.4e^{-2}$ &$1.3e^{-11}$  \\ \hline
$10$  &$1.5e^{-16}$&$2.3e^{-6}$ &$2.1e^{-15}$&$1.9e^{-4}$& $1.6e^{-22}$ \\ \hline
$15$  &$1.9e^{-24}$ &$4.4e^{-9}$&$9.1e^{-23}$&$2.6e^{-6}$& $2.0e^{-33}$ \\ \hline
\end{tabular}
\end{small}
\vspace{-1ex}
\end{table}
}

\begin{algorithm}[!t]
\caption{\embalgo{} Algorithm}\label{alg:emb}
\KwIn{ An MRG $\G$, parameters $\alpha$, $\beta$, and $L$.}
\KwOut{Node feature vectors $\HM$}
${\omega_r} = \frac{1}{R}\ \forall{1\le r\le R}$\;
$\tilde{\EM}^{(r)} \gets \texttt{CountSketch}(\hat{\EM}^{(r)}, m)\ \forall{1\le r\le R}$\;
\Repeat{$\HM$ converges}{
Compute $\NAM$ by Eq.~\eqref{eq:A-sum}\;
$\widehat{\XM}^{(0)}\gets \frac{1}{1+\alpha}\cdot\XM,\ \HM\gets \widehat{\XM}^{(0)}$\;
\For{$\ell\gets 1$ \KwTo $L$}{
$\widehat{\XM}^{(\ell)} \gets \frac{\alpha}{1+\alpha}\cdot \NAM \widehat{\XM}^{(\ell-1)}$\;
$\HM \gets \HM + \widehat{\XM}^{(\ell)}$\;
}
$\HM \gets \HM + \alpha\cdot \widehat{\XM}^{(L)}$\;
Normalize $\HM$ such that $\HM\in \NR$\;
Update $\omega_r$ according to Eq.~\eqref{eq:update-w-new} $\forall{1\le r\le R}$\;
}
\end{algorithm}

\subsection{Fast Approximate Alternating Optimization}\label{sec:FAAO}
Recall that in Section~\ref{sec:BFAO}, the leading cause of the immense computational burden of building $\HM$ is the inversion of $\IM-\frac{\alpha}{1+\alpha}\NAM$ in Eq.~\eqref{eq:H-inv-comp}, which needs an $O(N^3)$ time.
On top of that, although $\{\|\NAM^{(r)}\|_F^2\}_{r=1}^R$ can be precomputed and the exact calculation of $\texttt{\textnormal{trace}}\left(\HM^{\top}(\IM-\NAM^{(r)})\HM\right)$ for each relation type $r$ in Eq.~\eqref{eq:update-w} takes a linear time of $O(Nd^2+M^{(r)}d)$ per iteration, the overall computational expenditure for updating $R$ relation type weights $\{\omega_r\}_{r=1}^R$ for multiple iterations is also significant.
Subsequently, we delineate the rationale behind \embalgo{} for tackling these efficiency challenges. 


\begin{theorem}[\cite{horn2012matrix}]\label{lem:neu}
Let $\MM$ be a matrix whose dominant eigenvalue $\lambda(\MM)$ satisfies $|\lambda(\MM)|<1$. Then, the inverse $(\IM-\MM)^{-1}$ can be expanded as a Neumann series: $(\IM-\MM)^{-1}=\sum_{\ell=0}^\infty\MM^\ell$.
\end{theorem}

\begin{lemma}\label{eq:eigval-bound}
Let $\lambda(\NAM)$ be the dominant eigenvalue of $\NAM$. $|\lambda(\NAM)|\leq 1$.
\end{lemma}

\stitle{Basic Idea}
As per our theoretical outcome in Lemma~\ref{eq:eigval-bound}, the dominant eigenvalue of $\frac{1}{1+\alpha}\NAM$ is bounded by $\frac{1}{1+\alpha}< 1$. Combining it with Theorem~\ref{lem:neu} transforms Eq.~\eqref{eq:H-inv-comp} into an equivalent form:
\begin{small}
\begin{equation}\label{eq:comp-H-new}
\HM=\frac{1}{1+\alpha}\sum_{\ell=0}^{\infty}{\left(\frac{1}{1+\alpha}\right)}^\ell\NAM^\ell\XM,
\end{equation}
\end{small}
{ which remains the optimal solution to our conditional objective function in Eq.~\eqref{eq:obj-H} when RTWs are fixed.}
Although Eq.~\eqref{eq:comp-H-new} offers an iterative way of calculating $\HM$, its exact computation requires summing up an infinite series, which is infeasible.

{Notice that $\NAM^L$ can be interpreted as $L$-hop random walks over $\G$, wherein each entry $\NAM^L_{i,j}$ signifies the probability of a random walk originating from node $v_i$ visiting node $v_j$ at the $L$-th hop. Accordingly, the term $\sum_{\ell=0}^{\infty}{\left(\frac{1}{1+\alpha}\right)}^\ell\NAM^\ell$ in $\HM$ can be perceived as the total probabilities of random walks of various lengths, where length-$\ell$ random walks are weighted with ${\left(\frac{1}{1+\alpha}\right)}^\ell$. As such, one potential solution to estimate $\HM$ is to discard long random walks, i.e., random walks beyond $L$ ($L$ is a small integer) hops, as their weights are lower.
}

Due to the {\em mixing time}~\cite{levin2017markov} of random walks on graphs, the $L$-hop random walk probability $\NAM^{L}_{i,j}$ converges to an invariant value $a_{i,j}$ after a number of steps.
Mathematically, the overall discrepancy between $(L+1)$-hop and $L$-hop random walk probabilities $\|\NAM^{L+1}-\NAM^L\|_2$ can be proved to be equal to the $(L,L+1)$-OME $\mu_{L,L+1}$:
\begin{equation*}\label{eq:OME-dist}
\|\NAM^{L+1}-\NAM^L\|_2=\mu_{L,L+1}.
\end{equation*}
As reported in Figure~\ref{fig:mu_L}, $(L,L+1)$-OME of real MRGs {\em DBLP}~\cite{ZhaoWSLY20} and {\em Yelp}~\cite{Shi2022RHINERS} dwindles to nearly zero when $L$ is roughly $8$, indicating that the convergence/mixing of $\NAM^L$ can be achieved with merely a handful of hops.
Inspired by this, our idea is to compute an approximate $\HM$, 
\begin{small}
\begin{align}
\HM & \approx \frac{1}{1+\alpha}\sum_{\ell=0}^{L}{\left(\frac{\alpha}{1+\alpha}\right)}^\ell\NAM^\ell\XM + \frac{1}{1+\alpha}\sum_{\ell=L+1}^{\infty}{\left(\frac{\alpha}{1+\alpha}\right)}^\ell\NAM^L\XM \notag\\
& = \frac{1}{1+\alpha}\sum_{\ell=0}^{L}{\left(\frac{\alpha}{1+\alpha}\right)}^\ell\NAM^\ell\XM + \left(\frac{\alpha}{1+\alpha}\right)^{L+1}\NAM^L\XM \label{eq:Hprime},
\end{align}
\end{small}
wherein the terms $\NAM^\ell$ beyond $L$-th orders ($\ell \ge L+1$) are estimated using $\NAM^L$. In doing so, $\HM$ can be efficiently calculated as $L$ is merely up to a few dozen in practice.





\begin{lemma}\label{lem:trace-Fnorm}
Let $\hat{\EM}^{(r)}=\DM^{{(r)}-\frac{1}{2}}\EM^{(r)}$. \(\texttt{trace}\left(\HM^{\top}(\IM-\NAM^{(r)})\HM\right)=\|\HM^{\top}\hat{\EM}^{(r)}\|^2_F\ \forall{1\le r\le R}\).
\end{lemma}
On the other hand, Lemma~\ref{lem:trace-Fnorm} suggests that we can leverage the matrix norm \(\|\HM^{\top}\hat{\EM}^{(r)}\|^2_F\) instead of the matrix trace for updating RTW $\omega_r$ in Eq.~\eqref{eq:update-w} in $O(M^{(r)}d)$ time since the normalized oriented incidence matrix $\hat{\EM}^{(r)}$ contains $2M^{(r)}$ non-zero entries and can be materialized in the preprocessing.
This time cost can be further reduced if a low-dimensional sparse matrix $\tilde{\EM}^{(r)} \in \mathbb{R}^{N\times m}$ ($m\ll M^{(r)}$ and $\texttt{nnz}(\tilde{\EM}^{(r)})\ll M^{(r)}$) can be created such that $\|\HM^{\top}\hat{\EM}^{(r)}\|^2_F\approx \|\HM^{\top}\tilde{\EM}^{(r)}\|^2_F$ for estimating $\omega_r$.

\begin{figure}[!t]
\centering
\begin{small}
\begin{tikzpicture}[scale=1,every mark/.append style={mark size=3pt}]
    \begin{axis}[
        height=\columnwidth/2.5,
        width=\columnwidth/1.3,
        ylabel={$\mu_{L,L+1}$},
        xmin=0, xmax=13,
        ymin=0, ymax=1.1,
        xtick={0,1,2,3,4,5,6,7,8,9,10,11,12,13}, 
        ytick={0,0.2,0.4,0.6,0.8,1.0},
        xticklabel style = {font=\small},
        yticklabel style = {font=\small},
        xticklabels={0,1,2,4,6,8,10,15,20,30,50,70,90,100},
        legend style={fill=none,font=\normalsize,anchor=north east,draw=none},
    ]

    \addplot[line width=0.7mm, color=teal] coordinates {
        (0,1.1) (1,0.3849) (2,0.3254) (3,0.2223) (4,0.1314) (5,0.0733)
        (6,0.0398) (7,0.0083) (8,0.0017) (9,0.000068)
        (10,0.000012) (11,0.000017) (12,0.000022) (13,0.000024)
    };


    \addplot[line width=0.7mm, color=NSCcol3, dashed] coordinates {
        (0,1.1) (1,0.3849) (2,0.2683) (3,0.103) (4,0.0357) (5,0.012)
        (6,0.004) (7,0.0003) (8,0.000016) (9,0.0000037)
        (10,0.0000061) (11,0.000085) (12,0.000011) (13,0.000012)
    };

    \legend{{\em DBLP},{\em Yelp}}
    \end{axis}
\end{tikzpicture}
\end{small}
 \vspace{-3ex}
\caption{The OME $\mu_{L,L+1}$ when varying $L$.} \label{fig:mu_L}
\vspace{-2ex}
\end{figure}
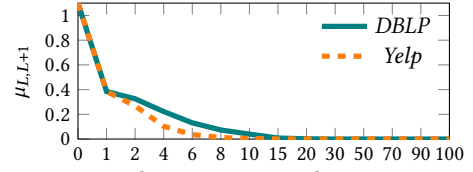

\stitle{Algorithm} Algorithm~\ref{alg:emb} displays the pseudo-code of \embalgo{}. Similar in spirit to the brute-force approach in Section~\ref{sec:BFAO}, \embalgo{} initializes $\omega_r$ as $\frac{1}{R}$ for each relation type at Line 1, and iteratively updates $\HM$ and $\{\omega_r\}_{r=1}^R$ (Lines 3-12). The differences are as follows. Algorithm~\ref{alg:emb} takes as input additional parameters $m, L$ and generates an $m$-dimensional approximation $\tilde{\EM}^{(r)}$ of $\hat{\EM}^{(r)}$ via \texttt{CountSketch}~\cite{clarkson2017low} at Line 2 before entering the iterations.
Moreover, in each iteration, \embalgo{} builds terms $\widehat{\XM}^{(\ell)}=\left(\frac{\alpha}{1+\alpha}\right)^\ell\NAM^\ell\XM \ \forall{0\le \ell \le L}$ using $L$ rounds of iterative sparse matrix multiplications (Lines 5-8), followed by assembling them with $\alpha\cdot\widehat{\XM}^{(L)}$ into $\HM$ as in Eq.~\eqref{eq:Hprime} at Line 9. 
On the basis of updated node feature vectors $\HM$ and precomputed $\{\|\NAM^{(r)}\|_F^2\}_{r=1}^R$, \embalgo{} calculates matrix norm $\|\HM^{\prime \top}\tilde{\EM}^{(r)}\|_F^2$ for each relation type and updates the estimated relation type weight $\omega_r$ by
\begin{small}
\begin{equation}\label{eq:update-w-new}
\omega_r = \frac{\left(\beta \cdot\|\NAM^{(r)}\|_F^2+\alpha\cdot \|\HM^{\prime \top}\tilde{\EM}^{(r)}\|_F^2\right)^{-2}}{\sum_{r=1}^{R}\left(\beta \cdot\|\NAM^{(r)}\|_F^2+\alpha\cdot \|\HM^{\prime \top}\tilde{\EM}^{(r)}\|_F^2 \right)^{-2}}.
\end{equation}
\end{small}

\stitle{Correctness Analysis} Denote by $\HM^{\ast}$ the exact node feature vectors defined in Eq.~\eqref{eq:comp-H-new}.
The following theorem establishes the approximation guarantees of $\HM$ obtained at Line 9 in Algorithm~\ref{alg:emb}. 
\begin{theorem}\label{lem:H-Hprime} 
$\|\HM-\HM^{\ast}\|_F \le \sum_{\ell=L+1}^{\infty}{\frac{\alpha^\ell}{(1+\alpha)^{\ell+1}}}\left\|\NAM^{\ell}-\NAM^L\right\|_2\cdot\|\XM\|_F$, which can be upper bounded by $\left(\frac{\alpha}{1+\alpha}\right)^{L+1}\cdot \|\XM\|_F\cdot \underset{\ell\ge 1}{\max}{\ \mu_{L,L+\ell}}$.
\end{theorem}
Recall that in Figure~\ref{fig:mu_L}, the empirical values of $(L,L+1)$-OME are negligible when $L$ is small, which implies that $\NAM^L$ is close to $\NAM^{L+1}$, and thus, $\NAM^{L+\ell}$ for $\ell> L+1$, rendering approximation error $\|\HM-\HM^{\ast}\|_F=0$.

As for the relation type weights $\{\omega_r\}_{r=1}^R$ in Eq.~\eqref{eq:update-w-new}, \embalgo{} harnesses $\left\| \HM^{\top} \tilde{\EM}^{(r)} \right\|^2_F$ as an approximation of $\texttt{trace}\left(\HM^{\top}(\IM-\NAM^{(r)})\HM\right)$. Particularly, we can derive the following corollary using Theorem 11 in Ref.~\cite{clarkson2017low}:
\begin{corollary}\label{lem:sketch-approximation}
Let $\QM \in \mathbb{R}^{M \times m}$ be a count-sketch matrix and $\tilde{\EM}^{(r)} = \hat{\EM}^{(r)}\QM$, where $m=O(r\epsilon^{-4}\log{(r/\epsilon\delta)}\cdot (r+\log{(1/\epsilon\delta)}))$, $\epsilon$ is an error threshold and $r$ is the rank of $\hat{\EM}^{(r)}$. Then,
\begin{small}
\begin{equation*}
\left\| \HM^{\top} \tilde{\EM}^{(r)} \right\|^2_F = (1\pm \epsilon)^2\cdot \texttt{trace}\left(\HM^{\top}(\IM-\NAM^{(r)})\HM\right) 
\end{equation*}
\end{small}
holds with a probability of at least $1-\delta$.
\end{corollary}
\textcolor{black}{
As empirically validated in Appendix~\ref{sec:add-param-als}, a small $m$ (e.g., $20$) leads to accurate approximation of $\hat{\EM}^{(r)}$, ensuring excellent and stable final clustering quality. 
}

\stitle{Complexity Analysis}
Recall that the invocation of \texttt{CountSketch} at Line 2 essentially computes $\hat{\EM}^{(r)}\mathbf{R}^\top$, where $\hat{\EM}^{(r)}$ is the normalized oriented incidence matrix of $\EDG^{(r)}$ with $2M^{(r)}$ non-zero entries (each column has two entries) and sketching matrix $\mathbf{R}\in \mathbb{R}^{m\times M^{(r)}}$ ($m\ll M^{(r)}$) solely has a single non-zero entry in each column. The sparse matrix multiplication $\hat{\EM}^{(r)}\mathbf{R}^\top$ hence entails $O(M^{(r)})$ time, summing up to $O(M)$ time for all the $R$ relation types.
In each iteration (Lines 4-11) of the alternative optimization, the dominant computational overhead lies in Lines 7 and 11. The former costs $O(Md)$ time for each sparse matrix multiplication $\NAM\widehat{\XM}^{(\ell-1)}$, and hence, $O(MLd)$ time for $L$ rounds, while the latter calculates $\|\HM^{\top} \tilde{\EM}^{(r)}\|^2_F$ for updating each relation type weight $\omega_r$, which needs $NdmR$ operations for all the $R$ relation types. In short, the time cost of each iteration for updating $\HM$ and $\{\omega_r\}_{r=1}^R$ is $O(MLd+NdmR)$. Given that $L$, $m$, and the number of iterations are at most a few dozen in practice, and thus, can be considered as constants, the overall time complexity of \embalgo{} is $O(Md+NdR)$.

Algorithm~\ref{alg:emb} only needs incidence and adjacency matrices with $O(M)$ non-zero entries in total, sketched incidence matrix $\tilde{\EM}^{(r)}\in \mathbb{R}^{N\times m}$, and $N\times d$ intermediate feature vectors $\widehat{\XM}^{(\ell)}$ and $\HM$ in the main memory. Consequently, its space cost is $O(M+Nd+Nm)$, which equals $O(M+Nd)$ when $m$ is regarded as a constant.

Let $w^{\ast}_r$ be the new weight of the next iteration. Define $\boldsymbol{\Delta}$ as
\begin{equation}
\boldsymbol{\Delta} = \sum_{r=1}^R{(w^\ast_r-w_r)\cdot \NAM^{(r)}}.
\end{equation}
The new normalized adjacency matrix of the next iteration is
\begin{equation}
\NAM^\ast = \NAM + \boldsymbol{\Delta}.
\end{equation}

\begin{figure}[!t]
\centering
\includegraphics[width=\columnwidth]{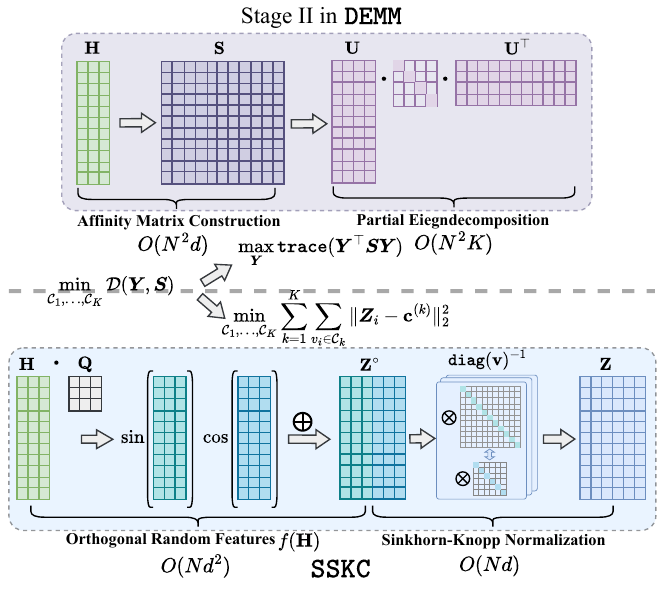}
\vspace{-6ex}
\caption{\textcolor{black}{Illustration of \clustalgo{}.}}
\label{fig:SSKC}
\vspace{-3ex}
\end{figure}

\subsection{Symmetric Sinkhorn-Knopp Clustering}\label{sec:SSKC}

\begin{theorem}\label{lem:clust-trans}
If $\SM$ is doubly stochastic and $\SM=\ZM\ZM^\top$, optimizing Eq.~\eqref{eq:obj-clust} is equivalent to optimizing $\underset{\C_1,\ldots,\C_K}{\min}{\sum_{k=1}^K\sum_{v_i\in \C_k}{\|\ZM_i-\mathbf{c}^{(k)}\|^2_2}}$, 
where $\mathbf{c}^{(k)}=\sum_{v_j\in \C_k}{\frac{\ZM_j}{|\C_k|}}$ stands for the center of cluster $\C_k$.
\end{theorem}

\stitle{Basic Idea} 
{As remarked in Figure~\ref{fig:SSKC}, \algo{} relies on a partial eigendecomposition of the $N\times N$ dense affinity matrix $\SM$ to approximately solve the NP-hard problem in Eq.~\eqref{eq:obj-clust}, which takes $O(N^2\cdot (d+K))$ time and is still prohibitively expensive. 
Our theoretical finding in Theorem~\ref{lem:clust-trans} pinpoints that the clustering objective is equivalent to minimizing the {\em within-cluster sum of squares} (WCSS) on a matrix $\ZM\in \mathbb{R}^{N\times z}$ that satisfies $\ZM\ZM^\top=\SM$ where $\SM$ is {\em doubly stochastic}. 
This implies that the above spectral clustering over $\SM$ can be further transformed and simplified into a tractable task, i.e., running $K$-Means over $\ZM$, if we make an adjustment to (a normalization) $\SM$ and calculate $\ZM$ such that $\ZM\ZM^\top=\SM$ is doubly stochastic.
Doing so sidesteps the costly eigendecomposition, and hence, results in a time cost of $O(NKz)$, which is almost linear when $z\ll N$.
}

{To make $\ZM\ZM^\top=\SM$ doubly stochastic,} a straightforward way is to first materialize the affinity matrix $\SM$ as in \algo{}, apply a doubly stochastic normalization of $\SM$, and then decompose it into the product of $\ZM$ and its transpose, all of which, however, are rather costly. Inspired by the {\em kernel tricks}~\cite{liu2011kernel}, the idea of \clustalgo{} is to eliminate the need to explicitly materialize $\SM$ via a mapping function $f(\cdot)$ on $\HM$ such that
\begin{equation*}
\SM \approx f(\HM)\cdot f(\HM)^\top,
\end{equation*}
and $f(\HM)$ can be used as $\ZM$ for subsequent $K$-Means clustering.
Since $\SM$ is defined using a Gaussian kernel, such a mapping function $f(\cdot)$ can be derived via {\em random Fourier features}   (RFF)~\cite{rahimi2007random}. \textcolor{black}{
RFF serves as an alternative to the Gaussian kernel, reducing the computational complexity of kernel methods from nonlinear to linear.
That is to say, RFF leverages the Bochner theorem~\cite{rahimi2007random} to map the kernel function with $f(\cdot)$, which avoids computing Eq.~\eqref{eq:edge-weight} with $O(N^2)$ computational complexity.} Along this line, the next task is to make $\ZM\ZM^\top$ doubly stochastic.

\begin{algorithm}[!t]
\caption{\clustalgo{} Algorithm}\label{alg:clust}
\KwIn{Node feature vectors $\HM$ and the number $K$ of clusters}
\KwOut{A set of $K$ clusters $\{\C_1,\ldots,\C_K\}$.}
Normalize $\HM$ according to Eq.\eqref{eq:PCC}\;
${\ZM}^{\circ}\gets \texttt{ORF}(\HM)$\;
$\overleftarrow{\ZM} \gets \ZM^{\circ},\ \overrightarrow{\ZM} \gets \ZM^{\circ}$\;
\Repeat{$\overrightarrow{\ZM}$ converges}{
$\mathbf{v}\gets \overleftarrow{\ZM}\cdot \left(\overrightarrow{\ZM}^\top\cdot \mathbf{1}\right)$\;
$\overleftarrow{\ZM} \gets \texttt{diag}(\mathbf{v})^{-1}\cdot {\overleftarrow{\ZM}}$\;
$\mathbf{v}\gets \left(\mathbf{1}^\top\cdot \overleftarrow{\ZM}\right)\cdot \overrightarrow{\ZM}^\top$\;
$\overrightarrow{\ZM}\gets \texttt{diag}(\mathbf{v})^{-1}\cdot{\overrightarrow{\ZM}} $\;
}
Run $K$-Means over $\overrightarrow{\ZM}$ to generate $\{\C_1,\ldots,\C_K\}$\;
\end{algorithm}
\stitle{Algorithm} 
Figure~\ref{fig:SSKC} summarizes the core steps of \clustalgo{}. It first constructs the mapping function $f(\cdot)$ and $\ZM^\circ=f(\HM)$, i.e., the initial version of $\ZM$, using random Fourier features, followed by a normalization of $\ZM^\circ$ into $\ZM$ for subsequent clustering, both of which can be done in $O(Nd)$ time.

In Algorithm~\ref{alg:clust}, we present the details of \clustalgo{}. Initially, \clustalgo{} leverages the {\em Orthogonal Random Features} (\texttt{ORF}) technique~\cite{yu2016orthogonal} as the mapping function $f(\cdot)$ to transform node feature vectors $\HM$ to $\ZM^\circ$, an initial version of target $\ZM$, such that $\ZM^\circ{\ZM^\circ}^\top \approx \SM$ (Line 1).
More concretely, \texttt{ORF} first transforms $\HM$ into $\tilde{\HM}=\HM\cdot\QM^\top,$ using a uniformly distributed random orthogonal matrix $\QM\in \mathbb{R}^{d\times d}$, and then constructs ${\ZM}^{\circ}$ by
\begin{small}
\begin{equation*}
{\ZM}^{\circ}=\frac{1}{\sqrt{d}}\cdot(\texttt{sin}(\tilde{\HM}) \mathbin\Vert \texttt{cos}(\tilde{\HM})),
\end{equation*}
\end{small}
where $\mathbin\Vert$ denotes the horizontal concatenation operator for matrices.
It is worth mentioning that the resulting feature dimension $z$ of ${\ZM}^{\circ}$ is $2d$ and $d\ll N$.
Subsequently, \clustalgo{} begins the doubly stochastic normalization of $\ZM^\circ{\ZM^\circ}^\top$.
\textcolor{black}{
We introduce {\em Sinkhorn-Knopp algorithm}~\cite{sinkhorn1967concerning} (\texttt{SK}), which obtains a doubly stochastic matrix by iteratively normalizing the rows and columns of the affinity matrix $\ZM\ZM^{\top}$.} Instead of simply employing the SK that requires materializing $\ZM^\circ{\ZM^\circ}^\top$ for normalization, Algorithm~\ref{alg:clust} initializes 
$\overleftarrow{\ZM}$ and $\overrightarrow{\ZM}$ as $\ZM^{\circ}$ at Line 2 and iteratively normalizes them alternately (Lines 3-8), thereby enforcing $\overleftarrow{\ZM}\overrightarrow{\ZM}^\top$ bistochastic.
Particularly, in each iteration, \clustalgo{} computes the row sum vector $\mathbf{v}$ of $\overleftarrow{\ZM}\overrightarrow{\ZM}^\top$ using a trick reordering the matrix multiplication as $\overleftarrow{\ZM}\cdot \left(\overrightarrow{\ZM}^\top\cdot \mathbf{1}\right)$ for higher efficiency, followed by normalizing each row $\overleftarrow{\ZM}$ by $\texttt{diag}(\mathbf{v})^{-1}\cdot {\overleftarrow{\ZM}}$ (Lines 4-5). In the same vein, $\overrightarrow{\ZM}$ is normalized by the column sum vector of $\overleftarrow{\ZM}\overrightarrow{\ZM}^\top$ (Lines 6-7).
As such, at the end of each iteration, a symmetric normalization of rows and columns is imposed on $\overleftarrow{\ZM}\overrightarrow{\ZM}^\top$. The following theorem indicates that 
$\overleftarrow{\ZM}\overrightarrow{\ZM}^\top$ is doubly stochastic with sufficient iterations and $\overleftarrow{\ZM}=\overrightarrow{\ZM}=f(\HM)$.
\begin{theorem}\label{lem:doubly-sto}
$\overleftarrow{\ZM}\overrightarrow{\ZM}^\top$ is doubly stochastic and $\overleftarrow{\ZM}=\overrightarrow{\ZM}$.
\end{theorem}
Finally, Algorithm~\ref{alg:clust} applies the $K$-Means over $\overrightarrow{\ZM}$ and generates clusters $\{\C_1,\C_2,\ldots,\C_K\}$.

\begin{figure}[!t]
\centering
\includegraphics[width=\columnwidth]{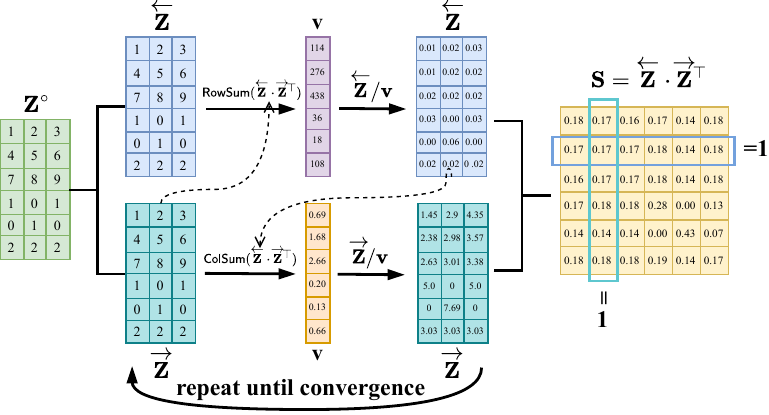}
\vspace{-5ex}
\caption{A running example for the \texttt{SK} normalization.}
\label{fig:sk}
\vspace{-2ex}
\end{figure}

\textcolor{black}{
\begin{example}
Figure~\ref{fig:sk} exemplifies how \clustalgo{} leverages the \texttt{SK} normalization to achieve $\SM=\ZM\ZM^\top$. 
Given a $6\times 3$ feature matrix $\ZM^\circ$ output by \texttt{ORF} (see example in Appendix~\ref{sec:example-ORF}), we initialize $\overleftarrow{\ZM}=\overrightarrow{\ZM}=\ZM^\circ$. 
In the first iteration, \texttt{SK} calculates the sum of entries in each row of $\overleftarrow{\ZM}\overrightarrow{\ZM}^\top$, yielding a vector $\mathbf{v}$ with six rows $[114,276,438,36,18,108]^\top$. Afterwards, six rows in $\overleftarrow{\ZM}$ are normalized by dividing their respective entries in $\mathbf{v}$, e.g., $[1,2,3]/114=[0.01,0.02,0.03]$.
Based on the updated $\overleftarrow{\ZM}$, we start to normalize $\overrightarrow{\ZM}$. \texttt{SK} then calculates the sum of entries in each column of $\overleftarrow{\ZM}\overrightarrow{\ZM}^\top$, leading to a new length-6 vector $\mathbf{v}=[0.69,1.68,2.66,0.2,0.13,,0.66]^\top$. $\overrightarrow{\ZM}$ is subsequently updated by dividing each row by its respective entry in the new $\mathbf{v}$.
By repeating the above alternate procedure sufficiently, we can finally obtain $\overleftarrow{\ZM}=\overrightarrow{\ZM}$ such that the entries in each row and column of $\SM=\overleftarrow{\ZM}\overrightarrow{\ZM}^\top$ sum up to $1.0$, i.e., doubly stochastic. As such, the clusters can be obtained by simply running $K$-means over row vectors of $\overleftarrow{\ZM}$ or $\overrightarrow{\ZM}$.
\end{example}
}

\stitle{Complexity Analysis}
According to \cite{yu2016orthogonal}, $\ZM^\circ$ can be obtained in $O(Nd^2)$ time. By reordering the matrix multiplications as in Lines 5 and 7, $\mathbf{v}$ can be calculated using $O(Nd)$ time. Since the normalizations at Lines 6 and 8 involve $Nd$ operations, each iteration (Lines 5-8) then takes $O(Nd)$ time.
Recall that $K$-Means runs in $O(NK)$ time per iteration. In sum, the total time cost of \clustalgo{} is bounded by $O(Nd^2+NK)$ when the numbers of iterations are considered as constants. Its space cost is $O(Nd)$ since $\HM$ and $\ZM^\circ$ contain $Nd$ and $2Nd$ entries, respectively.

\section{Experiments}\label{sec:exp}

This section experimentally evaluates \algo{}, \algoplus, and \algoal{} against 20 competitors regarding clustering quality and efficiency on 9 real MRGs of varied volumes. 
All experiments are conducted on a Linux machine with an NVIDIA Ampere A100 GPU (80 GB memory), AMD EPYC 7513 CPUs (2.6 GHz), and 1TB RAM.
The codes of all algorithms are collected from their
respective authors, and all are implemented in Python, except \texttt{LMVSC} and \texttt{MCGC}. For reproducibility, the source code and datasets are available at \url{https://github.com/HKBU-LAGAS/DEMM}.

\begin{table}[!t]
\centering
\renewcommand{\arraystretch}{0.9} 
\begin{footnotesize}
\caption{Statistics of Datasets.}\label{tbl:exp-data}
\vspace{-3mm}
\resizebox{\columnwidth}{!}{%
\begin{tabular}{l|r|l|r|r|r}
    \hline
    {\bf Dataset} & $N$ & {\bf Relation Types} & $M$ & $D$ & $K$  \\
    \hline
    \multirow{2}{*}{\centering{\em ACM}} & \multirow{2}{*}{\centering 3K} & Paper-Subject-Paper & 2.2M  & \multirow{2}{*}{\centering 1,870} & \multirow{2}{*}{\centering 3}   \\
     & & Paper-Author-Paper & 29.3K & & \\
    \hline
    \multirow{3}{*}{\centering{\em DBLP}} & \multirow{3}{*}{\centering 4K} & Author-Paper-Author  & 11.1K & \multirow{3}{*}{\centering{334}} & \multirow{3}{*}{\centering{4}}  \\
    & & Author-Paper-Venue-Paper-Author & 5M & & \\
    & & Author-Paper-Term-Paper-Author & 6.8M & & \\
    \hline
     \multirow{2}{*}{\centering{{\em ACM2}}} &  \multirow{2}{*}{\centering{4K}} & Paper-Subject-Paper & 4.3M&  \multirow{2}{*}{\centering{1,902}} &  \multirow{2}{*}{\centering{3}} \\
     & & Paper-Author-Paper & 58K & & \\
    \hline
   \multirow{3}{*}{\centering{ {\em Yelp}}} & \multirow{3}{*}{\centering 2.6K} & Business-User-Business  & 528.3K & \multirow{3}{*}{\centering 82} & \multirow{3}{*}{\centering 3}   \\
   & & Business-Rating-Business & 1.5M & & \\
   & & Business-Service-Business & 2.5M & & \\
    \hline
    \multirow{2}{*}{\centering{\em IMDB}}  & \multirow{2}{*}{\centering 3.6K} & Movie-Actor-Movie & 66.4K & \multirow{2}{*}{\centering 2,000} & \multirow{2}{*}{\centering 3}  \\
    & & Movie-Director-Movie & 13.8K & & \\
    \hline
    \multirow{3}{*}{\centering{\em Protein}} & 
    \multirow{3}{*}{\centering{18.8K}} & 
    {Protein-Protein} & 
    {2.0M} & 
    \multirow{3}{*}{\centering{1280}} & 
    \multirow{3}{*}{\centering{6}} \\
    & & {Protein-Gene-Protein} & {18.9K} & & \\
    & & {Protein-Disease-Protein} & {60.1K} & & \\
    \hline
    \multirow{3}{*}{\centering{\em Amazon}} & 
    \multirow{3}{*}{\centering{11.9K}} & 
    {User-Product-User} & 
    {363.2K} & 
    \multirow{3}{*}{\centering{25}} & 
    \multirow{3}{*}{\centering{2}} \\
    & & {User-Star-User} & {7.1M} & & \\
    & & {User-Review-User} & {2.1M} & & \\
    \hline
    \arrayrulecolor{black}
    \multirow{2}{*}{\centering{\em MAG}}  & \multirow{2}{*}{\centering 113.9K} & Paper-Paper & 1.8M & \multirow{2}{*}{\centering 128} & \multirow{2}{*}{\centering 4}   \\
    & & Paper-Author-Paper & 10.1M & & \\
    \hline
    \multirow{3}{*}{\centering{\em OAG-ENG}} & \multirow{3}{*}{\centering 370.6K} & Paper-Field-Paper & 14.6M & \multirow{3}{*}{\centering 768} & \multirow{3}{*}{\centering 20}   \\
    & & Paper-Author-Paper & 455.7K & & \\
    & & Paper-Paper & 2.1M & & \\
    \hline
    \multirow{3}{*}{\centering{\em OAG-CS}} & \multirow{3}{*}{\centering 546.7K } & Paper-Field-Paper &53.9M & \multirow{3}{*}{\centering 768} & \multirow{3}{*}{\centering 20}   \\
    & & Paper-Author-Paper &1.6M& & \\
    & & Paper-Paper &11.7M & & \\
    \hline
    \multirow{2}{*}{\centering{\em RCDD}}  & \multirow{2}{*}{\centering 11.9M} & Item-b-Item & 421.1M & \multirow{2}{*}{\centering 256} & \multirow{2}{*}{\centering 2}  \\
    & & Item-f-Item  & 353.7M & & \\ \hline
\end{tabular}
}
\end{footnotesize}
\vspace{0ex}
\end{table}

\subsection{Experimental Setup}
\stitle{Datasets}
\textcolor{black}{
We experiment with 11 benchmark MRG datasets of varied volumes and types, whose statistics are presented in Table \ref{tbl:exp-data}. Amid them, {\em ACM}~\cite{Fan2020One2MultiGA}, {\em ACM2}~\cite{Fu2020MAGNNMA}, {\em DBLP}~\cite{ZhaoWSLY20}, {\em MAG}~\cite{Hu2020OpenGB}, {\em OAG-CS}, and {\em OAG-ENG}~\cite{Zhang2019OAGTL} are academic citation networks; {\em Yelp}~\cite{Shi2022RHINERS} and {\em Amazon}~\cite{PengRe2021} are e-commerce review networks; {\em IMDB}~\cite{Wang2019HeterogeneousGA}is a movie review network; {\em RCDD}~\cite{liu2023dink} is risk commodity detection network; and {\em Protein}~\cite{GU2022106127} is a biological network.
}

\stitle{Baselines and Parameters}
For a comprehensive evaluation, we include 20 competing methods in the experiments, which can be categorized into four types:
\begin{itemize}[leftmargin=*]
\item MRGC: \texttt{DMGI}~\cite{Park2019UnsupervisedAM}, \texttt{MvAGC}~\cite{Lin2021GraphFM}, \texttt{MGDCR}~\cite{Mo2023MultiplexGR}, \texttt{BTGF}~\cite{Qian2023UpperBB}, \texttt{DuaLGR}~\cite{Ling2023DualLG}, \texttt{BMGC}~\cite{Shen2024BalancedMG}, and \texttt{DMG}~\cite{Mo2023DisentangledMG};
\item Multi-view graph clustering: \texttt{MCGC}~\cite{Pan2021MultiviewCG}, \texttt{MMGC}~\cite{tan2023metric}, and \texttt{LMVSC}~\cite{kang2020large};
\item Attributed graph clustering: \texttt{Dink-Net} \cite{liu2023dink}, \texttt{DMoN} \cite{tsitsulin2023graph}, \texttt{S3GC} \cite{devvrit2022s3gc}, and \texttt{S\textsuperscript{2}CAG} \cite{lin2024spectral};
\item Attribute-less graph clustering: \texttt{LeadEigvec}~\cite{newman2006finding}, \texttt{SpecClust}~\cite{von2007tutorial}, \texttt{LabelProg}~\cite{raghavan2007near_LabelProg}, \texttt{Louvain}~\cite{blondel2008fast_Louvain}, \texttt{node2vec}~\cite{grover2016node2vec}, \texttt{DeepWalk}~\cite{perozzi2014deepwalk}.
\end{itemize}
In attributed and attribute-less graph clustering baselines, we input the single-relational graph converted from the MRG with equal weights.
For multi-view graph clustering methods, we use the same parameters as in \embalgo{} to generate the feature matrix for each relation type. 
The number of iterations in \algo{}, \algoplus{}, and \algoal{} is fixed to $10$ due to the rapid convergence.
For a fair comparison, we run grid searches on the parameters and report the best clustering performance attained by each evaluated method.
\textcolor{black}{
Table~\ref{tab:baseline_comparison} summarizes the categories, complexities, objectives, and backbone models of the main competitors and our methods.
}

\begin{table}[!h] 
\begin{small}
\vspace{-1ex}
    \centering 
    \caption{Summary of evaluated method.} 
    \label{tab:baseline_comparison} 
\vspace{-2ex}
\resizebox{\columnwidth}{!}{%
\addtolength{\tabcolsep}{-0.3em}
{\color{black}
    \begin{tabular}{l|cccc} 
        \hline
        	&	 \textbf{Category} 	&	  \textbf{Complexity} 	&	 \textbf{Objective} 	&	\textbf{Backbone} 	\\ \hline
 \texttt{DMG} 	&	 MRGC (MVE) 	&	 ${O}(RNd+Md)$ 	&	 Reconstruction 	&	 GNN 	\\
 \texttt{DuaLGR} 	&	 MRGC (MRS) 	&	 ${O}(RN^2)$ 	&	  Reconstruction 	&	 GNN 	\\
 \texttt{MGDCR} 	&	 MRGC (MVE) 	&	 ${O}(R^2Nd^2+Md)$ 	&	 Mutual Info. Max. 	&	 GNN 	\\
 \texttt{DMGI} 	&	 MRGC (MRS) 	&	 ${O}(RNd^2+Md)$ 	&	 Modularity Max. 	&	 GNN 	\\
    \texttt{MvAGC} 	&	MRGC(MRS)	&	 ${O}(Nd^2)$ 	&	 Subspace Clustering 	&	 - 	
  \\ 
   \texttt{MGDCR} 	&	MRGC(MVE)	&	 ${O}(MN+NK^2)$ 	&	 Subspace Clustering 	&	 - 	
  \\ 
   \texttt{BTGF} 	&	MRGC(MVE)	&	 ${O}(N^2d+M^2Nd^2)$ 	& Reconstruction	&	 GNN 	
  \\ 
   \texttt{BMGC} 	&	MRGC(MVE)	&	 ${O}(MN^2+MNd)$ 	&	 Contrastive 	&	 GNN 	
  \\ 
  \texttt{MCGC} 	&	MVGC	&	 ${O}(MN^2(d+K))$ 	&	 Contrastive 	&	 - 	\\
  \texttt{LMVSC} 	&	MVGC	&	 ${O}(MN+NK^2)$ 	&	 Subspace Clustering 	&	 - 	
  \\ 

   \texttt{MMGC} 	&	MVGC	&	 ${O}(MN^2K+MNK)$ 	&	 Subspace Clustering 	&	 - 	
  \\
   \texttt{DMoN} 	&	AGC	&	 ${O}(Nd^2+Md)$ 	&	 Contrastive 	&	 GNN 	\\
 \texttt{Dink-Net} 	&	AGC	&	 ${O}(NdK+dK^2)$ 	&	 Adversarial 	&	 GNN 	\\
   \texttt{S3GC} 	&	AGC	&	 ${O}(Nd^2)$ 	&	 Contrastive	&	 GNN 	
  \\ 
   \texttt{S2AGC} 	&	AGC	&	 ${O}(NKd)$ 	&	 Subspace Clustering 	&	 - 	
  \\ \hline
  \algo{} &	 MRGC	&	 ${O}(MN+Nd(N+dR))$	&	 MRDE	&	 -	\\
 \algoplus{}	&	 MRGC	&	 ${O}(Nd^2+Md)$	&	 MRDE	&	 -	\\
        \hline
    \end{tabular}
}
}
\vspace{-2ex}
\end{small}
\end{table}

\stitle{Evaluation Protocol}
Following previous works~\cite{bhowmick2024dgcluster,Cai2022EfficientDE}, we adopt three classic metrics {\em clustering accuracy} (ACC), {\em Normalized Mutual Information} (NMI), {\em Adjusted Rand Index} (ARI) to assess the quality of output clusters. All of them are calculated against the ground-truth cluster labels, and higher values indicate better quality. Particularly, ACC and NMI scores range from $0$ to $1.0$, whereas ARI falls in the range of $[-0.5,1.0]$.  

For the interest of space, we refer interested readers to Appendix~\ref{sec:add-exp} for more details regarding datasets, \textcolor{black}{baselines}, parameters, and evaluation metrics.

\begin{table*}[!t]
\centering
\renewcommand{\arraystretch}{0.8}
\caption{Clustering quality on small MRGs (best is highlighted in \textcolor{blue!50}{blue} and best baseline underlined).}\vspace{-3mm}
\begin{small}
\addtolength{\tabcolsep}{-0.25em}
\begin{tabular}{c c|ccc | ccc | ccc | ccc|ccc }
\hline
& \multirow{2}{*}{\bf Method} & \multicolumn{3}{c|}{\bf{ {\em ACM}}} & \multicolumn{3}{c|}{\bf{ {\em DBLP}}}   & \multicolumn{3}{c|}{\bf{ {\em ACM2}}} & \multicolumn{3}{c|}{\bf{ {\em Yelp}}}  & \multicolumn{3}{c}{\bf{ {\em IMDB}}}  \\ \cline{3-17}
& & ACC \textuparrow & NMI \textuparrow & ARI \textuparrow & ACC \textuparrow & NMI \textuparrow & ARI \textuparrow & ACC \textuparrow & NMI \textuparrow & ARI \textuparrow & ACC \textuparrow & NMI \textuparrow & ARI \textuparrow & ACC \textuparrow & NMI \textuparrow & ARI \textuparrow \\ 
\hline
\multirow{8}{*}{\rotatebox[origin=c]{90}{w/o attributes}} & \texttt{node2vec}~\cite{grover2016node2vec} &$60.8$ & {$40.7$ }& $32.1$ & $28.5$ & $0.4$ & $0.3$ & $65.1$ & $39.7$ & $31.5$ & $35.7$ & $0.2$ & $0.1$ & $35.4$ & $0.3$ & $0.2$\\ 
& \texttt{DeepWalk}~\cite{perozzi2014deepwalk} &{$61.4$} & $34.9$ & $31.6$ & $75.9$ & $60.4$ & $55.7$ & $56.5$ & $21.1$ & $15.9$ & $51.7$ & $14.4$ & $13.5$ & $36.2$ & $0.2$ & $0.1$\\ 
& \texttt{LeadEigvec}~\cite{newman2006finding} &$35.2$ & $0.7$ & $0.0$ & $79.3$ & $66.1$ & $65.7$ & $49.5$ & $0.2$ & $-0.1$ & {$66.0$} & $29.7$ & $35.6$ & $36.3$ & $6.8$ & $0.0$\\ 
& \texttt{LabelProg}~\cite{raghavan2007near_LabelProg} &$57.1$ &$40.3$ &{$39.4$} &$29.5$ &$0.0$ &$0.0$ &$63.2$ &$40.6$ &$35.0$ &$41.4$ &$0.0$ &$0.0$ &$11.3$ &{$10.9$} &$0.6$\\ 
& \texttt{Louvain}~\cite{blondel2008fast_Louvain} &$55.3$ &$40.1$ &$36.4$ &$79.3$ &$67.6$ &$66.1$ &$60.7$ &$39.3$ &$34.8$ &$60.6$ &$36.6$ &$40.9$ &$13.3$ &$4.9$ &{$1.1$}\\ 
& \texttt{SpecClust}~\cite{von2007tutorial} &$35.3$ &$0.4$ &$0.0$ &{$91.6$} &{$76.7$} &{$80.3$} &{$70.3$} &{$51.1$} &{$41.0$} &$65.2$ &{$37.5$} &{$41.4$} &{$37.9$} &$0.3$ &$0.0$\\ \cline{2-17}
& Improv. &+$6.4$&+$4.0$ &+$3.1$ &+$0.6$ &+$0.9$ &+$1.7$ &+$2.7$ &-$9.8$ &+$1.6$ &+$2.5$ &-$2.2$&-$3.7$&+$0.9$ &-$10.5$ & -$1.1$\\ 
& \algoal{} &$68.0$ & $44.7$ & $42.5$ & $92.2$ & $77.6$ & $82.0$ & $73.0$ & $41.3$ & $42.6$ & $68.5$ & $35.3$ & $37.7$ & $38.8$ & $0.4$ & $0.0$ \\ 
\hline
\hline
\multirow{18}{*}{\rotatebox[origin=c]{90}{w/ attributes}} & \texttt{S3GC} \cite{devvrit2022s3gc} & $66.7$ & $41.9$ & $44.7$ & $54.1$ & $38$ & $20.3$ & $64.2$ & $50.9$ & $46.6$ & $66.5$ & $41.7$ & $44.3$ & $44.7$ & $5.5$ & $5.8$ \\ 
& \texttt{DMoN} \cite{tsitsulin2023graph} & $70.7$ & $45.6$ & $49.5$ & $80.6$ & $54.6$ & $60.2$ & $69.7$ & $38.7$ & $37.6$ & $75.3$ & $51.5$ & $52.2$ & $49.4$ & $12$ & $9.7$ \\ 
& \texttt{Dink-Net} \cite{liu2023dink} & $72.3$ & $49.2$ & $46.1$ & $90.6$ & $74.9$ & $77.4$ & $76.9$ & $48.2$ & $47.8$ & $71.8$ & $42.6$ & $46.1$ & $51.2$ & $10.6$ & $12.5$ \\ 
& \texttt{S\textsuperscript{2}CAG} \cite{lin2024spectral} & $88.6$ & $65$ & $69.5$ & $83.1$ & $58.1$ & $63.2$ & $80.9$ & $55.2$ & $55.2$ & $87.0$ & $59.9$ & $64$ & $53.9$ & $18.0$ & $18.9$ \\ 
& \texttt{DMGI} \cite{Park2019UnsupervisedAM} 	&$84.8$	&$59.6$&$61.5$&$89.0$ &$68.5$ &$74.5$&$76.0$	&$46.5$&$40.0$&$69.2$&$37.3$&$39.2$&$58.5$&$19.0$&$18.9$ \\ 
& \texttt{LMVSC}~\cite{kang2020large} 	&$91.6$&$72.5$&$76.7$&$70.1$&$46.6$&$39.9$&$89.5$	&$64.5$	&$70.1$	&$85.7$		& $58.6$ & $58.4$&$51.9$&$11.9$&$12.3$ \\ 
& \texttt{MvAGC}~\cite{Lin2021GraphFM} 	&$89.8$&$67.4$&$72.1$&$92.8$&$77.3$&$82.8$	&$49.6$	&$0.1$&$0.0$		&$74.4$ &$38.7$&$40.7$ &$56.3$&$3.7$&$9.7$ \\ 
& \texttt{MCGC}~\cite{Pan2021MultiviewCG}	&$91.5$&$71.3$&$76.3$&$92.9$&$77.5$&$83.0$&$70.1$&$45.8$&$36.5$&$56.6$&$20.9$&$8.8$&$61.8$&$11.5$&$18.1$ \\ 
& \texttt{MMGC}~\cite{tan2023metric} 	&$86.6$&$58.1$&$64.5$&$65.8$&$29.4$&$58.5$&$82.3$&$48.4$&$53.1$&$54.9$&$28.0$&$55.7$&$45.2$&$19.5$&$20.1$ \\ 
& \texttt{MGDCR}~\cite{Mo2023MultiplexGR}& $91.9$ &$72.1$&$65.1$&$91.9$&$75.9$&$80.7$&$66.4$&$54.3$&$50.3$&$71.6$&$38.9$&$42.6$&$56.3$&$21.2$&$19.5$ \\ 
& \texttt{BTGF}~\cite{Qian2023UpperBB}&\underline{$93.2$}&\underline{$75.8$}&\underline{$80.9$}&$83.1$&$62.4$&$59.7$&$88.3$&$64.2$&$67.6$&$73.2$&$44.2$&$45.4$&\underline{$66.8$}&\underline{$22.6$}&\underline{$25.7$} \\ 
& \texttt{DuaLGR}~\cite{Ling2023DualLG} 	&$92.7$&$73.2$&$79.4$&$92.4$&$75.5$&$81.7$&$87.3$&$61.3$&$64.8$&$88.1$&$63.4$&$65.0$&$52.4$&$16.0$&$14.5$ \\ 
& \texttt{DMG}~\cite{Mo2023DisentangledMG} 	&$93.0$&$73.6$&$80.3$&$93.4$& \underline{$79.1$}&$83.3$&$87.9$&$67.3$&$63.4$&$56.1$&$42.6$&$39.1$&$48.3$&$11.3$&$14.5$\\
& \texttt{BMGC}~\cite{Shen2024BalancedMG} 	&$93.0$&$75.7$&$80.4$&\underline{$93.4$}&$78.3$&\underline{$84.0$}&\underline{$91.3$}&\cellcolor{blue!30}\underline{$72.0$}& $74.2$&\underline{$91.5$}&\underline{$71.7$}&\underline{$73.8$}&$51.0$&$14.3$&$14.4$\\ \cline{2-17}
& \algo{} &$93.2$ & $75.6$   & $80.7$ & $92.6$ & $76.5$ & $82.1$ & $90.8$ & $70.1$ & $73.2$ & $91.7$ & $69.7$ & $74.7$ &\cellcolor{blue!30} $68.5$ &\cellcolor{blue!30} $25.0$ &\cellcolor{blue!30} $28.1$ \\ 
& Improv. &$0.0$&-$0.2$&-$0.2$&-$0.8$&-$2.6$&-$1.9$&-$0.5$&-$1.9$&-$1.0$&+$0.2$&-$2.0$&+$0.9$&+$1.7$&+$2.4$&+$2.4$\\ 
& \algoplus{}&\cellcolor{blue!30}$93.6$&$\cellcolor{blue!30}77.2$&\cellcolor{blue!30}$81.9$&\cellcolor{blue!30}$93.7$&\cellcolor{blue!30}$79.6$&\cellcolor{blue!30}$84.8$&\cellcolor{blue!30}$91.3$&$71.2$&\cellcolor{blue!30}$74.7$&\cellcolor{blue!30}$92.7$&\cellcolor{blue!30}$72.6$&\cellcolor{blue!30}$77.7$&$67.6$&$24.4$&$26.5$\\ 
& Improv. &+$0.4$&+$1.4$&+$1.0$&+$0.3$&+$0.5$&+$0.8$&+$0.0$&-$0.8$&+$0.5$&+$1.2$& +$1.3$&+$3.9$ &+$0.8$&+$1.8$&+$0.8$ \\ 
\hline
\end{tabular}
\end{small}
\label{tbl:node-clustering-small}
\vspace{0ex}
\end{table*}

\begin{table*}[!t]
\centering
\renewcommand{\arraystretch}{0.8}
\caption{Clustering quality on large MRGs (best is highlighted in \textcolor{blue!50}{blue} and best baseline underlined).}\vspace{-3mm}
\resizebox{\textwidth}{!}{%
\begin{small}
\addtolength{\tabcolsep}{-0.25em}
\begin{tabular}{cc|
ccc|ccc|
ccc|ccc|ccc|ccc}
\hline
& \multirow{2}{*}{\bf Method} 
& \multicolumn{3}{c|}{\bf{\textcolor{black}{{\em Protein}}}} 
& \multicolumn{3}{c|}{\bf{\textcolor{black}{{\em Amazon}}}} 
& \multicolumn{3}{c|}{\bf{ {\em MAG}}}
& \multicolumn{3}{c|}{\bf{ {\em OAG-ENG}}}
& \multicolumn{3}{c|}{\bf{ {\em OAG-CS}}}
& \multicolumn{3}{c}{\bf{ {\em RCDD}}}
\\ \cline{3-20}
& & \textcolor{black}{ACC \textuparrow} & \textcolor{black}{NMI \textuparrow} & \textcolor{black}{ARI \textuparrow}
  & \textcolor{black}{ACC \textuparrow} & \textcolor{black}{NMI \textuparrow} & \textcolor{black}{ARI \textuparrow}
  & ACC \textuparrow & NMI \textuparrow & ARI \textuparrow
  & ACC \textuparrow & NMI \textuparrow & ARI \textuparrow
  & ACC \textuparrow & NMI \textuparrow & ARI \textuparrow
  & ACC \textuparrow & NMI \textuparrow & ARI \textuparrow\\ 
\hline
\multirow{8}{*}{\rotatebox[origin=c]{90}{w/o attributes}} 
& \texttt{node2vec}~\cite{grover2016node2vec}
& \textcolor{black}{$27.1$} & \textcolor{black}{$4.9$} & \textcolor{black}{$2.7$}
& \textcolor{black}{$57.2$} & \textcolor{black}{$3.0$} & \textcolor{black}{$-2.8$}
& $52.1$ & $31.8$ & $19.1$ & $19.7$ & $18.4$ & $2.1$ & $19.5$ & $11.8$ & {$6.5$} & $50.3$ & $0.0$ & $0.0$ \\
& \texttt{DeepWalk}~\cite{perozzi2014deepwalk}
& \textcolor{black}{$33.5$} & \textcolor{black}{$4.7$} & \textcolor{black}{$2.5$}
& \textcolor{black}{$60.2$} & \textcolor{black}{$1.5$} & \textcolor{black}{$2.0$}
& $49.9$ & $35.6$ & $30.1$ & $9.1$ & $3.0$ & $1.1$ & $18.3$ & $12.2$ & $6.1$ &$54.7$ & $0.0$ & {$0.2$} \\
& \texttt{LeadEigvec}~\cite{newman2006finding}
& \textcolor{black}{$32.4$} & \textcolor{black}{$0.3$} & \textcolor{black}{$-0.1$}
& \textcolor{black}{$61.4$} & \textcolor{black}{$0.7$} & \textcolor{black}{$-1.9$}
& $27.1$ & $2.1$ & $0.0$ & $7.3$ & $14.8$ & $0.2$ & $9.8$ & $1.7$ & $0.0$ &$-$&$-$&$-$ \\
& \texttt{LabelProg}~\cite{raghavan2007near_LabelProg}
& \textcolor{black}{$31.5$} & \textcolor{black}{$5.5$} & \textcolor{black}{$0.3$}
& \textcolor{black}{$91.4$} & \textcolor{black}{$1.2$} & \textcolor{black}{$4.2$}
&$15.7$ &$24.5$ &$12.6$ &$11.4$ &$36.8$ &$5.5$ &$17.0$ &{$19.4$} &$5.3$ &$4.3$ &{$4.9$} &$0.1$ \\
& \texttt{Louvain}~\cite{blondel2008fast_Louvain}
& \textcolor{black}{$32.6$} & \textcolor{black}{$11.4$} & \textcolor{black}{$4.6$}
& \textcolor{black}{$40.1$} & \textcolor{black}{$0.5$} & \textcolor{black}{$0.2$}
&$40.8$ &$37.5$ &$28.6$ &$23.2$ &$30.0$ &$10.6$ &$18.2$ &$13.7$ &$5.6$ &$4.1$ &$4.6$ &$0.1$ \\
& \texttt{SpecClust}~\cite{von2007tutorial}
& \textcolor{black}{$35.6$} & \textcolor{black}{$5.8$} & \textcolor{black}{$2.8$}
& \textcolor{black}{$76.3$} & \textcolor{black}{$1.6$} & \textcolor{black}{$-5.6$}
&$27.2$ &$0.1$ &$0.0$ &$7.5$ &$0.6$ &$0.0$ &$9.8$ &$0.1$ &$0.0$ &$-$ &$-$&$-$ \\
\cline{2-20}
& Improv.
& \textcolor{black}{-$3.2$} & \textcolor{black}{-$9.5$} & \textcolor{black}{-$4.6$}
& \textcolor{black}{+$0.2$} & \textcolor{black}{+$1.4$} & \textcolor{black}{+$11.2$}
&+$11.5$ &+$24.8$ &+$21.2$ &+$2.8$ &-$14.7$ &-$0.3$ &+$9.0$ &+$18.9$ &+$10.7$ &-$2.6$ &-$4.9$ & -$0.2$ \\
& \algoal{}
& \textcolor{black}{$32.3$}& \textcolor{black}{$1.9$}& \textcolor{black}{$0.0$}
& \textcolor{black}{$91.6$}& \textcolor{black}{$4.4$}& \textcolor{black}{$15.4$}
&$63.6$&$62.3$&$51.3$&$26.0$&$22.1$&$10.3$&$28.5$&$38.3$&$17.2$&$52.1$ &$0.0$&$0.0$ \\
\hline
\hline
\multirow{14}{*}{\rotatebox[origin=c]{90}{w/ attributes}} 
& \texttt{S3GC} \cite{devvrit2022s3gc}
& \textcolor{black}{$37.7$}& \textcolor{black}{$15.5$}& \textcolor{black}{$9.7$}
& \textcolor{black}{$87.3$}& \underline{\textcolor{black}{$10.3$}}& \textcolor{black}{$2.6$}
& $64.5$ & $61.5$ & $51.5$ & $5.6$ & $3.7$ & $3.4$ &\underline{$35.4$} &\underline{$38.5$} & \underline{$21.4$} &$-$&$-$&$-$ \\
& \texttt{DMoN} \cite{tsitsulin2023graph}
& \underline{\textcolor{black}{$38.0$}}& \textcolor{black}{$6.9$}& \textcolor{black}{$5.5$}
& \textcolor{black}{$44.5$}& \textcolor{black}{$5.8$}& \textcolor{black}{$6.7$}
&$55.8$ & $43.5$ &\cellcolor{blue!30}\underline{$53.7$}  &$13.0$ & $8.4$ & $3.9$ & $11.1$ & $8.5$ & $6.0$&$-$&$-$&$-$ \\
& \texttt{Dink-Net} \cite{liu2023dink}
& \textcolor{black}{$33.1$}& \textcolor{black}{$8.7$}& \textcolor{black}{$4.5$}
& \textcolor{black}{$76.8$}& \textcolor{black}{$2.3$}& \textcolor{black}{$2.1$}
&$64.8$& $61.7$&$49.6$&$-$&$-$&$-$&$-$&$-$&$-$&$-$&$-$&$-$ \\
& \texttt{S\textsuperscript{2}CAG} \cite{lin2024spectral}
& \textcolor{black}{$22.8$}& \textcolor{black}{$1.4$}& \textcolor{black}{$0.6$}
& \textcolor{black}{$63.7$}& \textcolor{black}{$1.4$}& \textcolor{black}{$3.6$}
&\underline{$66.7$}&\underline{$62.5$}&$53.5$&$6.9$&$0.1$&$0.0$&$6.8$&$0.1$&$0.0$& $69.3$&\underline{$13.2$}&\underline{$16.9$} \\
& \texttt{DMGI} \cite{Park2019UnsupervisedAM}
& \textcolor{black}{$23.4$}& \textcolor{black}{$2.1$}& \textcolor{black}{$0.9$}
& \textcolor{black}{$56.0$}& \textcolor{black}{$3.8$}& \textcolor{black}{$1.3$}
&$29.1$ &$0.7$ &$1.0$ &$8.2$ &$1.8$ &$0.6$ &$9.8$ &$4.7$ &$1.3$ &$67.7$ &$2.6$ &$4.2$ \\
& \texttt{LMVSC}~\cite{kang2020large}
& \textcolor{black}{$29.6$}& \textcolor{black}{$3.7$}& \textcolor{black}{$0.0$}
& \textcolor{black}{$63.7$}& \textcolor{black}{$0.0$}& \textcolor{black}{$0.0$}
&$41.7$&$19.5$&$13.1$&$18.6$&$16.4$&$9.5$&$19.3$&$14.2$&$5.7$&$69.9$& $1.6$ &$1.9$ \\
& \texttt{MvAGC}~\cite{Lin2021GraphFM}
& \textcolor{black}{$35.1$}& \textcolor{black}{$11.5$}& \textcolor{black}{$8.8$}
& \textcolor{black}{$75.1$}& \textcolor{black}{$8.8$}& \underline{\textcolor{black}{$14.6$}}
&$54.0$&$32.7$&$27.7$&$12.2$&$5.4$&$2.0$&$10.9$&$4.4$&$1.6$&\underline{$75.1$}&$4.2$&$11.3$ \\
& \texttt{MGDCR}~\cite{Mo2023MultiplexGR}
& \textcolor{black}{$29.1$}& \textcolor{black}{$0.3$}& \textcolor{black}{$0.0$}
& \textcolor{black}{$81.6$}& \textcolor{black}{$2.6$}& \textcolor{black}{$0.0$}
& $61.4$&$54.5$&$44.0$& \underline{$25.7$} &$21.0$&\underline{$13.8$}&$25.3$&$25.9$&$16.8$&$ -$&$-$&$-$ \\
& \texttt{DMG}~\cite{Mo2023DisentangledMG}
& \textcolor{black}{$32.2$}& \textcolor{black}{$0.2$}& \textcolor{black}{$0.1$}
& \underline{\textcolor{black}{$90.9$}}& \textcolor{black}{$1.4$}& \textcolor{black}{$7.6$}
&$55.3$&$43.1$&$34.9$&$25.2$&\underline{$24.5$}&$10.9$&$25.9$&$28.3$&$13.9$&$-$&$-$&$-$ \\
& \texttt{BMGC}~\cite{Shen2024BalancedMG}
& \textcolor{black}{$37.5$}& \underline{\textcolor{black}{$17.3$}}& \underline{\textcolor{black}{$10.3$}}
& \textcolor{black}{$77.5$}& \textcolor{black}{$0.4$}& \textcolor{black}{$1.8$}
&$65.3$&$57.0$&$47.8$&$16.5$&$14.3$&$4.9$&$16.5$&$16.5$&$14.3$&$-$&$-$&$-$ \\
\cline{2-20}
& \algo{}
& \textcolor{black}{$38.9$}& \textcolor{black}{$14.1$}& \textcolor{black}{$8.2$}
& \textcolor{black}{$91.2$}& \textcolor{black}{$14.3$}& \textcolor{black}{$32.4$}
& \cellcolor{blue!30} $68.0$&\cellcolor{blue!30}$64.4$&$52.6$& $-$ &$-$ & $-$&$-$&$-$&$-$&$-$&$-$&$-$ \\
& Improv.
& \textcolor{black}{+$0.9$}& \textcolor{black}{-$3.2$}& \textcolor{black}{-$2.1$}
& \textcolor{black}{+$0.3$}& \textcolor{black}{+$4.0$}& \textcolor{black}{+$17.8$}
&+$1.3$&+$1.9$&-$1.1$& $-$& $-$& $-$&$-$&$-$&$-$&$-$&$-$&$-$ \\
& \algoplus{}
&\cellcolor{blue!30}\textcolor{black}{$39.2$}&\cellcolor{blue!30}\textcolor{black}{$19.4$}&\cellcolor{blue!30}\textcolor{black}{$12.8$}
&\cellcolor{blue!30}\textcolor{black}{$92.6$}&\cellcolor{blue!30}\textcolor{black}{$15.7$}&\cellcolor{blue!30}\textcolor{black}{$34.2$}
&$67.8$&$63.3$&$52.3$&\cellcolor{blue!30}$42.3$&\cellcolor{blue!30}$41.8$&\cellcolor{blue!30}$24.8$&\cellcolor{blue!30}$40.1$&\cellcolor{blue!30}$42.7$&\cellcolor{blue!30}$24.1$&\cellcolor{blue!30}$83.4$&\cellcolor{blue!30}$18.6$&\cellcolor{blue!30}$29.0$ \\
& Improv.
& \textcolor{black}{+$1.2$} & \textcolor{black}{+$2.1$} & \textcolor{black}{+$2.5$}
& \textcolor{black}{+$1.5$} & \textcolor{black}{+$5.4$} & \textcolor{black}{+$19.6$}
&+$1.1$&+$0.8$&-$1.4$&+$16.6$&+$17.3$&+$11.0$&+$4.7$&+$4.2$&+$2.7$ &+$8.3$&+$5.4$&+$12.1$ \\
\hline
\end{tabular}
\end{small}
}
\label{tbl:node-clustering-large}
\vspace{0ex}
\end{table*}

\input{tex/figs/efficiency}

\begin{table*}[!t]
\centering
\renewcommand{\arraystretch}{0.9}
\caption{Ablation studies on small MRGs.}\vspace{-3mm}
\begin{small}
\addtolength{\tabcolsep}{-0.25em}
\begin{tabular}{c c|ccc | ccc | ccc | ccc|ccc  }
\hline
& \multirow{2}{*}{\bf Method} & \multicolumn{3}{c|}{\bf{ {\em ACM}}} & \multicolumn{3}{c|}{\bf{ {\em DBLP}}}   & \multicolumn{3}{c|}{\bf{ {\em ACM2}}} & \multicolumn{3}{c|}{\bf{ {\em Yelp}}}  & \multicolumn{3}{c}{\bf{ {\em IMDB}}}   \\ \cline{3-17}
& & ACC \textuparrow & NMI \textuparrow & ARI \textuparrow & ACC \textuparrow & NMI \textuparrow & ARI \textuparrow & ACC \textuparrow & NMI \textuparrow & ARI \textuparrow & ACC \textuparrow & NMI \textuparrow & ARI \textuparrow & ACC \textuparrow & NMI \textuparrow & ARI \textuparrow \\ 
\hline
& {w/o $\{\omega_r\}_{r=1}^R$} & $92.5$ & $73.1$ & $78.8$ & $93.3$ & $78.2$ & $83.7$ & $90.8$ & $68.6$ & $73.5$ & $92.4$ & $71.6$ & $76.8$ & $66.8$ & $23.5$ & $25.6$ \\
& {w/o $\alpha\cdot \widehat{\XM}^{(L)}$} & $93.4$ & $76.6$ & $81.3$ & $92.9$ & $76.9$ & $82.8$ & $91.3$ & $70.2$ & $74.7$ & $92.3$ & $71.5$ & $76.4$ & $67.0$ & $24.1$ & $24.4$ \\
& {w/o $\mathcal{L}_{\textnormal{reg}}$} & $92.9$ & $75.8$ & $80.1$ & $91.6$ & $73.5$ & $79.7$ & $90.0$ & $69.4$ & $71.0$ & $92.0$ & $71.6$ & $75.5$ & $67.4$ & $24.3$ & $26.2$ \\
\hline
& \algoplus{}&\cellcolor{blue!30}{$93.6$}&\cellcolor{blue!30}{$77.2$}  &\cellcolor{blue!30}{$81.9$}  &\cellcolor{blue!30}{$93.7$}&\cellcolor{blue!30}{$79.6$}&\cellcolor{blue!30}{$84.8$} &\cellcolor{blue!30}{$91.3$}&\cellcolor{blue!30}{$71.2$}  & \cellcolor{blue!30}{$74.7$}&\cellcolor{blue!30}{$92.7$}  &\cellcolor{blue!30}{$72.6$}  &\cellcolor{blue!30}{$77.7$} &\cellcolor{blue!30}{$67.6$}  &\cellcolor{blue!30}{$24.4$}  &\cellcolor{blue!30}{$26.5$} \\ 
\hline

\end{tabular}
\end{small}
\label{tbl:abl-small}
\vspace{0ex}
\end{table*}

\begin{table}[!t]
\centering
\renewcommand{\arraystretch}{0.9}
\caption{Ablation studies on large MRGs.}\vspace{-3mm}
\begin{small}
\addtolength{\tabcolsep}{-0.25em}
\resizebox{\columnwidth}{!}{%
\begin{tabular}{c c|ccc | ccc | ccc  }
\hline
& \multirow{2}{*}{\bf Method} & \multicolumn{3}{c|}{\bf{ {\em OAG-ENG}}} & \multicolumn{3}{c|}{\bf{ {\em MAG}}}   & \multicolumn{3}{c}{\bf{ {\em OAG-CS}}} \\ \cline{3-11}
& & ACC \textuparrow & NMI \textuparrow & ARI \textuparrow & ACC \textuparrow & NMI \textuparrow & ARI \textuparrow & ACC \textuparrow & NMI \textuparrow & ARI \textuparrow \\ 
\hline
& {w/o $\{\omega_r\}_{r=1}^R$} & $36.1$ & $38.6$ & $20.7$ & $65.7$ & $62.5$ & $50.9$  & $36.1$ & $35.9$ & $20.2$ \\
& {w/o $\alpha\cdot \widehat{\XM}^{(L)}$}  & $31.7$ & $33.7$ & $17.4$ & $67.7$ & $61.7$ & $51.3$& $32.7$ & $32.7$ & $16.8$ \\
& {w/o $\mathcal{L}_{\textnormal{reg}}$} & $24.4$ & $22.8$ & $10.1$ & $\cellcolor{blue!30} 67.8$ &\cellcolor{blue!30}{$63.4$}  & \cellcolor{blue!30}{$52.4$} & $20.2$ & $15.5$ & $5.5$ \\
\hline
& \algoplus{}&\cellcolor{blue!30}{$42.3$}  &\cellcolor{blue!30}{$41.8$}  &\cellcolor{blue!30}{$24.8$}  &\cellcolor{blue!30}{$67.8$}  &$63.3$&$52.3$&\cellcolor{blue!30}{$40.1$}&\cellcolor{blue!30}{$42.7$}&\cellcolor{blue!30}{$24.1$} \\ 
\hline
\end{tabular}
}
\end{small}
\label{tbl:abl-large}
\vspace{-1ex}
\end{table}

\subsection{Clustering Quality Evaluation}
This set of experiments studies the clustering quality attained by \algo{}, \algoplus{}, \algoal{}, and 20 competitors on all 9 MRG datasets. We exclude a method or omit its results if it fails to return valid outcomes within 2 days or runs beyond physical memory limits.
Tables~\ref{tbl:node-clustering-small} and~\ref{tbl:node-clustering-large} report the ACC, NMI and ARI scores of all evaluated methods on small and large MRGs, respectively. Each table is divided into two parts, where the top part compares \algoal{} against attribute-less graph clustering baselines by discarding the attributes of all datasets. The best results are highlighted in blue, and the best baselines are underlined.

From the tables, we can make the following observations. Firstly, \algoplus{} consistently and considerably outperforms the best baselines in almost all cases. Particularly, on the large datasets, \algoplus is able to achieve significant gains of $16.6\%$, $17.3\%$, and $11.0\%$ in ACC, NMI, and ARI on {\em OAG-ENG} and remarkable improvements of $8.3\%$, $5.4\%$, and $12.1\%$ on {\em RCDD}, respectively. 
\textcolor{black}{On medium‐sized datasets {\em Protein} and {\em Amazon}, \algoplus{} also outperforms all baselines, yielding notable gains of $1.2\%$, $2.1\%$, $2.5\%$, and $1.7\%$, $5.4\%$, and $19.6\%$ in ACC, NMI and ARI, respectively.}
In addition, it can be observed that \algo{} is comparable to \algoplus{} on most small MRGs but slightly better on {\em IMDB} and {\em MAG}. On larger datasets, \algo{} fails to report results due to the quadratic complexity analyzed in Section~\ref{sec:analysis-1}. The superiority of \algo{} and \algoplus{} over MRGC, attributed graph clustering, and multi-view graph clustering baselines substantiates the effectiveness of our proposed two-stage objectives based on MRDE and DE in fusing multi-relational graph structures. 

On attribute-less MRGs, the variant \algoal{} of \algoplus{} surpasses the best baselines in terms of ACC on all datasets except {\em RCDD}. Most notably, on {\em MAG}, \algoal{} takes a lead of $11.5\%$, $24.8\%$, and $21.2\%$ in ACC, NMI, and ARI. Notice that \texttt{LabelProg} and \texttt{Louvain} determine the number of clusters automatically, which accidentally leads to higher NMI and ARI values on {\em Yelp}, {\em IMDB}, and {\em OAG-ENG} compared to \algoal{}. 

\subsection{Clustering Efficiency Evaluation}
Figure~\ref{fig:time-small} plots the runtime costs consumed by \algoplus{} and 10 strong baselines in Tables~\ref{tbl:node-clustering-small} and~\ref{tbl:node-clustering-large}.
Note that the $y$-axis is in log-scale and the measurement unit for running time is seconds (sec). For fairness, we exclude the time costs needed for loading input data and outputting results in all methods, as well as their pre-training or pre-processing costs. The baselines with the best clustering quality are marked with $\star$.
We exclude \texttt{MCGC}, \texttt{MMGC}, \texttt{BTGF}, and \texttt{DuaLGR} on large MRGs as they are unable to terminate with valid outcomes.

As evidenced in Figure~\ref{fig:time-small}, \algoplus consistently demonstrates higher efficiency across all benchmark datasets.
Compared to the best baselines in Tables~\ref{tbl:node-clustering-small} and~\ref{tbl:node-clustering-large}, \algoplus{} is able to achieve remarkable speedups of $62.5\times$, $23.9\times$, $25.6\times$, $21.4\times$, and $67.6\times$ on small datasets {\em ACM}, {\em DBLP}, {\em ACM2}, {\em Yelp}, and {\em IMDB}, respectively.
Notably, on large MRGs {\em OAG-CS} and {\em OAG-ENG} datasets with tens of millions of edges, the accelerations achieved by \algoplus{} are over $139\times$ and $53\times$, respectively.
Even on the largest dataset {\em RCDD} with $11.9$ million nodes and $0.78$ billion edges, where most recent competitive MRGC approaches \texttt{BTGF}, \texttt{DuaLGR}, \texttt{MGDCR}, \texttt{DMG}, and \texttt{BMGC} fail, \algoplus{} is still nearly $2\times$ faster compared to the best viable baseline \texttt{S\textsuperscript{2}CAG}, while producing significant improvements of $14.1\%$, $5.4\%$, and $12.1\%$ in ACC, NMI, and ARI. 

In Figure~ ~\ref{fig:demm-vs-demm+}, we further corroborate the effectiveness of our proposed algorithms \embalgo{} (Stage I) and \clustalgo{} (Stage II) in enhancing computational efficiency. As reported, 
\algoplus accelerates the computation of both stages in \algo, i.e., the construction of $\HM$ and the generation of clusters. The acceleration is particularly pronounced on the large MRG dataset {\em MAG}, where \algoplus obtains an overall speedup of $3,252\times$ than \algo. 
Moreover, \algo{} cannot handle larger MRGs within 2 days, whereas \algoplus{} finishes the clustering over {\em RCDD} using less than 30 minutes (see Figure~\ref{fig:time-small}).

\input{tex/figs/parameters}

\subsection{Ablation Study}
In this set of experiments, we empirically analyze the efficiency of three key ingredients in \algoplus{}, including the adjustments of RTWs $\{\omega_r\}_{r=1}^R$, the estimator $\alpha\cdot \widehat{\XM}^{(L)}$ of the terms beyond $L$ hops in $\HM$ in Eq.~\eqref{eq:comp-H-new}, and the regularization term $\mathcal{L}_{\textnormal{reg}}$ in Eq.~\eqref{eq:obj-emb}.

According to Tables~\ref{tbl:abl-small} and ~\ref{tbl:abl-large}, compared to three ablated versions that remove the three ingredients, the complete \algoplus always obtains conspicuously superior ACC, NMI, and ARI results on all MRGs. 
Notably, on the {\em ACM2} and {\em DBLP}, the ACC scores increase by $1.3\%$ and $2.1\%$, respectively, by including $\mathcal{L}_{\textnormal{reg}}$ term, which indicates the significance of the regularization term in balanced fusion of multiplex graph structures. The improvements are more significant on {\em OAG-NEG} and {\em OAG-CS}, where substantial ACC improvements of $17.9\%$ and $19.9\%$ can be gained. On {\em MAG}, the conducive effects of the first and second ingredients are still noticeable, whereas the $\mathcal{L}_{\textnormal{reg}}$ term contributes minimally.

\subsection{Parameter Analysis}\label{sec:para_set}

This section investigates the impact of parameters $\alpha$, $\beta$, $L$, and $d$ in \algoplus on two small datasets {\em ACM} and {\em IMDB} and two large MRGs {\em MAG} and {\em OAG-CS}, respectively, by varying each parameter while fixing others. We report ACC scores only as NMI and ARI results are quantitatively similar, and thus, are deferred to Appendix~\ref{sec:add-exp}.

\stitle{Varying $\alpha$} Figure~\ref{fig:alpha-small} shows the impact of varying $\alpha$ from 1 to 8 on the clustering performance on {\em ACM} and {\em IMDB}, while Figure~\ref{fig:alpha-large} presents its effects on {\em MAG} and {\em OAG-CS} when varying it from 10 to 150. The results reveal that $\alpha$ has a negligible influence on {\em ACM}, but a profound impact on {\em IMDB}, {\em MAG}, and {\em OAG-CS}. Specifically, the ACC scores of {\em IMDB} improve monotonically with $\alpha$ until reaching its maximum value at $\alpha=7$, whereas {\em MAG} and {\em OAG-CS} exhibit oscillatory behaviors, attaining peak values at $\alpha=50$ and $110$, respectively. 
Recall that in Eq. \eqref{eq:obj-emb}, $\alpha$ is the weight assigned to the MRDE term towards injecting graph topology information into the node feature vectors $\HM$. Thus, a higher $\alpha$ indicates a larger portion of structural features encoded into $\HM$.
Generally, on the four datasets, a large $\alpha$ is preferred, implying the importance of graph structures in MRGC.

\stitle{Varying $\beta$} Figure~\ref{fig:beta} displays the effects of the regularization weight $\beta$ on ACC scores in Eq.\eqref{eq:obj-emb}. In Figure~\ref{fig:beta-small}, where $\beta$ varies within a short range from 2.5 to 6, the ACC scores of datasets {\em ACM} and {\em IMDB} exhibit divergent trends: the clustering performance of {\em ACM} deteriorates monotonically with increasing $\beta$, whereas that on {\em IMDB} grows progressively. 
In Figure~\ref{fig:beta-large}, when varying $\beta$ from 20 to 90, it can be observed that increasing $\beta$ has little impact on {\em MAG}, but brings a considerable performance rise on {\em OAG-CS}. The differences can be ascribed to their unique structural disparities and volume differences between edges of different relation types.

\stitle{Varying $L$} Figures~\ref{fig:L-small} and~\ref{fig:L-large} depict how the ACC scores change when $L$ is varied from 3 to 17 on {\em ACM} and {\em IMDB}, and from 6 to 20 on {\em MAG} and {\em OAG-CS}. It can be seen that increasing $L$ has little impact on ACC scores on {\em ACM} and {\em IMDB}. 
In comparison, on larger MRGs {\em MAG} and {\em OAG-CS}, the ACC scores first undergo upticks when increasing $L$ to roughly $12$ or $14$, followed by a decrease or plateau. 
The results imply that estimating $\HM$ as in Eq.~\eqref{eq:Hprime} with up to a small number $L$ hops of terms is sufficiently accurate, consistent with our empirical and theoretical analyses in Section~\ref{sec:FAAO}.

\stitle{Varying $d$} 
The parameter $d$ represents the dimension of initial feature vectors $\XM$, which are reduced from the input attribute matrix through a principal component analysis (Section~\ref{sec:objtives}).
Figures~\ref{fig:d-small} and ~\ref{fig:d-large} illustrate the changes in ACC scores on all four datasets when varying $d$ in the ranges of $[8,1024]$ and $[4,128]$. 
For all datasets, we can see a clear rise in performance when enlarging $d$ from 4 to 128, meaning more features are retained. However, the performance of \algoplus{} starts to remain invariant or even undergoes minor drops when $d$ exceeds $128$, on either {\em ACM} and {\em IMDB} whose original attribute dimensions $D$ are up to 2,000, or {\em MAG} and {\em OAG-CS} with $D=128$ and 768. The drops are caused by data noise embodied in original attribute vectors, while the invariance can be explained by the well-known Johnson-Lindenstrauss lemma.

\vspace{-2ex}
\section{Related Work}\label{related-work}

\stitle{Multi-relational Graph Clustering} MRGC focuses on generating consistent node representation by integrating consistency information across different relation types. 
Previous methods typically use adaptive weights to fuse each relation together and construct a unified graph~\cite{Hassani2020ContrastiveMR,Lin2021GraphFM,Pan2021MultiviewCG}, 
\texttt{SwMC}~\cite{Nie2017SelfweightedMC} and \texttt{MvAGC}~\cite{Lin2021GraphFM} are the representative methods with a self-adjusting weight computation algorithm.
To further extract shared patterns from MRG, numerous methods have incorporated consistency information during the fusion of different relation types. \texttt{DuaLGR}~\cite{Ling2023DualLG} proposed a method where soft labels derived from consistency information are used to refine the graphs of each relation type before fusion. 
\texttt{DMGI}~\cite{Park2019UnsupervisedAM} reconstructs MRG by maximizing the mutual information across relation types.
However, these methods cannot fully exploit the dependencies between different relation types and the feature matrices, resulting in their underperformance in MRGs. 

Recently, many approaches generate node embeddings for each relation type individually and identify cross-relational consistencies from different relational graphs ~\cite{Liu2021MultilayerGC,Xia2021MultiviewGE,Pan2023BeyondHR,Shen2024BeyondRI,Peng2023UnsupervisedMG,Pan2021MultiviewCG}. \texttt{BTGF}~\cite{Qian2023UpperBB} designs filters with non-shared parameters for each relation type to obtain node embeddings from diverse perspectives. \texttt{DMG}~\cite{Mo2023DisentangledMG} disentangles consistent and redundant information from the features of different relations. \texttt{BMGC}~\cite{Shen2024BalancedMG} introduces imbalanced multiview learning to refine embeddings derived from less important relation types. 
Nevertheless, these methods overlook the complementary information introduced by fusing MRGs, thus hindering the exploitation of MRGs.

\stitle{Attributed Graph Clustering} Attributed graph clustering (AGC) has been extensively studied nowadays~\cite{bothorel2015clustering,yang2021effective,lai2023re,li2024versatile,xie2025diffusion,yang2024effective,li2023efficient,zheng2025adaptive}. Most recent research has focused on integrating graph topology with node attributes to produce cohesive embeddings ~\cite{Akbas2017AttributedGC,Combe2015ILouvainAA,Li2018CommunityDI,yang2023pane,Zhu2021SimpleSG}, which are then clustered by using classical clustering methods to obtain the final results. With the widespread adoption of deep learning, methods that leverage deep learning models like GNNs~\cite{Scarselli2009TheGN} to learn consistent node representations have gained popularity~\cite{sdcn2020,Cui2020AdaptiveGE,liu2022survey,Huo2021CaEGCNCF,DCRN}, \texttt{DMoN}~\cite{tsitsulin2023graph}, \texttt{Dink-Net}~\cite{liu2023dink}, and \texttt{S3GC}~\cite{devvrit2022s3gc} are the representative methods among them. \textcolor{black}{\texttt{H-GCN}~\cite{HuHGCN19} introduces graph coarsening to capture long-range information, thereby addressing the potential overfitting caused by increasing the depth of GNN models.} To fully integrate topological and attribute information of graphs, attention mechanisms~\cite{Xia2023RobustCM,wang2019attributed,Zhao2022HierarchicalAN} and graph contrastive learning~\cite{Hassani2020ContrastiveMR,Yang2023ClusterguidedCG,Zhao2021GraphDC} have also been widely employed in this process. 
Some recent approaches~\cite{fettal2023scalable,lin2024spectral} integrate subspace clustering with spectral clustering techniques~\cite{Ng2001OnSC}.
However, AGC fails to account for the varying significance of distinct relations, rendering it inapplicable to MRGs.

\stitle{Multi-View Graph Clustering} Multi-view clustering is to group data with heterogeneous feature representations. Due to dimensional differences across vertices, directly linearly combining features from different views is not feasible. Early graph-based approaches rely on constructing similarity matrices followed by spectral clustering.~ \cite{Tang2009ClusteringWM,Zhou2007SpectralCA,Strehl2003ClusterE,Son2016AdaptiveSC}, \texttt{LMVSC}~\cite{kang2020large} enhances scalability by introducing anchor graphs to replace fully connective graph.
\texttt{GTLEC}~\cite{Chen2023OnRM} and \texttt{CGL}~\cite{Li2021ConsensusGL} enhance multi-view consistency through optimized affinity matrix construction.
These methods often incur significant memory consumption for similarity matrix construction. To this end, \texttt{UOMvSC}~\cite{Tang2023UnifiedOM} eliminates the need for explicit similarity matrix construction. Matrix factorization-based methods extract cross-view shared information through matrix decomposition and integrate it into a unified representation ~\cite{Chaudhuri2009MultiviewCV,White2012ConvexMS,Wei2014LearningAM,Wang2017ExclusivityConsistencyRM,Huang2023MultiViewSC,cui2023deep}.

Recent deep learning-based approaches define and optimize specific metrics such as \texttt{MCGC}~\cite{Pan2021MultiviewCG} and \texttt{MAGCN} \cite{Cheng2020MultiViewAG}.
Despite effectively integrating cross-dimensional features, they struggle to generalize to MRG due to incompatible relation modeling. 

\section{Conclusion}
This paper proposes two effective methods, \algo{} and \algoplus{}, for MRGC. \algo{} achieves remarkable clustering performance on MRGs, via our innovative two-stage optimization objectives formulated upon the MRDE of MRGs and DE of affinity graphs. Based thereon, we develop \algoplus{}, which significantly advances the efficiency and scalability of \algo{} via two elaborate secondary algorithms \embalgo{} and \clustalgo{} containing several non-trivial optimization techniques. Our extensive evaluations experimentally manifest the consistent superiority of \algoplus{} over a wide range of baselines in clustering quality and empirical efficiency.
However, the proposed techniques are mainly designed for static MRGs, which struggle to cope with dynamic MRGs with frequent updates.
In the future, our work can be extended to dynamic MRGs by devising sampling and incremental techniques for structural changes (e.g., node/edge insertions/deletions). Moreover, the notion of MRDE can be further generalized to heterogeneous graphs with multiple node types, enabling broader applications in real-world scenarios.

\begin{acks}
This work is supported by the Hong Kong RGC ECS grant (No. 22202623), NSFC No. 62302414, and the Huawei gift fund.
\end{acks}

\balance


\appendix
{\color{black}

\begin{algorithm}[!t]
\caption{\algoal{} Algorithm}\label{alg:emd-AL}
\KwIn{An attribute-less MRG $\G$, parameters $\alpha$, $\beta$, and $K$}
\KwOut{A set of $K$ clusters $\{\C_1,\ldots,\C_K\}$}
{\nonl Lines 1-4 are the same as in Algorithm~\ref{alg:emb}\;}
\setcounter{AlgoLine}{4}
$\HM\gets$ the first $d$ eigenvectors of $\NAM$\;
{\nonl Lines 6-7 are the same as Lines 10-12 in Algorithm~\ref{alg:emb}\;}
\setcounter{AlgoLine}{7}
$\{\C_1,\ldots,\C_K\}\gets \clustalgo{}(\HM, K)$\;
\end{algorithm}

\section{Extension to Attribute-less MRGs}\label{sec:extend}
In this section, we further extend \algoplus{} to handle attribute-less MRGs and dub the extended version as \algoal{}.

\stitle{Idea} Since in an attribute-less MRG $\G$, attribute matrix $\XM=\mathbf{0}$, our objective function in Eq.~\eqref{eq:obj-emb} then becomes
\begin{small}
\begin{equation*}
\begin{gathered}
\min_{\HM\in \NR,\ \omega_r\in \mathbb{R}} \alpha\cdot\mathcal{L}_{\textnormal{MRDE}} + \beta\cdot \sum_{r=1}^{R}\omega_r \cdot \|\NAM^{(r)}\|_F^2\quad \text{s.t.}  \quad \sum_{r=1}^R {\omega_r}=1,
\end{gathered}
\end{equation*}
\end{small}
consisting of two valid terms, MRDE and regularization. As per our analysis in Section~\ref{sec:BFAO}, $\mathcal{L}_{\textnormal{MRDE}}=\texttt{trace}(\HM^{\top}(\IM-\NAM)\HM)$, wherein $\NAM$ denotes the unified normalized adjacency matrix. Although we can analogously apply the alternating optimization scheme and update relation type weights $\{\omega_r\}_{r=1}^R$ efficiently as in Section~\ref{sec:FAAO}, the updating of node feature vectors $\HM$ is still problematic. 

Specifically, although the constraint $\HM\in \NR$ on $\HM$ can avoid trivial solutions to $\texttt{trace}(\HM^{\top}(\IM-\NAM)\HM)$, e.g., $\mathbf{0}$, the direct optimization with such a constraint undergoes numerous iterations of time-consuming projected gradient ascent steps.
As a workaround, the idea of \algoal{} is to impose an additional orthogonality constraint $\HM^\top\HM=\IM$ to $\HM$, thereby facilitating the problem transformation from minimizing $\texttt{trace}(\HM^{\top}(\IM-\NAM)\HM)$ to
\begin{equation*}
\max_{\HM^\top\HM=\IM}{ \texttt{trace}(\HM^{\top}\NAM\HM)}.
\end{equation*}
By Ky Fan's trace maximization principle~\cite{fan1949theorem}, the optimal $\HM$ to this problem is the first $d$ eigenvectors of $\NAM$, which can be efficiently computed via fast partial eigendecomposition solvers as $d\ll N$.





\stitle{Algorithm} As displayed in Algorithm~\ref{alg:emd-AL}, \algoal{} takes as input an attribute-less MRG $\G$, parameters $\alpha, \beta$, and the number $K$ of clusters. As Lines 1-2 in Algorithm~\ref{alg:emb}, Algorithm~\ref{alg:emd-AL} begins by initializing relation type weights $\{\omega_r\}_{r=1}^R$ and building matrix $\tilde{\EM}^{(r)}$. Afterwards, at Lines 3-7, \algoal{} iteratively updates node feature vectors $\HM$ and relation type weights. In each iteration, Algorithm~\ref{alg:emd-AL} computes the unified normalized adjacency matrix $\NAM$ by Eq.~\eqref{eq:A-sum} at Line 4, takes the first $K$ eigenvectors of $\NAM$ as $\HM$ at Line 5 through the {\em Arnoldi iterative solver}~\cite{lehoucq1996deflation}, followed by normalizing $\HM$ such that $\HM\in \NR$ at Line 6, respectively. Additionally, with $\HM$ and $\tilde{\EM}^{(r)}$ at hand, we update $\{\omega_r\}_{r=1}^R$ in the same way as in Algorithm~\ref{alg:emb} (Line 7). Eventually, the resulting node feature vectors $\HM$ after convergence will be input to \clustalgo{} (Algorithm~\ref{alg:clust}) to derive the final clusters $\{\C_1,\ldots,\C_K\}$.
 
\stitle{Complexity Analysis}
Lines 1-7 are identical to Algorithm~\ref{alg:emb} except for updating $\HM$ at Line 5, which involves a partial eigendecomposition of sparse matrix $\NAM$ and consumes $O(Md)$ time~\cite{lehoucq1996deflation}.
Combined with the cost analysis in Section~\ref{sec:FAAO}, the time overhead for generating $\HM$ in each iteration in the first stage is $O(Md+NdR)$.
Additionally, Algorithm~\ref{alg:emd-AL} invokes Algorithm~\ref{alg:clust} at Line 8 for the second stage. As per its cost analysis in Section~\ref{sec:SSKC}, the overall time complexity of \algoal{} is bounded by $O(Md+N(d^2+dR+K))$ when the numbers of iterations are regarded as constants.
The space overhead is the same as \algoplus{}, i.e., $O(M+N (d+K))$.
}

\section{Theoretical Proofs}\label{sec:proof}

\begin{lemma}[Lidskii Inequality~\cite{horn1962eigenvalues,lidskii1982spectral}]\label{le:lidin}
Suppose $\AM$ is a random matrix, and let $\lambda(\AM)$ denote the largest eigenvalue of $\AM$, For any Hermitian matrices $\AM$ and $\BM$, the following inequality holds:
\begin{equation*}
    \lambda(\AM + \BM) \leq \lambda(\AM) + \lambda(\BM)
\end{equation*}
\end{lemma}

\begin{lemma}\label{lem:OME}
$\|\NAM^{L+\ell}-\NAM^L\|_2=\mu_{L,L+\ell}.$
\end{lemma}

\begin{proof}[\bf Proof of Eq.~\eqref{eq:Ncut}]
Let $s_i=\|\SM_i\|_1$. By the definition of the DE, we can rewrite $\mathcal{D}(\YM,\SM)$ in Eq.~\eqref{eq:obj-clust} as follows:
\begin{align*}
\mathcal{D}(\YM,\SM) =& \frac{1}{2}\sum_{v_i,v_j\in \V}{\SM}_{i,j}\cdot\left\|{\YM_i}/{\sqrt{s_i}}-{\YM_j}/{\sqrt{s_j}}\right\|^2_2 \\
=& \frac{1}{2}\sum_{k=1}^K\sum_{v_i,v_j\in \V}{\SM}_{i,j}\cdot\left({\YM_{i,k}}/{\sqrt{s_i}}-{\YM_{j,k}}/{\sqrt{s_j}}\right)^2 \\
=& \frac{1}{2}\sum_{k=1}^K\sum_{v_i, v_j\in \C_k}{\frac{\SM_{i,j}}{|\C_k|}\cdot \left(\frac{1}{\sqrt{s_i}}-\frac{1}{\sqrt{s_j}}\right)}^2 \\
& + \frac{1}{2}\sum_{k=1}^K\sum_{v_i\in \C_k, v_j\in \V\setminus\C_k}{\SM_{i,j}\cdot \frac{1}{|\C_k|\cdot s_i}}.
 \end{align*}
 If we assume that $s_i=s_j\ \forall{v_i,v_j\in \V}$, we can derive that the minimization of $\mathcal{D}(\YM,\SM)$ is equivalent to minimizing $$\sum_{k=1}^K\sum_{v_i\in \C_k, v_j\in \V\setminus\C_k}{\frac{\SM_{i,j}}{|\C_k|}}.$$
\end{proof}

\begin{proof}[\bf Proof of Lemma~\ref{eq:eigval-bound}]
 Consider a vector $\mathbf{x} \in \mathbb{R}^n$ such that $\mathbf{x}_i \neq 0$ for all $i \in \{1, 2, \dots, n\}$. By the Courant-Fischer Theorem, we have:
\begin{align*}
    \lambda (\NAM^{(r)}) = \frac{\mathbf{x}^\top \NAM^{(r)} \mathbf{x}}{\mathbf{x}^\top \mathbf{x}}.
\end{align*}
Let $\mathbf{y} = {\DM^{(r)}}^{-\frac{1}{2}} \mathbf{x}$. Substituting this into the above expression, we obtain: 
\begin{align*}
    \lambda (\NAM^{(r)}) = \frac{\mathbf{y}^\top {\AM^{(r)}} \mathbf{y}}{\mathbf{y}^\top {\DM^{(r)}} \mathbf{y}}.
\end{align*}
For any vector $y$, applying the Cauchy-Schwarz inequality yields: 
\begin{align*}
    \mathbf{y}^\top {\AM^{(r)}} \mathbf{y} &= \sum_{i,j} \AM^{(r)}_{ij} \mathbf{y}_i \mathbf{y}_j \leq \frac{1}{2} \sum_{i,j} \AM^{(r)}_{ij} \left( \mathbf{y}_i^2 + \mathbf{y}_j^2 \right) \\
    & = \sum_i d_i \mathbf{y}_i^2 = \mathbf{y}^\top \DM^{(r)} \mathbf{y}.
\end{align*}
From this, we conclude that $\lambda(\NAM^{(r)}) \leq 1$.

Next, observe that:
\begin{align*}
    \NAM = \sum_{r=1}^{R} \omega_r \NAM^{(r)} \quad \Rightarrow \quad \lambda(\NAM) = \lambda\left(\sum_{r=1}^{R} \omega_r \NAM^{(r)}\right).
\end{align*}
Since each $\NAM^{(r)}$ is a symmetric normalized positive definite matrix, it follows that $\NAM^{(r)} = {\NAM^{(r)}}^\top$ and $x^\top \NAM^{(r)} x \geq 0$ for any $x$. Thus, $\NAM^{(r)}$ is Hermitian. As $\NAM$ is a weighted sum of Hermitian matrices, it is also Hermitian. By Lemma \ref{le:lidin}, we have:
\begin{align*}
    \lambda\left(\sum_{r=1}^{R} \omega_r \NAM^{(r)}\right) \leq \sum_{r=1}^{R} \omega_r \lambda(\NAM^{(r)}) \leq \sum_{r=1}^{R} \omega_r = 1.
\end{align*}
This completes the proof.
\end{proof}
\begin{proof}[\bf Proof of Lemma~\ref{lem:opt-H}]
By setting its derivative w.r.t. $\HM$ to zero and , we obtain the optimal $\HM$ as:
\begin{align}
& \frac{\partial{\{\alpha\cdot\textsf{trace}(\HM^{\top}(\IM-\NAM)\HM)+\|\HM - \XM\|^2_F \}}}{\partial{\HM}}=0 \notag\\
& \Longrightarrow \alpha\cdot(\IM-\NAM)\HM +  (\HM - \XM) = 0 \notag\\
& \Longrightarrow ((1+\alpha)\IM - \alpha\cdot\NAM) \cdot \HM = \XM \notag\\
& \Longrightarrow (\IM - \frac{\alpha}{1+\alpha}\cdot \NAM) \cdot \HM = \frac{1}{1+\alpha}\XM \notag\\
& \Longrightarrow \HM = \frac{1}{1+\alpha}\cdot \left(\IM-\frac{\alpha}{1+\alpha}\NAM\right)^{-1} \XM. \label{eq:Z-derivative},
\end{align}
which seals the proof.
\end{proof}

\begin{proof}[\bf Proof of Eq~\eqref{eq:update-w}]

Assume $\HM$ is fixed during the adjustment of $\omega_r$. Let
$$
c_r = \beta \cdot\|\NAM^{(r)}\|_F^2 + \alpha \cdot \textsf{trace}\left(\HM^{\top}(\IM-\NAM^{(r)})\HM\right) \geq 0,
$$
which simplifies the objective function to $\sum_{r=1}^{R}\omega_r c_r$.

Applying the Cauchy-Schwarz inequality:
$$
\left( \sum_{r=1}^{R}\omega_r c_r \right) \left( \sum_{r=1}^{R}\frac{1}{c_r} \right) \geq \left( \sum_{r=1}^{R}\sqrt{\omega_r c_r} \cdot \frac{1}{\sqrt{c_r}} \right)^2 = \left( \sum_{r=1}^{R}\sqrt{\omega_r} \right)^2 \geq 1.
$$
Equality holds if and only if $\sqrt{\omega_r c_r} \propto \frac{1}{\sqrt{c_r}}$, i.e., $\omega_r = p \cdot c_r^{-2}$ for some constant $p$. 
With the constraint $\sum_{r=1}^{R}\omega_r = 1$ 
, we can easily get $p$:
$$
p = \frac{1}{\sum_{r=1}^{R}c_r^{-2}}.
$$
Substituting $p$ into $\omega_r = p \cdot c_r^{-2}$
we can get $\omega_r=\frac{c_r^{-2}}{\sum_{i=1}^{R}c_i^{-2}}$, which completes the proof.
\end{proof}

\begin{proof}[\bf Proof of Lemma~\ref{lem:obj-clust}]
Let $s_i=\|\SM_i\|_1$. We can expand $\mathcal{D}(\YM, \SM)$ as follows:
\begin{align*}
\mathcal{D}(\YM, \SM) &= \frac{1}{2}\sum_{v_i,v_j\in \V}{\SM_{i,j}\left\|{\YM_i}/{\sqrt{s_i}}-{\YM_j}/{\sqrt{s_j}}\right\|^2_2}\\
&= \sum_{k=1}^{K}\frac{1}{2}\sum_{v_i,v_j\in \V}{\SM_{i,j}\cdot \left(\frac{\YM_{i,k}}{\sqrt{s_i}}-\frac{\YM_{j,k}}{\sqrt{s_j}}\right)^2}\\
&= \sum_{k=1}^{K}{\YM_{\cdot,k}^\top (\IM-\SM)\YM_{\cdot,k}}\\
& = \texttt{trace}(\YM^{\top}(\IM-\SM)\YM) = \texttt{trace}(\YM^{\top}\YM)-\texttt{trace}(\YM^{\top}\SM\YM).
\end{align*}
By the definition of $\YM$ in Eq.~\eqref{eq:NCI}, $\YM^{\top}\YM=\IM$, which is a constant. Thus, the minimization of \(\mathcal{D}(\YM, \SM)\) is equivalent to the maximization of \(\texttt{trace}(\YM^{\top}\SM\YM)\).
\end{proof}

\eat{
\begin{proof}[\bf Proof of Corollary~\ref{lem:opt-H-new}]
\begin{lemma}[\cite{horn2012matrix}]\label{col:neu}
Let $\MM$ be a matrix whose dominant eigenvalue $\lambda$ satisfies $|\lambda|<1$. Then, $\IM-\MM$ is invertible, and its inverse $(\IM-\MM)^{-1}$ can be expanded as a Neumann series: $(\IM-\MM)^{-1}=\sum_{\ell=0}^\infty\MM^\ell$.
\end{lemma}

By the property of Neumann series, we have 
$\left(\IM-\frac{\alpha}{1+\alpha}\NAM\right)^{-1} = \sum_{k=0}^{\infty}{\left(\frac{\alpha}{1+\alpha}\right)^k \NAM^{k}}$. Plugging it into Eq.~\eqref{eq:Z-derivative} completes the proof.
\end{proof}
}

\begin{proof}[\bf Proof of Lemma~\ref{lem:trace-Fnorm}]
According to the definition of the oriental incidence matrix, we have $\DM^{(r)}-\AM^{(r)}=\EM^{(r)}{\EM^{(r)}}^\top$. Hence,
\begin{small}
\begin{align*}
\texttt{trace}\left(\HM^{\top}(\IM-\NAM^{(r)})\HM\right) & = \texttt{trace}\left(\HM^{\top}\DM^{{(r)}-\frac{1}{2}}(\DM-\AM^{(r)})\DM^{{(r)}-\frac{1}{2}}\HM\right) \\
& = \texttt{trace}\left(\HM^{\top}\DM^{{(r)}-\frac{1}{2}}\EM^{(r)}{\EM^{(r)}}^\top\DM^{{(r)}-\frac{1}{2}}\HM\right)\\
& = \texttt{trace}\left(\HM^{\top}\hat{\EM}^{(r)}{\hat{\EM}^{(r) \top}}\HM\right) = \|\HM^{\top}\hat{\EM}^{(r)}\|^2_F,
\end{align*}
\end{small}
which completes the proof.
\end{proof}

\begin{proof}[\bf Proof of Theorem~\ref{lem:H-Hprime}]
According to Lines 5-8, we have
\begin{align*}
\HM =& \widehat{\XM}^{(0)} + \frac{\alpha}{1+\alpha}\cdot \NAM \widehat{\XM}^{(0)} + \left(\frac{\alpha}{1+\alpha}\right)^2\cdot \NAM^2 \widehat{\XM}^{(0)}+\ldots\\
& + \left(\frac{\alpha}{1+\alpha}\right)^L\cdot \NAM^L \widehat{\XM}^{(0)} + \alpha\cdot \left(\frac{\alpha}{1+\alpha}\right)^L\cdot \NAM^L \widehat{\XM}^{(0)}\\
=& \frac{1}{1+\alpha}\sum_{\ell=0}^{L}{\left(\frac{\alpha}{1+\alpha}\right)}^\ell\NAM^\ell\XM + \left(\frac{\alpha}{1+\alpha}\right)^{L+1}\NAM^L\XM,
\end{align*}
which is exactly Eq.~\eqref{eq:Hprime}.
By the definition of $\HM^{\ast}$ in Eq.~\eqref{eq:comp-H-new} and the Frobenius norm and operator norm inequality,
\begin{align*}
\|\HM-\HM^{\ast}\|_F & = \left\|\frac{1}{1+\alpha}\sum_{\ell=L+1}^{\infty}{\left(\frac{\alpha}{1+\alpha}\right)}^\ell\cdot(\NAM^\ell-\NAM^L)\cdot\XM\right\|_F \\
& \le \frac{1}{1+\alpha}\sum_{\ell=L+1}^{\infty}{\left(\frac{\alpha}{1+\alpha}\right)}^\ell\cdot\left\|(\NAM^\ell-\NAM^L)\cdot\XM\right\|_F \\
& \le \frac{1}{1+\alpha}\sum_{\ell=L+1}^{\infty}{\left(\frac{\alpha}{1+\alpha}\right)}^\ell\cdot\left\|\NAM^{\ell}-\NAM^L\right\|_2\cdot\|\XM\|_F \\
& \le \frac{1}{1+\alpha}\sum_{\ell=1}^{\infty}{\left(\frac{\alpha}{1+\alpha}\right)}^{L+\ell}\cdot\left\|\NAM^{L+\ell}-\NAM^L\right\|_2\cdot\|\XM\|_F.
\end{align*}

By Lemma~\ref{lem:OME},
\begin{align*}
\|\HM-\HM^{\ast}\|_F & \le \frac{1}{1+\alpha}\sum_{\ell=1}^{\infty}{\left(\frac{\alpha}{1+\alpha}\right)}^{L+\ell}\cdot \mu_{L,L+\ell} \cdot\|\XM\|_F \\
& = \frac{1}{1+\alpha}\cdot \left(\frac{\alpha}{1+\alpha}\right)^L \cdot \sum_{\ell=1}^{\infty}{\left(\frac{\alpha}{1+\alpha}\right)}^{\ell} \cdot\|\XM\|_F \cdot \max_{\ell \ge 1}{\mu_{L,L+\ell}}\\
& = \left(\frac{\alpha}{1+\alpha}\right)^{L+1}\cdot \|\XM\|_F\cdot \max_{\ell \ge 1}{\mu_{L,L+\ell}}.
\end{align*}
This completes the proof.
\end{proof}

\eat{
\begin{proof}[\bf Proof of Corollary~\ref{lem:sketch-approximation}]
\renchi{to-do}
It's easy to get that:
\begin{align*}
    \left\| \HM^{\top} \hat{\EM}^{(r)} \right\|^2_F 
    &= \texttt{\textnormal{trace}}\left( \HM^{\top} \hat{\EM}^{(r)} (\hat{\EM}^{(r)})^\top \HM \right) \label{eq:trace_original}
\end{align*}

Let $\BM = \HM^{\top} \hat{\EM}^{(r)} \in \mathbb{R}^{d \times m}$. By Theorem 4.1 in \cite{}, for any $i,j \in \V$:
\begin{equation*}\label{eq:sketch-prob}
    \mathbb{P}\left( \left| (\BM_{i,:} \PM^\top)(\PM \BM_{:,j}) - \BM_{i,:} \BM_{:,j} \right| < \epsilon \right) \geq 1 - \delta
\end{equation*}
Where $\epsilon$ and $\delta$ are infinitesimal values. 

This implies the inner product preservation:
\begin{equation*}
    \BM \BM^{\top} \approx \BM \PM^{\top}\PM\BM^{\top}= \HM^{\top} \tilde{\EM}^{(r)} (\tilde{\EM}^{(r)})^\top \HM
\end{equation*} 

Substituting into the trace operation:
\begin{align}
    \texttt{\textnormal{trace}}\left( \HM^{\top} \hat{\EM}^{(r)} (\hat{\EM}^{(r)})^\top \HM \right) 
    &\approx \texttt{\textnormal{trace}}\left( \HM^{\top} \tilde{\EM}^{(r)} (\tilde{\EM}^{(r)})^\top \HM \right) \\
    &= \left\| \HM^{\top} \tilde{\EM}^{(r)} \right\|^2_F
\end{align}
\renchi{see Theorem 4.1 in \url{https://arxiv.org/pdf/2406.05482}}
\end{proof}
}

\begin{proof}[\bf Proof of Theorem~\ref{lem:clust-trans}]
Let $\mathcal{J}={\sum_{k=1}^K\sum_{v_i\in \C_k}{\|\ZM_i-\mathbf{c}^{(k)}\|}^2}$, and we can compute $\mathcal{J}$ as follows:
\begin{align*}
\mathcal{J} &= \sum_{k=1}^K\sum_{v_i\in \C_k}\left(\ZM_i\ZM_i^\top - 2\ZM_i\mathbf{c}^{(k)\top}+ \mathbf{c}^{(k)}\mathbf{c}^{(k)\top}\right) \\
&= \sum_{i=1}^{|\V|}\ZM_i\ZM_i^\top - 2\sum_{k=1}^K|\C_k|\mathbf{c}^{(k)\top}\mathbf{c}^{(k)} + \sum_{k=1}^K|\C_k|\mathbf{c}^{(k)}\mathbf{c}^{(k)\top} \\
&= \sum_{i=1}^{|\V|}\ZM_i\ZM_i^\top - \sum_{k=1}^K|\C_k|\mathbf{c}^{(k)}\mathbf{c}^{(k)\top}.
\end{align*}

Since we have $\mathbf{c}^{(k)}=\frac{1}{|\C_k|}\sum_{v_j\in \C_k}{\ZM_j}$:
\begin{align*}
|\C_k|^2\mathbf{c}^{(k)}\mathbf{c}^{(k)\top} &= \left(\sum_{v_i\in \C_k}\ZM_i\right)\left(\sum_{v_j\in \C_k}\ZM_j^\top\right)
&= \sum_{v_i,v_j\in \C_k}\ZM_i\ZM_j^\top.
\end{align*}
This allows us to rewrite $\mathcal{J}$ as:
\begin{align*}
\mathcal{J} &= \sum_{i=1}^{|\V|}\ZM_i \ZM_i^\top - \sum_{k=1}^K\frac{1}{|\C_k|}\sum_{v_i,v_j\in \C_k}\ZM_i\ZM_j^\top 
\end{align*}
Since $\SM=\ZM^{\top}\ZM$ , we can get that:
$$
\sum_{k=1}^K\frac{1}{|\C_k|}\sum_{v_i,v_j\in \C_k} \SM_{i,j}=\texttt{trace}(\YM^\top \SM \YM)
$$
So we can compute $\mathcal{J}$ by $\SM$ using following function:
$$
\mathcal{J}=\texttt{trace}(\SM) - \texttt{trace}(\YM^\top\SM\YM)
$$
where is a NCI $\YM$ satisfying $\YM\YM^{\top}=\IM$ and $\YM\YM^{\top} \mathbb{1}=\mathbb{1}$.
Thus, we establish the equivalence:
\begin{align*}
\min_{\C_1,\ldots,\C_K}\mathcal{J} \Leftrightarrow \max_{\YM}\texttt{trace}(\YM^\top\SM\YM),
\end{align*}

 By Lemma~\ref{lem:obj-clust}, this confirms the equivalence between optimizing Eq.~\eqref{eq:obj-clust} and  ${\sum_{k=1}^K\sum_{v_i\in \C_k}{\|\ZM_i-\mathbf{c}^{(k)}\|}}$.
\end{proof}

\begin{proof}[\bf Proof of Theorem~\ref{lem:doubly-sto}]
\eat{
    Suppose $\SM=\hat{\ZM}\cdot\hat{\ZM}^T$,when $\hat{\ZM}$ is Doubly Stochastic matrix \renchi{why upper case}, we can know that $\sum_{i=1}^{n}\hat{\ZM}_{ij}=1$ and $\sum_{j=1}^{d}\hat{\ZM}_{ij}=1$ \renchi{$\hat{\ZM}_{ij}$ -> $\hat{\ZM}_{i,j}$. Be careful with notations and be rigorous.}
    The i-th row of $\SM$ can be calculated by: \renchi{i-th row -> $i$-th row. what is $\SM_{i\_}$? a row is 1?}
    \begin{equation}
    \SM_{i\_} = \sum_{j=1}^{n} \sum_{k=1}^{d} \hat{\ZM}_{ik} \hat{\ZM}_{jk}
    = \sum_{k=1}^{d} \hat{\ZM}_{ik} \sum_{j=1}^{n} \hat{\ZM}_{jk}
    = 1
\end{equation}
    As $\SM$ is a symmetric matrix, it is not difficult to deduce that the sum of its column vectors is also 1.
    To sum up, if $\hat{\ZM}$ is Doubly Stochastic matrix, then $\hat{\ZM}\cdot\hat{\ZM}^T$ is also Doubly Stochastic matrix.
}
Denote by $\overleftarrow{\mathbf{v}}^{(\ell)}$ (resp. $\overrightarrow{\mathbf{v}}^{(\ell)}$) the row (resp. column) sum vector $\mathbf{v}$ at Line $4$ (resp. Line 6) in the $\ell$-th iteration. Suppose that \clustalgo{} terminates the iterative process in the $T$-th iteration. At the end of the $T$-th iteration, we have
\begin{equation*}
\overleftarrow{\ZM} = \prod_{\ell=1}^T{\texttt{diag}(\overleftarrow{\mathbf{v}}^{(\ell)})^{-1}}\cdot \ZM^\circ\ \text{and}
\end{equation*}
\begin{equation*}
\overrightarrow{\ZM} = \prod_{\ell=1}^T{\texttt{diag}(\overrightarrow{\mathbf{v}}^{(\ell)})^{-1}}\cdot \ZM^\circ,
\end{equation*}
leading to
\begin{equation*}
\overleftarrow{\ZM}\overrightarrow{\ZM}^\top = \prod_{\ell=1}^T{\texttt{diag}(\overleftarrow{\mathbf{v}}^{(\ell)})^{-1}}\cdot \left(\ZM^\circ{\ZM^{\circ}}^\top\right) \cdot \prod_{\ell=1}^T{\texttt{diag}(\overrightarrow{\mathbf{v}}^{(\ell)})^{-1}}.
\end{equation*}
This result is equivalent to the {\em Iterative Proportional Fitting Procedure} in the Sinkhorn-Knopp algorithm, and using the Birkhoff-von Neumann theorem, we can conclude that $\overleftarrow{\ZM}\overrightarrow{\ZM}^\top$ is doubly stochastic~\cite{knight2008sinkhorn}.

Since $\ZM^\circ{\ZM^{\circ}}$ is a non-negative square matrix, according to Sinkhorn's theorem~\cite{sinkhorn1967concerning}, $\prod_{\ell=1}^T{\texttt{diag}(\overleftarrow{\mathbf{v}}^{(\ell)})^{-1}}$ and $\prod_{\ell=1}^T{\texttt{diag}(\overrightarrow{\mathbf{v}}^{(\ell)})^{-1}}$ are unique modulo multiplying the first matrix by a positive number and dividing the second one by the same number. By the symmetry of $\ZM^\circ{\ZM^{\circ}}$ and $\overleftarrow{\ZM}\overrightarrow{\ZM}^\top$,
\begin{equation*}
\overleftarrow{\ZM}\overrightarrow{\ZM}^\top = \prod_{\ell=1}^T{\texttt{diag}(\overrightarrow{\mathbf{v}}^{(\ell)})^{-1}} \cdot \left(\ZM^\circ{\ZM^{\circ}}^\top\right) \cdot \prod_{\ell=1}^T{\texttt{diag}(\overleftarrow{\mathbf{v}}^{(\ell)})^{-1}},
\end{equation*}
and the uniqueness of the two scaling matrices, we can conclude that
\begin{equation*}
 \prod_{\ell=1}^T{\texttt{diag}(\overleftarrow{\mathbf{v}}^{(\ell)})^{-1}} = \prod_{\ell=1}^T{\texttt{diag}(\overrightarrow{\mathbf{v}}^{(\ell)})^{-1}},
\end{equation*}
The theorem is then proved.
\end{proof}

\begin{proof}[\bf Proof of Lemma~\ref{lem:OME}]
By the definition of $\|\NAM^{L+\ell}-\NAM^L\|_2$, $\|\NAM^{L+\ell}-\NAM^L\|_2=\sigma_{\max}\left(\NAM^{L+\ell}-\NAM^L\right)$, i.e., the maximum singular value of $\NAM^{L+\ell}-\NAM^L$.

Further, let $\VM\texttt{diag}(\boldsymbol{\lambda})\VM^\top$ be the full eigendecomposition of $\NAM$, wherein eigenvalue  $\boldsymbol{\lambda}_{i}=\lambda_i(\NAM) \ \forall{1\le i\le N}$. Using the semi-unitary property of $\VM$, i.e., $\VM^\top\VM=\IM$, we have $\NAM^{L+\ell}=\VM\texttt{diag}(\boldsymbol{\Lambda})^{\ell+L}\VM^{\top}$ and $\NAM^{L}=\VM\texttt{diag}(\boldsymbol{\Lambda})^{L}\VM^{\top}$. This leads to $\NAM^{L+\ell}-\NAM^L = \VM(\texttt{diag}(\boldsymbol{\Lambda})^{\ell+L}-\texttt{diag}(\boldsymbol{\Lambda})^{L})\VM^{\top}$. 
\begin{align*}
\sigma_{\max}\left(\NAM^{L+\ell}-\NAM^L\right) & = \max_{1\le i\le N}{|\boldsymbol{\lambda}_{i}^{L+\ell}-\boldsymbol{\lambda}_{i}^{L}|},
\end{align*}
which finishes the proof.
\end{proof}
\section{Additional Algorithmic Details}\label{sec:algo-detail}

\eat{
\begin{algorithm}[!t]
\caption{\texttt{SNEM-Rounding}~\cite{yang2024efficient}}
\label{alg:round}
\KwIn{The first $K$ eigenvectors $\UM$}
\KwOut{A set $\{\C_1,\C_2,\ldots,\C_K\}$ of $K$ clusters.}
$\YM \gets \UM$\;
\Repeat{$\YM$ converges}{
    $\RM \gets \UM^{\top}\YM,\ \YM \gets \mathbf{0}$\;
    \For{$v_i\in \V$}{
        $j^{*}= \arg\max_{1\le j\le K}{(\UM\RM)_{i,j}}$\;
        $\YM_{i,j^{*}}\gets 1$\;
    }
    Normalize $\YM$ such that $\YM\in \NC$\;
}
Convert $\YM$ into $K$ clusters $\{\C_1,\C_2,\ldots,\C_K\}$\;
\end{algorithm}
}

\subsection{The \texttt{CountSketch} Algorithm}
Algorithm~\ref{alg:count-sketch} displays the pseudo-code of \texttt{CountSketch} Algorithm, at the beginning, it need to generate the oriented incidence matrix $\EM^{(r)}\in \mathbb{R}^{N \times 2M^{(r)}}$ for $\NAM^{(r)}$ (Line 1), and then, in Line 2 we normalized $\EM^{(r)}$ so that we can get $\hat{\EM}^{(r)}$ which can eatimate\\ $\texttt{trace}\left(\HM^{\top}(\IM-\NAM^{(r)})\HM\right)$, and then we can get count-sketch matrix by following equation:
\begin{equation}\label{eq:countsketch}
\tilde{\EM}^{(r)}[k, j] = \sum_{\substack{i=1 \\ h_k(i) = j}}^{n} s_k(i) \cdot \hat{\EM}^{(r)}[i,:] 
\end{equation}
Where $h_k=\{1, 2, \ldots, n\} \to \{1, 2, \ldots, t\}$ is the random hash function, and $s_k= \{1, 2, \ldots, n\} \to \{\pm 1\}$ is the $k$-th Rademacher sign function.

\begin{algorithm}[!t]
\caption{\texttt{CountSketch} Algorithm}\label{alg:count-sketch}
\KwIn{Normalized oriented incidence matrix $\hat{\EM} \in \{0,1\}^{n \times M}$, Target dimension $k$}
\KwOut{Sketch matrix $\tilde{\EM} \in \mathbb{R}^{n \times m}$}

Initialize hash function $h: \{1,\dots,n\} \to \{1,\dots,k\}$ with uniform randomness\;
Initialize diagonal sign matrix $\Delta \in \{-1,+1\}^{M \times M}$ with $\Delta_{i,i} \sim \text{Rademacher}$\;
Construct sparse bucket matrix $\Phi \in \{0,1\}^{m \times M}$ where $\Phi_{j,i} = \mathbf{1}_{[h(i)=j]}$\;
Compute combined projection matrix $\RM \gets \Phi\Delta$\;
$\tilde{\EM} \gets \hat{\EM}\RM^{\top}$\;

\end{algorithm}
\subsection{The \texttt{ORF} Algorithm}
Here, we describe the details of Orthogonal Random Features (\texttt{ORF}) algorithm. First, we generate a Gaussian random matrix $\mathbf{W} \in \mathbb{R}^{N \times d}$ (Line 1), followed by performing a QR decomposition of it to obtain the orthogonal matrix $\mathbf{Q}$ (Line 2). Finally, we use the following formula to derive $\mathbf{Z}^{\circ}$:
\begin{equation}\label{eq:R-prime}
{\ZM}^{\circ} = \sqrt{\frac{2}{d}}\cdot \left(sin(\ZM) \mathbin\Vert cos(\ZM)\right)\in \mathbb{R}^{N \times 2d},
\end{equation}
Where $\mathbin\Vert$ represent horizontal
concatenation operator for matrices. 
\begin{algorithm}[!t]
\caption{\texttt{ORF}}\label{alg:orf}
\KwIn{Node feature vectors $\HM$, Feature dimension $d$}
\KwOut{${\ZM}^{\circ}$}
Sample a Gaussian random matrix $\WM\in \mathbb{R}^{d \times d}$\;
Compute $\QM$ by a QR decomposition over $\WM$\;
$\ZM\leftarrow \sqrt{d}\cdot\HM \QM^{\top}$\;
Compute ${\ZM}^{\circ}$ according to Eq. \eqref{alg:orf}\;
\end{algorithm}

\begin{figure}[H]
\centering
\includegraphics[width=0.8\columnwidth]{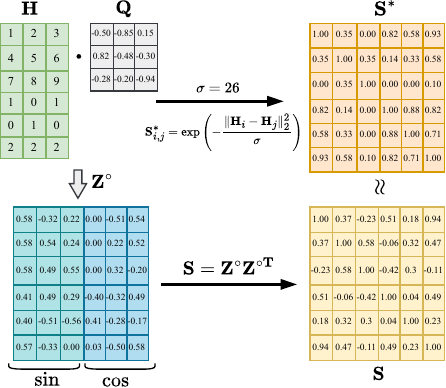}
\vspace{-3ex}
\caption{{Running example for \texttt{ORF}.}}
\label{fig:rff}
\end{figure}

\subsection{Illustrative Example for \texttt{ORF}}\label{sec:example-ORF}

\textcolor{black}{
In Fig~\ref{fig:rff}, the feature matrix $\HM \in \mathbb{R}^{6\times 3}$ is first multiplied by an orthogonal random matrix $\QM$, after that, the first row of $\HM$ becomes $[-0.2,-1.36,-3.27]$. Then, the mapping functions $\texttt{sin}$ and 
$\texttt{cos}$ are applied to this feature matrix, to be more precise, the first row of the multiplied feature matrix becomes $[0.58,-0.32,0.22]$ and $[0.0,-0.51,0.54]$ after computing by $\texttt{sin}$ and 
$\texttt{cos}$.
Then we horizontally connect the mapped features to obtain ${\ZM}^{\circ}$ . The matrix $\SM$
 obtained by ${\ZM}^{\circ}{\ZM}^{\circ \top}$ is closely resembles to the matrix $\SM^{\ast}$ given by Eq. \eqref{eq:edge-weight}. We can observe that in the first row of $\SM$, the largest element except for $\SM_{1,1}$ is $\SM_{1,6}=0.93$, and the smallest element is $\SM_{1,3}=0.0$. Similarly, in the first row of $\SM^{\ast}$, the largest element except for $\SM_{1,1}^{\ast}$ is $\SM_{1,6}^{\ast}=0.94$, and the smallest element is $\SM_{1,3}^{\ast}=-0.23$, that is to say, the overall distributions of the two matrices are similar. Nevertheless, the error between the two matrices is still relatively large, which is mainly because the dimension of $\HM$ ($d=3$) in the example is too small to well approximate the infinite-dimensional kernel function.  
}

\section{Additional Experimental Settings and Results}\label{sec:add-exp}

\begin{table}[!t]
\small
\centering
\caption{Parameter setting in \algoplus{}}
\label{tab:paramsdemm+}
\vspace{-3ex}
\setlength{\tabcolsep}{1.5pt}  
\resizebox{\columnwidth}{!}{%
\begin{tabular}{@{}c*{9}{c}@{}}
\toprule
 & \multicolumn{9}{c}{Datasets} \\
\cmidrule(l{1pt}r{1pt}){2-10}
Parameter & {\em ACM} & {\em DBLP} & {\em ACM2} &{\em YELP}  & {\em IMDB} & {\em MAG} & {\em OAG-CS} & {\em OAG-ENG} &{\em RCDD}  \\
\midrule
$\alpha$   &  4 & 28 &  4 & 32 &  7 & 50 & 110 & 120 &  4 \\
$\beta$    & 2.5 & 40 & 4.2 & 3 & 6 & 30 & 90 & 120 & 1.5 \\
$L$        &  5 &  6 &  3 & 16 & 13 & 14 &  12 &  16 &  4 \\
$d$        & 128 & 64 & 512 & 32 & 1024 & 32 & 128 & 128 & 128 \\
$m$        & $(10,14)$ & $(10,8,10)$ &$10$  &$(14,12,16)$  &$16$  &$12$  &$36$  &$40$  &$40$  \\
\bottomrule
\end{tabular}
}
\end{table}

\begin{table}[!t]
\small
\centering
\caption{Parameter setting in \algoal{}}
\label{tab:paramsdemmal}
\vspace{-3ex}
\setlength{\tabcolsep}{3.5pt}
\resizebox{\columnwidth}{!}{%
\begin{tabular}{@{}c*{9}{c}@{}}
\toprule
 & \multicolumn{9}{c}{Datasets} \\
\cmidrule(l{3pt}r{3pt}){2-10}
Parameter & {\em ACM} & {\em DBLP} & {\em ACM2} & {\em YELP} & {\em IMDB} & {\em MAG} & {\em OAG-CS} & {\em OAG-ENG} & {\em RCDD} \\
\midrule
$d$      &  6 &  4 &  4 &  3 & 80 & 30 & 68 & 62 &  8 \\
$\beta$  &  2 & 25 &  2 & 24 & 10 & 50 & 280 & 340 &  4 \\
$m$      & $10$ & $10$ & $10$ & $10$  & $16$ &$5$& $36$ &$36$ & $40$ \\
\bottomrule
\end{tabular}
}
\end{table}

\begin{table}[!t]
\small
\centering
\caption{Parameter setting in \algo{}}
\label{tab:paramsdemm}
\vspace{-3ex}
\setlength{\tabcolsep}{3.5pt}
\begin{tabular}{@{}c*{6}{c}@{}}
\toprule
 & \multicolumn{6}{c}{Datasets} \\
\cmidrule(l{3pt}r{3pt}){2-7}
Parameter & {\em ACM} & {\em DBLP} & {\em ACM2} & {\em YELP} & {\em IMDB} & {\em MAG}  \\
\midrule
$\alpha$ & 2 & 1900 & 1.5 & 26 & 6 & 50  \\
$\beta$  & 2 & 4200 & 2 & 50 & 8 & 6 \\
\bottomrule
\end{tabular}
\end{table}

\subsection{Datasets} We describe the details of each dataset used in the experiments in what follows:
\begin{itemize}
    \item {\em ACM}~\cite{Fan2020One2MultiGA} contains a paper collaboration network of $3{,}025$ publications with two relational edges: paper-subject connections (shared research subjects) and paper-author connections (shared authorship). Node features are bag-of-words representations of paper abstracts. Ground-truth labels classify publications into three research domains: database, wireless communication, and data mining.
    
    \item {\em DBLP}~\cite{ZhaoWSLY20} contains an academic collaboration network of $4{,}057$ papers with three relational edges: author-paper connections (co-authorship), paper-conference associations (shared venues), and paper-term linkages (shared technical terms). Node features are bag-of-words representations of paper abstracts. Ground-truth labels classify publications into four categories: database, data mining, machine learning, and information retrieval.
    
    \item {\em ACM2}~\cite{Fu2020MAGNNMA} contains an enhanced paper network of $4{,}019$ publications with two relational edges: paper-subject connections (subject-based) and paper-author interactions (author collaboration). Node features are bag-of-words representations of paper abstracts. Ground-truth labels classify publications into three academic domains: database, wireless communication, and data mining.
    
    \item {\em Yelp}~\cite{Shi2022RHINERS} contains a business interaction network of $2{,}614$ establishments with three relational edges: business-user interactions (shared customers), business-rating associations (common ratings), and business-service relationships (shared services). Node features are bag-of-words representations of rating descriptions. Ground-truth labels categorize businesses into three service types: Mexican flavor, hamburger, and food bar.
    
    \item {\em IMDB}~\cite{Wang2019HeterogeneousGA} contains a movie collaboration network of $3{,}550$ films with two relational edges: movie-actor connections (co-starring) and movie-director connections (shared directors). Node features are bag-of-words representations of movie plots. Ground-truth labels categorize films into three genres: Action, Comedy, and Drama.
\textcolor{black}{
\item {\em Amazon}~\cite{Shi2022RHINERS} comprises a product review network of $11{,}949$ users under the musical instrument category, with three types of relational edges: user‐product interactions (shared reviewed products), user‐star associations (identical star ratings within a week), and user‐review similarities (top 5\% review text similarity via TF‐IDF). Each user node is represented by a 25‐dimensional feature vector, encompassing attributes such as rating statistics, voting patterns, temporal activity, username length, and sentiment analysis of comments. The dataset provides a binary ground‐truth classiﬁcation for fraud detection.
\item {\em Protein}~\cite{GU2022106127} contains a protein interaction network of $18{,}877$ proteins, with three relational edge types: protein‐protein interactions (direct interactions), protein‐gene associations (shared genes), and protein‐disease associations (related diseases). Each protein node is represented by a $1{,}280$‐dimensional feature vector generated from its molecular sequence. Ground‐truth labels categorize proteins into six functional classes according to their biological roles.
}

    \item {\em MAG}~\cite{Hu2020OpenGB} contains a citation network of $113{,}919$ papers with two relational edges: paper-paper citations and paper-author connections (co-authorship). Node features are Word2Vec embeddings. Ground-truth labels classify publications into four research domains from the original dataset.
    
    \item {\em OAG-ENG} \& {\em OAG-CS}~\cite{Zhang2019OAGTL} contain academic citation networks with $370{,}623$ (engineering) and $546{,}704$ (computer science) papers respectively. Relational edges include citations, shared research fields, and shared authors. Node features are Word2Vec embeddings of paper keywords. Ground-truth labels preserve the 20 largest classes, with $77{,}768$ (OAG-ENG) and $50{,}247$ (OAG-CS) labeled nodes.
    
    \item {\em RCDD}~\cite{liu2023dink} contains an anonymized e-commerce network of $421{,}089{,}810$ items with relational connections (e.g., item-b-item). Node features are anonymized representations. Ground-truth labels provide a 9:1 imbalanced binary classification task with $122{,}487$ labeled nodes.
\end{itemize}

\input{tex/figs/appendix_NMI}

\input{tex/figs/appendix_ARI}

\subsection{Parameter Settings}
In this section, we introduce the parameters that we did not mention in the main text. Some parameters are fixed for each dataset since it did not make a big difference for the experiment results, e.g. the $\sigma$ in Eq.~\eqref{eq:edge-weight} is fixed as $1$ for all datasets, and the iteration rounds of \clustalgo{} is fixed as $2$ for all small datasets and $10$ for all large datasets. We perform exhaustive grid search over the parameter space of \algo{}, \algoplus{}, and \algoal{} to obtain optimal configurations, and analyze the influence of $\alpha$,$\beta$,$d$ and $L$ in Section~\ref{sec:para_set}, $m$ is the  dimension of $\tilde{\EM}^{(r)}$. 
In datasets with significant edge count disparity across relations (e.g., {\em ACM}), we set different $m$ for each relation.
All the parameters with the best performance are listed in Table~\ref{tab:paramsdemm+}, Table~\ref{tab:paramsdemmal} and Table~\ref{tab:paramsdemm}.
\subsection{Evaluation Metrics}
The specific mathematical definitions of {\em Clustering Accuracy} (ACC), {\em Normalized Mutual Information} (NMI), and {\em Adjusted Rand Index} (ARI) are as follows:
\begin{small}
\begin{equation*}
ACC = \frac{\sum_{u_i\in \V}{\mathbb{1}_{y_{u_i}=\texttt{map}(y^{\prime}_{u_i})}}}{|\V|},
\end{equation*}  
\end{small}
where $y^{\prime}_{u_i}$ and $y_{u_i}$ stand for the predicted and ground-truth cluster labels of node $u_i$, respectively, $\texttt{map}(y^{\prime}_{u_i})$ is the permutation function that maps each $y^{\prime}_{u_i}$ to the equivalent cluster label provided via Hungarian algorithm \cite{kuhn1955hungarian}, and the value of $\mathbb{1}_{y_{u_i}=\texttt{map}(y^{\prime}_{u_i})}$ is 1 if $y_{u_i}=\texttt{map}(y^{\prime}_{u_i})$ and 0 otherwise,
\begin{small}
\begin{equation*}
NMI = \frac{\sum_{i=1}^{k}\sum_{j=1}^{k}{|\C^{\ast}_i\cap \C_j|\cdot \log{\frac{|\C^{\ast}_i\cap \C_j|}{|\C^{\ast}_i|\cdot |\C_j|}}}}{\sqrt{\sum_{i=1}^k{|\C^{\ast}_i|\cdot \log{\frac{|{\C^{\ast}_i}|}{|\V|}}}}\cdot \sqrt{\sum_{i=1}^k{|\C_i|\cdot \log{\frac{|{\C_i}|}{|\V|}}}}},
\end{equation*}
and
\begin{equation*}
ARI=\frac{\sum_{i=1}^k\sum_{j=1}^k{\binom {|\C^{\ast}_i\cap \C_j|}2}-\left(\sum_{i=1}^k{\binom {|\C^{\ast}_i|}2}\cdot \sum_{j=1}^k{\binom {|\C_j|} 2}\right)/{\binom {|\V|}2}}{0.5\left(\sum_{i=1}^k{\binom {|\C^{\ast}_i|}2}+ \sum_{j=1}^k{\binom {|\C_j|}2} \right)- \left(\sum_{i=1}^k{\binom {|\C^{\ast}_i|}2}\cdot \sum_{j=1}^k{\binom {|\C_j|}2}\right)/{\binom {|\V|}2} },
\end{equation*}
\end{small}
where $\C^{\ast}_i$ and $\C_i$ represent the $i$-th ground-truth and predicted clusters for $\V$ in $\G$, respectively.

\begin{figure}[!t]
\centering
{\color{black}
\begin{small}
\begin{tikzpicture}
    \begin{customlegend}
    [legend columns=4,
        legend entries={{ACM}, {IMDB}, {MAG}, {OAG-ENG}},
        legend style={at={(0.45,1.35)},anchor=north,draw=none,font=\footnotesize,column sep=0.2cm}]
    \addlegendimage{line width=0.7mm,mark size=4pt,color=NSCcol1}
    \addlegendimage{line width=0.7mm,mark size=4pt,dashed,color=teal}
    \addlegendimage{line width=0.7mm,mark size=4pt,color=myred_new2}
    \addlegendimage{line width=0.7mm,mark size=4pt,dotted,color=NSCcol3}
    \end{customlegend}
\end{tikzpicture}
\\[-\lineskip]
\vspace{-1mm}

\subfloat[Varying $m$ in \algoplus{}]{
\begin{tikzpicture}[scale=1,every mark/.append style={mark size=3pt}]
    \begin{axis}[
        xmode=log,
        log basis x=2,
        height=\columnwidth/2.4,
        width=\columnwidth/1.8,
        xmin=9, xmax=330,
        ymin=0.39, ymax=0.95,
        xtick={10,20,40,80,160,320},
        xticklabels={10,20,40,80,160,320},
        ytick={0.4,0.5,0.6,0.7,0.8,0.9,0.94},
        yticklabels={0.4,0.5,0.6,0.7,0.8,0.9},
        xticklabel style = {font=\footnotesize},
        yticklabel style = {font=\footnotesize},
        legend style={fill=none,draw=none},
    ]
    \addplot[line width=0.7mm, smooth, color=NSCcol1]
        coordinates {(10, 0.936) (20, 0.935) (40, 0.935) (80, 0.934) (160, 0.935) (320, 0.932)};
    \addplot[line width=0.7mm, smooth, dashed, color=teal]
        coordinates {(10, 0.676) (20, 0.675) (40, 0.677) (80, 0.676) (160, 0.675) (320, 0.672)};
    \addplot[line width=0.7mm, smooth, color=myred_new2]
        coordinates {(10, 0.675) (20, 0.677) (40, 0.676) (80, 0.678) (160, 0.678) (320, 0.678)};
    \addplot[line width=0.7mm, smooth, dotted, color=NSCcol3]
        coordinates {(10, 0.398) (20, 0.421) (40, 0.423) (80, 0.422) (160, 0.422) (320, 0.424)};
    \end{axis}
\end{tikzpicture}
\label{fig:m-demm+}
}
\subfloat[Varying $m$ in \algo{}]{
\begin{tikzpicture}[scale=1,every mark/.append style={mark size=3pt}]
    \begin{axis}[
        xmode=log,
        log basis x=2,
        height=\columnwidth/2.4,
        width=\columnwidth/1.8,
        xmin=9, xmax=330,
        ymin=0.24, ymax=0.69,
        xtick={10,20,40,80,160,320},
        xticklabels={10,20,40,80,160,320},
        ytick={0.25,0.35,0.45,0.55,0.65,0.68},
        yticklabels={0.25,0.35,0.45,0.55,0.65},
        xticklabel style = {font=\footnotesize},
        yticklabel style = {font=\footnotesize},
        legend style={fill=none,draw=none},
    ]
    \addplot[line width=0.7mm, smooth, color=NSCcol1]
        coordinates {(10, 0.680) (20, 0.678) (40, 0.679) (80, 0.681) (160, 0.680) (320, 0.682)};
    \addplot[line width=0.7mm, smooth, dashed, color=teal]
        coordinates {(10, 0.388) (20, 0.382) (40, 0.386) (80, 0.385) (160, 0.389) (320, 0.387)};
    \addplot[line width=0.7mm, smooth, color=myred_new2]
        coordinates {(10, 0.634) (20, 0.632) (40, 0.637) (80, 0.634) (160, 0.633) (320, 0.635)};
    \addplot[line width=0.7mm, smooth, dotted, color=NSCcol3]
        coordinates {(10, 0.248) (20, 0.259) (40, 0.260) (80, 0.261) (160, 0.260) (320, 0.260)};
    \end{axis}
\end{tikzpicture}
\label{fig:m-demm}
}
\end{small}
\vspace{-3mm}
\caption{{\color{black}{Varying $m$ in \algo{} and \algoplus{}.}}}
\label{fig:m-parameter-both}
}
\end{figure}

\begin{figure}[!t]
\centering
{\color{black}
\begin{small}
\begin{tikzpicture}
    \begin{customlegend}
    [legend columns=4,
        legend entries={{ACM}, {IMDB}, {MAG}, {OAG-ENG}},
        legend style={at={(0.45,1.35)},anchor=north,draw=none,font=\footnotesize,column sep=0.2cm}]
    \addlegendimage{line width=0.7mm,mark size=4pt,color=NSCcol1}
    \addlegendimage{line width=0.7mm,mark size=4pt,dashed,color=teal}
    \addlegendimage{line width=0.7mm,mark size=4pt,color=myred_new2}
    \addlegendimage{line width=0.7mm,mark size=4pt,dotted,color=NSCcol3}
    \end{customlegend}
\end{tikzpicture}
\\[-\lineskip]
\begin{tikzpicture}[scale=1, every mark/.append style={mark size=3pt}]
    \begin{axis}[
        width=0.3\textwidth, 
        height=0.2\textwidth, 
        xmin=0, xmax=10,
        ymin=40, ymax=95,
        xtick={0.1,1,2,4,6,8,10},
        xticklabels={0.1,1,2,4,6,8,10},
        ytick={45,55,65,75,85,95},
        yticklabels={45,55,65,75,85,95},
        xticklabel style = {font=\footnotesize},
        yticklabel style = {font=\footnotesize},
    ]
    \addplot[line width=0.7mm, smooth, color=NSCcol1]
        coordinates {(0.1,93.1)  (1,93.2) (2,93.4) (4,93.4) (6,93.4) (8,93.4) (10,93.4)};
    \addplot[line width=0.7mm, smooth, dashed, color=teal]
        coordinates {(0.1,45.24)  (1,68.5) (2,64.3) (4,64.3) (6,64.3) (8,64.3) (10,64.3)};
    \end{axis}
\end{tikzpicture}
\label{fig:sigma-demm} 
\end{small}
\vspace{-3mm}
\caption{{\color{black}{Varying $\sigma$ in \algo{}.}}}
\label{fig:sigma-parameter}
}
\end{figure}
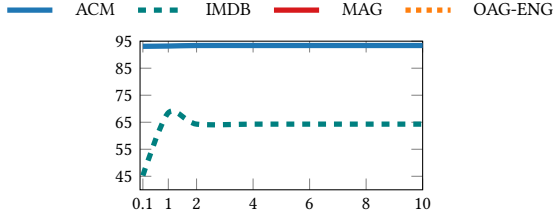

\subsection{Parameter Analysis}\label{sec:add-param-als}

We analyze the parameters for NMI and ARI, with results shown in Figures~\ref{fig:alpha-nmi}--\ref{fig:d-nmi} (NMI) and Figures~\ref{fig:alpha-ari}--\ref{fig:d-ari} (ARI).

The variation trends of NMI and ARI closely align with ACC across most datasets. In the majority of cases, these metrics attain their optimal values under consistent conditions, e.g., the ACC, NMI, and ARI metrics of {\em ACM} all achieve their maximum values at $L=5$. However, in rare cases, parameter configurations maximizing NMI/ARI differ slightly from those optimizing ACC, e.g., NMI and ARI of {\em MAG} peak at $L=16$, while ACC get the highest score when $L=14$. In such conflicting situations, we adopt ACC as the decisive criterion for performance evaluation.

\textcolor{black}{
We employ the \texttt{CountSketch} method to the approximate normalized
oriented incidence matrix $\hat{\EM}$ as $\tilde{\EM}$. According to Corollary~\ref{lem:sketch-approximation}, selecting an appropriate sketch size $m$ can effectively minimize the approximation error, we can minimize the approximation error. From Fig.~\ref{fig:m-parameter-both} and Fig.~\ref{fig:sigma-parameter}, for small and medium datasets {\em ACM}  ,{\em IMDB} and {\em MAG}, when $m$ is greater than $10$, the results keep invariant when increase $m$. For large dataset {\em OAG-ENG} with with abundant edges, the results keep unchanged when $m>40$. 
}

\textcolor{black}{
Due to the time and space complexity limitations of \algo{} ($\mathcal{O}(N^3)$ and $\mathcal{O}(N^2)$), we conduct $\sigma$ analysis only on two relatively small datasets {\em ACM}  and {\em IMDB}. Specifically, we find that the performance of {\em ACM} is almost unaffected by the changes of $\sigma$, while the performance of {\em IMDB} drops significantly when $\sigma$ is equal to $0.1$. This is mainly because {\em IMDB} has a higher $\HM$ dimension ($d=1024$). According to the {\em distance concentration}~\cite{Kriegel09}, for high-dimensional data, when $\sigma$ is too small, the off-diagonal elements of the affinity matrix  will be close to $0$, which causes the affinity matrix to become invalid. 
}
\input{tex/figs/add-exp-cpu-baselines}

\subsection{\textcolor{black}{Comparison with General-purpose Clustering Methods}}
\textcolor{black}{
    We fuse the MRGs into a single graph, and then use algorithms like \texttt{DeepWalk}, \texttt{Node2Vec}, and \texttt{PANE}~\cite{yang2020scaling} to generate node embeddings from graph structure, after that, we apply three clustering methods \texttt{DBSCAN}, \texttt{BIRCH}~\cite{zhangbirch96}, and \texttt{K-Means} on the embeddings to get the clustering results. According to Table~\ref{tbl:data-clustering}, we find that clustering methods like \texttt{DBSCAN}, which do not specify the number of clusters, tend to result in poor clustering performance. On datasets such as {\em ACM} and {\em DBLP}, the ACC and ARI scores of clustering with \texttt{DBSCAN} on embeddings generated by \texttt{DeepWalk} and \texttt{Node2Vec} are both $0$. Meanwhile, we can observe that clustering node embeddings with \texttt{K-Means} performs better than \texttt{DBSCAN} and \texttt{BIRCH}. Therefore, in \algoplus{}, we use \texttt{K-Means} to generate clusters. Additionally, \texttt{NMF} and \texttt{GMM}~\cite{DempGmms18} models are applied on node embeddings generated by FAAO algorithm with the same parameter settings as \algoplus{}.Experimental results indicate that \texttt{NMF}  generally outperforms \texttt{GMM} , as the latter tends to overfit when estimating Gaussian distribution parameters in high-dimensional spaces~\cite{Fraley02}.}

\subsection{\textcolor{black}{Computational Efficiency on CPUs}}

To demonstrate the computational advantage of \algoplus{} over deep learning methods, Figure~\ref{fig:cpu-time} compares their running times on CPUs across eight datasets of varying scales. 
Compared to running \algoplus{} on GPUs, running it on CPUs achieves more significant acceleration. Specifically, compared with the best baseline among the methods listed in Figure~\ref{fig:cpu-time}, \algoplus{} achieves speedups of $396\times$, $47\times$, $59\times$, $64\times$, and $52\times$ on small datasets {\em ACM}, {\em DBLP}, {\em ACM2}, {\em Yelp}, and {\em IMDB} using the CPUs. Compared to training on the GPUs, the average improvement rate of using the CPU on small datasets is $169.2\%$. For large datasets {\em MAG}, {\em Amazon} and {\em Protein}, a substantial improvement is also achieved: \algoplus{} achieves speedups of $645\times$, $23\times$, and $45\times$ compared to their respective best baseline. This is mainly because deep learning methods typically rely more heavily on the massively parallel computing architecture of GPUs, which means \algoplus{} can operate more efficiently even with limited computational resources.


{\color{black}
\section{Extension to Property Graphs}
Recall that a {\em property graph} is typically represented as a tuple $\G=(\V, \EDG, \ell, \pi)$, where $\V=\{v_1,v_2,\ldots,v_N\}$ denotes a set of $N$ nodes, $\EDG=\subset \V\times \V$ is a set of $M$ edges. $\ell: \V\cup\EDG \rightarrow 2^\LL$ is a labeling function that maps nodes and edges to finite sets of labels in $\LL$, and $\pi$ is a function that maps each node or edge to its respective properties (i.e., key-value pairs). Note that the properties of nodes and edges can be easily encoded as attribute vectors $\XM^{(\V)}$ and $\XM^{(\EDG)}$ with pre-trained language models, respectively, i.e., $\pi(v_i)=\XM^{(\V)}_i$ or $\pi((v_i,v_j))=\XM^{(\EDG)}_{(i,j)}$.
Suppose that there are $S$ (resp. $R$) distinct labels for nodes (resp. edges) in $\LL$.
If we regard these labels for nodes and edges as their types, the original property graph can be transformed into an augmented MRG where both nodes and edges are attributed and of various types, i.e., $\G=(\{\V^{(s)}\}_{s=1}^{S}, \{\EDG^{(r)}\}_{r=1}^R, \XM^{(\V)}, \XM^{(\EDG)})$, where $\V^{(s)}$ (resp. $\EDG^{(r)}$) is the set of nodes (resp. edges) with the $s$-th (resp. $r$-th) labels.

To extend our \algo{} and \algoplus{} to such graphs, 
we can first adapt the MRDE $\mathcal{L}_{\textnormal{MRDE}}$ in Eq.~\eqref{eq:MRDE} to the $S$ types of nodes with the $R$ edge sets $\{\EDG^{(r)}\}_{r=1}^R$ in $\G$ as follows:
\begin{equation}
\mathcal{L}_{\textnormal{MRDE}}= \sum_{s=1}^S\sum_{r=1}^R{\omega_{s,r}\cdot\mathcal{D}(\HM, \AM^{(r)}[\V^{(s)},\V^{(s)}])},
\end{equation}
where $\omega_{s,r}$ is the weight for node type $s$ and edge type $r$, and $\AM^{(r)}[\V^{(s)},\V^{(s)}]$ is the adjacency matrix constructed from edge set $\EDG^{(r)}$ and only contains nodes in $\V^{(s)}$.
Accordingly, the other two terms $\mathcal{L}_{\textnormal{fit}}$ and $\mathcal{L}_{\textnormal{reg}}$ in the Stage I objective in Eq.~\eqref{eq:obj-emb} can be adjusted as
\begin{small}
\begin{equation}
\mathcal{L}_{\textnormal{fit}} = \|\HM-\XM^{(\V)}\|_F^2, \quad \mathcal{L}_{\textnormal{reg}} = \sum_{s=1}^{S}\sum_{r=1}^{R}\omega_{s,r} \cdot \|\NAM^{(r)}[\V^{(s)},\V^{(s)}]\|_F^2.
\end{equation}
\end{small}
As for the attribute vectors of edges in $\XM^{(\EDG)}$, one simple way to incorporate such information into the objective function is to replace the above fitting term by the following term:
\begin{equation}
\mathcal{L}_{\textnormal{fit}} = \|\HM-\XM\|_F^2\ \text{and}\ \XM_{i} = \XM^{(\V)}_i + \frac{1}{R}\sum_{r=1}^R\sum_{(v_i,v_j)\in \EDG^{(r)}}{\frac{\XM^{(\EDG)}_{(i,j)}}{d^{(r)}_i}}.
\end{equation}
In doing so, \algo{} and \algoplus{} follow the same updating rules for $\HM$ and $\{\omega_{s,r}\}_{s=1,r=1}^{S,R}$ described in Sections~\ref{sec:DEMM} and \ref{sec:demm+}.
}

\pagebreak
\bibliographystyle{ACM-Reference-Format}
\bibliography{main}


\begin{thebibliography}{118}


\ifx \showCODEN    \undefined \def \showCODEN     #1{\unskip}     \fi
\ifx \showDOI      \undefined \def \showDOI       #1{#1}\fi
\ifx \showISBNx    \undefined \def \showISBNx     #1{\unskip}     \fi
\ifx \showISBNxiii \undefined \def \showISBNxiii  #1{\unskip}     \fi
\ifx \showISSN     \undefined \def \showISSN      #1{\unskip}     \fi
\ifx \showLCCN     \undefined \def \showLCCN      #1{\unskip}     \fi
\ifx \shownote     \undefined \def \shownote      #1{#1}          \fi
\ifx \showarticletitle \undefined \def \showarticletitle #1{#1}   \fi
\ifx \showURL      \undefined \def \showURL       {\relax}        \fi
\providecommand\bibfield[2]{#2}
\providecommand\bibinfo[2]{#2}
\providecommand\natexlab[1]{#1}
\providecommand\showeprint[2][]{arXiv:#2}

\bibitem[Akbas and Zhao(2017)]%
        {Akbas2017AttributedGC}
\bibfield{author}{\bibinfo{person}{Esra Akbas} {and} \bibinfo{person}{Peixiang Zhao}.} \bibinfo{year}{2017}\natexlab{}.
\newblock \showarticletitle{Attributed Graph Clustering: an Attribute-aware Graph Embedding Approach}.
\newblock \bibinfo{journal}{\emph{ASONAM}} (\bibinfo{year}{2017}).
\newblock


\bibitem[Ashourvan et~al\mbox{.}(2019)]%
        {ashourvan2019multi}
\bibfield{author}{\bibinfo{person}{Arian Ashourvan}, \bibinfo{person}{Qawi~K Telesford}, \bibinfo{person}{Timothy Verstynen}, \bibinfo{person}{Jean~M Vettel}, {and} \bibinfo{person}{Danielle~S Bassett}.} \bibinfo{year}{2019}\natexlab{}.
\newblock \showarticletitle{Multi-scale detection of hierarchical community architecture in structural and functional brain networks}.
\newblock \bibinfo{journal}{\emph{Plos one}} (\bibinfo{year}{2019}), \bibinfo{pages}{e0215520}.
\newblock


\bibitem[Bhowmick et~al\mbox{.}(2024)]%
        {bhowmick2024dgcluster}
\bibfield{author}{\bibinfo{person}{Aritra Bhowmick}, \bibinfo{person}{Mert Kosan}, \bibinfo{person}{Zexi Huang}, \bibinfo{person}{Ambuj Singh}, {and} \bibinfo{person}{Sourav Medya}.} \bibinfo{year}{2024}\natexlab{}.
\newblock \showarticletitle{DGCLUSTER: A Neural Framework for Attributed Graph Clustering via Modularity Maximization}. In \bibinfo{booktitle}{\emph{AAAI}}, Vol.~\bibinfo{volume}{38}. \bibinfo{pages}{11069--11077}.
\newblock


\bibitem[Blondel et~al\mbox{.}(2008)]%
        {blondel2008fast_Louvain}
\bibfield{author}{\bibinfo{person}{Vincent~D Blondel}, \bibinfo{person}{Jean-Loup Guillaume}, \bibinfo{person}{Renaud Lambiotte}, {and} \bibinfo{person}{Etienne Lefebvre}.} \bibinfo{year}{2008}\natexlab{}.
\newblock \showarticletitle{Fast unfolding of communities in large networks}.
\newblock \bibinfo{journal}{\emph{Journal of statistical mechanics: theory and experiment}} \bibinfo{volume}{2008}, \bibinfo{number}{10} (\bibinfo{year}{2008}), \bibinfo{pages}{P10008}.
\newblock


\bibitem[Bo et~al\mbox{.}(2020)]%
        {sdcn2020}
\bibfield{author}{\bibinfo{person}{Deyu Bo}, \bibinfo{person}{Xiao Wang}, \bibinfo{person}{Chuan Shi}, \bibinfo{person}{Meiqi Zhu}, \bibinfo{person}{Emiao Lu}, {and} \bibinfo{person}{Peng Cui}.} \bibinfo{year}{2020}\natexlab{}.
\newblock \showarticletitle{Structural Deep Clustering Network}. In \bibinfo{booktitle}{\emph{Proceedings of The Web Conference 2020}}. \bibinfo{publisher}{Association for Computing Machinery}, \bibinfo{pages}{1400–1410}.
\newblock
\showISBNx{9781450370233}


\bibitem[Bothorel et~al\mbox{.}(2015)]%
        {bothorel2015clustering}
\bibfield{author}{\bibinfo{person}{C{\'e}cile Bothorel}, \bibinfo{person}{Juan~David Cruz}, \bibinfo{person}{Matteo Magnani}, {and} \bibinfo{person}{Barbora Micenkova}.} \bibinfo{year}{2015}\natexlab{}.
\newblock \showarticletitle{Clustering attributed graphs: models, measures and methods}.
\newblock \bibinfo{journal}{\emph{Network Science}} \bibinfo{volume}{3}, \bibinfo{number}{3} (\bibinfo{year}{2015}), \bibinfo{pages}{408--444}.
\newblock


\bibitem[Cai et~al\mbox{.}(2022)]%
        {Cai2022EfficientDE}
\bibfield{author}{\bibinfo{person}{Jinyu Cai}, \bibinfo{person}{Jicong Fan}, \bibinfo{person}{Wenzhong Guo}, \bibinfo{person}{Shiping Wang}, \bibinfo{person}{Yunhe Zhang}, {and} \bibinfo{person}{Zhao Zhang}.} \bibinfo{year}{2022}\natexlab{}.
\newblock \showarticletitle{Efficient Deep Embedded Subspace Clustering}.
\newblock \bibinfo{journal}{\emph{CVPR}} (\bibinfo{year}{2022}), \bibinfo{pages}{21--30}.
\newblock


\bibitem[Chaudhuri et~al\mbox{.}(2009)]%
        {Chaudhuri2009MultiviewCV}
\bibfield{author}{\bibinfo{person}{Kamalika Chaudhuri}, \bibinfo{person}{Sham~M. Kakade}, \bibinfo{person}{Karen Livescu}, {and} \bibinfo{person}{Karthik Sridharan}.} \bibinfo{year}{2009}\natexlab{}.
\newblock \showarticletitle{Multi-view clustering via canonical correlation analysis}. In \bibinfo{booktitle}{\emph{International Conference on Machine Learning}}.
\newblock


\bibitem[Chen et~al\mbox{.}(2023)]%
        {Chen2023OnRM}
\bibfield{author}{\bibinfo{person}{Mansheng Chen}, \bibinfo{person}{Jia-Qi Lin}, \bibinfo{person}{Changdong Wang}, \bibinfo{person}{Wu-Dong Xi}, {and} \bibinfo{person}{Dong Huang}.} \bibinfo{year}{2023}\natexlab{}.
\newblock \showarticletitle{On Regularizing Multiple Clusterings for Ensemble Clustering by Graph Tensor Learning}.
\newblock \bibinfo{journal}{\emph{MM}} (\bibinfo{year}{2023}).
\newblock


\bibitem[Cheng et~al\mbox{.}(2020)]%
        {Cheng2020MultiViewAG}
\bibfield{author}{\bibinfo{person}{Jiafeng Cheng}, \bibinfo{person}{Qianqian Wang}, \bibinfo{person}{Zhiqiang Tao}, \bibinfo{person}{Deyan Xie}, {and} \bibinfo{person}{Quanxue Gao}.} \bibinfo{year}{2020}\natexlab{}.
\newblock \showarticletitle{Multi-View Attribute Graph Convolution Networks for Clustering}. In \bibinfo{booktitle}{\emph{International Joint Conference on Artificial Intelligence}}.
\newblock


\bibitem[Clarkson and Woodruff(2017)]%
        {clarkson2017low}
\bibfield{author}{\bibinfo{person}{Kenneth~L Clarkson} {and} \bibinfo{person}{David~P Woodruff}.} \bibinfo{year}{2017}\natexlab{}.
\newblock \showarticletitle{Low-rank approximation and regression in input sparsity time}.
\newblock \bibinfo{journal}{\emph{Journal of the ACM (JACM)}} \bibinfo{volume}{63}, \bibinfo{number}{6} (\bibinfo{year}{2017}), \bibinfo{pages}{1--45}.
\newblock


\bibitem[Combe et~al\mbox{.}(2015)]%
        {Combe2015ILouvainAA}
\bibfield{author}{\bibinfo{person}{David Combe}, \bibinfo{person}{Christine Largeron}, \bibinfo{person}{Mathias G{\'e}ry}, {and} \bibinfo{person}{El{\"o}d Egyed-Zsigmond}.} \bibinfo{year}{2015}\natexlab{}.
\newblock \showarticletitle{I-Louvain: An Attributed Graph Clustering Method}. In \bibinfo{booktitle}{\emph{International Symposium on Intelligent Data Analysis}}.
\newblock


\bibitem[Crofts et~al\mbox{.}(2022)]%
        {crofts2022structure}
\bibfield{author}{\bibinfo{person}{J.~J. Crofts}, \bibinfo{person}{M. Forrester}, \bibinfo{person}{S. Coombes}, {and} \bibinfo{person}{R.~D. O'Dea}.} \bibinfo{year}{2022}\natexlab{}.
\newblock \showarticletitle{Structure-Function Clustering in Weighted Brain Networks}.
\newblock \bibinfo{journal}{\emph{Scientific Reports}}  \bibinfo{volume}{12} (\bibinfo{year}{2022}), \bibinfo{pages}{16793}.
\newblock


\bibitem[Cui et~al\mbox{.}(2023)]%
        {cui2023deep}
\bibfield{author}{\bibinfo{person}{Chenhang Cui}, \bibinfo{person}{Yazhou Ren}, \bibinfo{person}{Jingyu Pu}, \bibinfo{person}{Xiaorong Pu}, {and} \bibinfo{person}{Lifang He}.} \bibinfo{year}{2023}\natexlab{}.
\newblock \showarticletitle{Deep multi-view subspace clustering with anchor graph}. In \bibinfo{booktitle}{\emph{IJCAI}}. \bibinfo{pages}{3577--3585}.
\newblock


\bibitem[Cui et~al\mbox{.}(2020)]%
        {Cui2020AdaptiveGE}
\bibfield{author}{\bibinfo{person}{Ganqu Cui}, \bibinfo{person}{Jie Zhou}, \bibinfo{person}{Cheng Yang}, {and} \bibinfo{person}{Zhiyuan Liu}.} \bibinfo{year}{2020}\natexlab{}.
\newblock \showarticletitle{Adaptive Graph Encoder for Attributed Graph Embedding}.
\newblock \bibinfo{journal}{\emph{KDD}} (\bibinfo{year}{2020}).
\newblock


\bibitem[Dempster et~al\mbox{.}(2018)]%
        {DempGmms18}
\bibfield{author}{\bibinfo{person}{A.~P. Dempster}, \bibinfo{person}{N.~M. Laird}, {and} \bibinfo{person}{D.~B. Rubin}.} \bibinfo{year}{2018}\natexlab{}.
\newblock \showarticletitle{Maximum Likelihood from Incomplete Data Via the EM Algorithm}.
\newblock \bibinfo{journal}{\emph{JRSS}} \bibinfo{volume}{39}, \bibinfo{number}{1} (\bibinfo{date}{12} \bibinfo{year}{2018}), \bibinfo{pages}{1--22}.
\newblock


\bibitem[Devvrit et~al\mbox{.}(2022)]%
        {devvrit2022s3gc}
\bibfield{author}{\bibinfo{person}{Fnu Devvrit}, \bibinfo{person}{Aditya Sinha}, \bibinfo{person}{Inderjit Dhillon}, {and} \bibinfo{person}{Prateek Jain}.} \bibinfo{year}{2022}\natexlab{}.
\newblock \showarticletitle{S3GC: scalable self-supervised graph clustering}.
\newblock \bibinfo{journal}{\emph{NeurIPS}}  \bibinfo{volume}{35} (\bibinfo{year}{2022}), \bibinfo{pages}{3248--3261}.
\newblock


\bibitem[Du and Li(2022)]%
        {du2022academic}
\bibfield{author}{\bibinfo{person}{Ouxia Du} {and} \bibinfo{person}{Ya Li}.} \bibinfo{year}{2022}\natexlab{}.
\newblock \showarticletitle{Academic Collaborator Recommendation Based on Attributed Network Embedding.}
\newblock \bibinfo{journal}{\emph{J. Data Inf. Sci.}} \bibinfo{volume}{7}, \bibinfo{number}{1} (\bibinfo{year}{2022}), \bibinfo{pages}{37--56}.
\newblock


\bibitem[Fan(1949)]%
        {fan1949theorem}
\bibfield{author}{\bibinfo{person}{Ky Fan}.} \bibinfo{year}{1949}\natexlab{}.
\newblock \showarticletitle{On a theorem of Weyl concerning eigenvalues of linear transformations I}.
\newblock \bibinfo{journal}{\emph{PNAS}} \bibinfo{volume}{35}, \bibinfo{number}{11} (\bibinfo{year}{1949}), \bibinfo{pages}{652--655}.
\newblock


\bibitem[Fan et~al\mbox{.}(2020)]%
        {Fan2020One2MultiGA}
\bibfield{author}{\bibinfo{person}{Shaohua Fan}, \bibinfo{person}{Xiao Wang}, \bibinfo{person}{Chuan Shi}, \bibinfo{person}{Emiao Lu}, \bibinfo{person}{Ken Lin}, {and} \bibinfo{person}{Bai Wang}.} \bibinfo{year}{2020}\natexlab{}.
\newblock \showarticletitle{One2Multi Graph Autoencoder for Multi-view Graph Clustering}.
\newblock \bibinfo{journal}{\emph{WWW}} (\bibinfo{year}{2020}).
\newblock


\bibitem[Fan et~al\mbox{.}(2022)]%
        {GNNRECOfan2022}
\bibfield{author}{\bibinfo{person}{Wenqi Fan}, \bibinfo{person}{Yao Ma}, \bibinfo{person}{Qing Li}, \bibinfo{person}{Jianping Wang}, \bibinfo{person}{Guoyong Cai}, \bibinfo{person}{Jiliang Tang}, {and} \bibinfo{person}{Dawei Yin}.} \bibinfo{year}{2022}\natexlab{}.
\newblock \showarticletitle{A Graph Neural Network Framework for Social Recommendations}.
\newblock \bibinfo{journal}{\emph{TKDE}} \bibinfo{volume}{34}, \bibinfo{number}{5} (\bibinfo{year}{2022}), \bibinfo{pages}{2033--2047}.
\newblock


\bibitem[Fettal et~al\mbox{.}(2023)]%
        {fettal2023scalable}
\bibfield{author}{\bibinfo{person}{Chakib Fettal}, \bibinfo{person}{Lazhar Labiod}, {and} \bibinfo{person}{Mohamed Nadif}.} \bibinfo{year}{2023}\natexlab{}.
\newblock \showarticletitle{Scalable Attributed-Graph Subspace Clustering}. In \bibinfo{booktitle}{\emph{AAAI}}, Vol.~\bibinfo{volume}{37}.
\newblock


\bibitem[Fraley and Raftery(2002)]%
        {Fraley02}
\bibfield{author}{\bibinfo{person}{Chris Fraley} {and} \bibinfo{person}{Adrian Raftery}.} \bibinfo{year}{2002}\natexlab{}.
\newblock \showarticletitle{Model-Based Clustering, Discriminant Analysis, and Density Estimation}.
\newblock \bibinfo{journal}{\emph{JASA}}  \bibinfo{volume}{97} (\bibinfo{date}{06} \bibinfo{year}{2002}), \bibinfo{pages}{611--631}.
\newblock


\bibitem[Fu et~al\mbox{.}(2020)]%
        {Fu2020MAGNNMA}
\bibfield{author}{\bibinfo{person}{Xinyu Fu}, \bibinfo{person}{Jiani Zhang}, \bibinfo{person}{Ziqiao Meng}, {and} \bibinfo{person}{Irwin King}.} \bibinfo{year}{2020}\natexlab{}.
\newblock \showarticletitle{MAGNN: Metapath Aggregated Graph Neural Network for Heterogeneous Graph Embedding}.
\newblock \bibinfo{journal}{\emph{WWW}} (\bibinfo{year}{2020}).
\newblock


\bibitem[Goldschmidt and Hochbaum(1988)]%
        {goldschmidt1988polynomial}
\bibfield{author}{\bibinfo{person}{Olivier Goldschmidt} {and} \bibinfo{person}{Dorit~S Hochbaum}.} \bibinfo{year}{1988}\natexlab{}.
\newblock \showarticletitle{Polynomial algorithm for the k-cut problem}. In \bibinfo{booktitle}{\emph{FOCS}}. \bibinfo{pages}{444--451}.
\newblock


\bibitem[Grover and Leskovec(2016)]%
        {grover2016node2vec}
\bibfield{author}{\bibinfo{person}{Aditya Grover} {and} \bibinfo{person}{Jure Leskovec}.} \bibinfo{year}{2016}\natexlab{}.
\newblock \showarticletitle{node2vec: Scalable feature learning for networks}. In \bibinfo{booktitle}{\emph{KDD}}. \bibinfo{pages}{855--864}.
\newblock


\bibitem[Gu et~al\mbox{.}(2022)]%
        {GU2022106127}
\bibfield{author}{\bibinfo{person}{Yaowen Gu}, \bibinfo{person}{Si Zheng}, \bibinfo{person}{Qijin Yin}, \bibinfo{person}{Rui Jiang}, {and} \bibinfo{person}{Jiao Li}.} \bibinfo{year}{2022}\natexlab{}.
\newblock \showarticletitle{REDDA: Integrating multiple biological relations to heterogeneous graph neural network for drug-disease association prediction}.
\newblock \bibinfo{journal}{\emph{Computers in Biology and Medicine}}  \bibinfo{volume}{150} (\bibinfo{year}{2022}), \bibinfo{pages}{106127}.
\newblock


\bibitem[Guesmi et~al\mbox{.}(2019)]%
        {guesmi2019community}
\bibfield{author}{\bibinfo{person}{Soumaya Guesmi}, \bibinfo{person}{Chiraz Trabelsi}, {and} \bibinfo{person}{Chiraz Latiri}.} \bibinfo{year}{2019}\natexlab{}.
\newblock \showarticletitle{Community detection in multi-relational social networks based on relational concept analysis}.
\newblock \bibinfo{journal}{\emph{PCS}}  \bibinfo{volume}{159} (\bibinfo{year}{2019}), \bibinfo{pages}{291--300}.
\newblock


\bibitem[Hassani and Ahmadi(2020)]%
        {Hassani2020ContrastiveMR}
\bibfield{author}{\bibinfo{person}{Kaveh Hassani} {and} \bibinfo{person}{Amir Hosein~Khas Ahmadi}.} \bibinfo{year}{2020}\natexlab{}.
\newblock \showarticletitle{Contrastive Multi-View Representation Learning on Graphs}. In \bibinfo{booktitle}{\emph{ICML}}.
\newblock


\bibitem[Horn(1962)]%
        {horn1962eigenvalues}
\bibfield{author}{\bibinfo{person}{Alfred Horn}.} \bibinfo{year}{1962}\natexlab{}.
\newblock \showarticletitle{Eigenvalues of sums of Hermitian matrices.}
\newblock  (\bibinfo{year}{1962}).
\newblock


\bibitem[Horn and Johnson(2012)]%
        {horn2012matrix}
\bibfield{author}{\bibinfo{person}{Roger~A Horn} {and} \bibinfo{person}{Charles~R Johnson}.} \bibinfo{year}{2012}\natexlab{}.
\newblock \bibinfo{booktitle}{\emph{Matrix analysis}}.
\newblock \bibinfo{publisher}{Cambridge university press}.
\newblock


\bibitem[Hu et~al\mbox{.}(2019)]%
        {HuHGCN19}
\bibfield{author}{\bibinfo{person}{Fenyu Hu}, \bibinfo{person}{Yanqiao Zhu}, \bibinfo{person}{Shu Wu}, \bibinfo{person}{Liang Wang}, {and} \bibinfo{person}{Tieniu Tan}.} \bibinfo{year}{2019}\natexlab{}.
\newblock \showarticletitle{Hierarchical graph convolutional networks for semi-supervised node classification}. In \bibinfo{booktitle}{\emph{IJCAI}}.
\newblock


\bibitem[Hu et~al\mbox{.}(2020)]%
        {Hu2020OpenGB}
\bibfield{author}{\bibinfo{person}{Weihua Hu}, \bibinfo{person}{Matthias Fey}, \bibinfo{person}{Marinka Zitnik}, \bibinfo{person}{Yuxiao Dong}, \bibinfo{person}{Hongyu Ren}, \bibinfo{person}{Bowen Liu}, \bibinfo{person}{Michele Catasta}, {and} \bibinfo{person}{Jure Leskovec}.} \bibinfo{year}{2020}\natexlab{}.
\newblock \showarticletitle{Open Graph Benchmark: Datasets for Machine Learning on Graphs}.
\newblock \bibinfo{journal}{\emph{ArXiv}}  \bibinfo{volume}{abs/2005.00687} (\bibinfo{year}{2020}).
\newblock


\bibitem[Huang et~al\mbox{.}(2023)]%
        {Huang2023MultiViewSC}
\bibfield{author}{\bibinfo{person}{Shudong Huang}, \bibinfo{person}{Yixi Liu}, \bibinfo{person}{Ivor Wai-Hung Tsang}, \bibinfo{person}{Zenglin Xu}, {and} \bibinfo{person}{Jiancheng Lv}.} \bibinfo{year}{2023}\natexlab{}.
\newblock \showarticletitle{Multi-View Subspace Clustering by Joint Measuring of Consistency and Diversity}.
\newblock \bibinfo{journal}{\emph{TKDE}}  \bibinfo{volume}{35} (\bibinfo{year}{2023}), \bibinfo{pages}{8270--8281}.
\newblock


\bibitem[Huo et~al\mbox{.}(2021)]%
        {Huo2021CaEGCNCF}
\bibfield{author}{\bibinfo{person}{Guangyu Huo}, \bibinfo{person}{Yong Zhang}, \bibinfo{person}{Junbin Gao}, \bibinfo{person}{Boyue Wang}, \bibinfo{person}{Yongli Hu}, {and} \bibinfo{person}{Baocai Yin}.} \bibinfo{year}{2021}\natexlab{}.
\newblock \showarticletitle{CaEGCN: Cross-Attention Fusion Based Enhanced Graph Convolutional Network for Clustering}.
\newblock \bibinfo{journal}{\emph{IEEE Transactions on Knowledge and Data Engineering}}  \bibinfo{volume}{35} (\bibinfo{year}{2021}), \bibinfo{pages}{3471--3483}.
\newblock


\bibitem[Kang et~al\mbox{.}(2020)]%
        {kang2020large}
\bibfield{author}{\bibinfo{person}{Zhao Kang}, \bibinfo{person}{Wangtao Zhou}, \bibinfo{person}{Zhitong Zhao}, \bibinfo{person}{Junming Shao}, \bibinfo{person}{Meng Han}, {and} \bibinfo{person}{Zenglin Xu}.} \bibinfo{year}{2020}\natexlab{}.
\newblock \showarticletitle{Large-scale multi-view subspace clustering in linear time}. In \bibinfo{booktitle}{\emph{AAAI}}, Vol.~\bibinfo{volume}{34}. \bibinfo{pages}{4412--4419}.
\newblock


\bibitem[Knight(2008)]%
        {knight2008sinkhorn}
\bibfield{author}{\bibinfo{person}{Philip~A Knight}.} \bibinfo{year}{2008}\natexlab{}.
\newblock \showarticletitle{The Sinkhorn--Knopp algorithm: convergence and applications}.
\newblock \bibinfo{journal}{\emph{SIAM J. Matrix Anal. Appl.}} \bibinfo{volume}{30}, \bibinfo{number}{1} (\bibinfo{year}{2008}), \bibinfo{pages}{261--275}.
\newblock


\bibitem[Kriegel et~al\mbox{.}(2009)]%
        {Kriegel09}
\bibfield{author}{\bibinfo{person}{Hans-Peter Kriegel}, \bibinfo{person}{Peer Kr\"{o}ger}, {and} \bibinfo{person}{Arthur Zimek}.} \bibinfo{year}{2009}\natexlab{}.
\newblock \showarticletitle{Clustering high-dimensional data: A survey on subspace clustering, pattern-based clustering, and correlation clustering}.
\newblock  \bibinfo{volume}{3}, \bibinfo{number}{1}, Article \bibinfo{articleno}{1} (\bibinfo{year}{2009}), \bibinfo{numpages}{58}~pages.
\newblock
\showISSN{1556-4681}


\bibitem[Kuhn(1955)]%
        {kuhn1955hungarian}
\bibfield{author}{\bibinfo{person}{Harold~W Kuhn}.} \bibinfo{year}{1955}\natexlab{}.
\newblock \showarticletitle{The Hungarian method for the assignment problem}.
\newblock \bibinfo{journal}{\emph{Naval research logistics quarterly}} \bibinfo{volume}{2}, \bibinfo{number}{1-2} (\bibinfo{year}{1955}), \bibinfo{pages}{83--97}.
\newblock


\bibitem[Lai et~al\mbox{.}(2023)]%
        {lai2023re}
\bibfield{author}{\bibinfo{person}{Xinying Lai}, \bibinfo{person}{Dingming Wu}, \bibinfo{person}{Christian~S Jensen}, {and} \bibinfo{person}{Kezhong Lu}.} \bibinfo{year}{2023}\natexlab{}.
\newblock \showarticletitle{A Re-evaluation of Deep Learning Methods for Attributed Graph Clustering}. In \bibinfo{booktitle}{\emph{CIKM}}. \bibinfo{pages}{1168--1177}.
\newblock


\bibitem[Lehoucq and Sorensen(1996)]%
        {lehoucq1996deflation}
\bibfield{author}{\bibinfo{person}{Richard~B Lehoucq} {and} \bibinfo{person}{Danny~C Sorensen}.} \bibinfo{year}{1996}\natexlab{}.
\newblock \showarticletitle{Deflation techniques for an implicitly restarted Arnoldi iteration}.
\newblock \bibinfo{journal}{\emph{SIAM J. Matrix Anal. Appl.}} \bibinfo{volume}{17}, \bibinfo{number}{4} (\bibinfo{year}{1996}), \bibinfo{pages}{789--821}.
\newblock


\bibitem[Levin and Peres(2017)]%
        {levin2017markov}
\bibfield{author}{\bibinfo{person}{David~A Levin} {and} \bibinfo{person}{Yuval Peres}.} \bibinfo{year}{2017}\natexlab{}.
\newblock \bibinfo{booktitle}{\emph{Markov chains and mixing times}}. Vol.~\bibinfo{volume}{107}.
\newblock \bibinfo{publisher}{American Mathematical Soc.}
\newblock


\bibitem[Li et~al\mbox{.}(2024b)]%
        {KEIGLi2024}
\bibfield{author}{\bibinfo{person}{Mingqi Li}, \bibinfo{person}{Wenming Ma}, {and} \bibinfo{person}{Zihao Chu}.} \bibinfo{year}{2024}\natexlab{b}.
\newblock \showarticletitle{KGIE: Knowledge graph convolutional network for recommender system with interactive embedding}.
\newblock \bibinfo{journal}{\emph{Knowledge-Based Systems}}  \bibinfo{volume}{295} (\bibinfo{year}{2024}), \bibinfo{pages}{111813}.
\newblock
\showISSN{0950-7051}


\bibitem[Li et~al\mbox{.}(2021b)]%
        {li2021graph}
\bibfield{author}{\bibinfo{person}{Rui Li}, \bibinfo{person}{Xin Yuan}, \bibinfo{person}{Mohsen Radfar}, \bibinfo{person}{Peter Marendy}, \bibinfo{person}{Wei Ni}, \bibinfo{person}{Terrence~J O’Brien}, {and} \bibinfo{person}{Pablo~M Casillas-Espinosa}.} \bibinfo{year}{2021}\natexlab{b}.
\newblock \showarticletitle{Graph signal processing, graph neural network and graph learning on biological data: a systematic review}.
\newblock \bibinfo{journal}{\emph{IEEE Reviews in Biomedical Engineering}}  \bibinfo{volume}{16} (\bibinfo{year}{2021}), \bibinfo{pages}{109--135}.
\newblock


\bibitem[Li et~al\mbox{.}(2024a)]%
        {li2024versatile}
\bibfield{author}{\bibinfo{person}{Yiran Li}, \bibinfo{person}{Gongyao Guo}, \bibinfo{person}{Jieming Shi}, \bibinfo{person}{Renchi Yang}, \bibinfo{person}{Shiqi Shen}, \bibinfo{person}{Qing Li}, {and} \bibinfo{person}{Jun Luo}.} \bibinfo{year}{2024}\natexlab{a}.
\newblock \showarticletitle{A versatile framework for attributed network clustering via K-nearest neighbor augmentation}.
\newblock \bibinfo{journal}{\emph{VLDBJ}} (\bibinfo{year}{2024}), \bibinfo{pages}{1--31}.
\newblock


\bibitem[Li et~al\mbox{.}(2018)]%
        {Li2018CommunityDI}
\bibfield{author}{\bibinfo{person}{Ye Li}, \bibinfo{person}{Chaofeng Sha}, \bibinfo{person}{Xin Huang}, {and} \bibinfo{person}{Yanchun Zhang}.} \bibinfo{year}{2018}\natexlab{}.
\newblock \showarticletitle{Community Detection in Attributed Graphs: An Embedding Approach}. In \bibinfo{booktitle}{\emph{AAAI}}.
\newblock


\bibitem[Li et~al\mbox{.}(2023)]%
        {li2023efficient}
\bibfield{author}{\bibinfo{person}{Yiran Li}, \bibinfo{person}{Renchi Yang}, {and} \bibinfo{person}{Jieming Shi}.} \bibinfo{year}{2023}\natexlab{}.
\newblock \showarticletitle{Efficient and effective attributed hypergraph clustering via k-nearest neighbor augmentation}.
\newblock \bibinfo{journal}{\emph{SIGMOD}} \bibinfo{volume}{1}, \bibinfo{number}{2} (\bibinfo{year}{2023}), \bibinfo{pages}{1--23}.
\newblock


\bibitem[Li et~al\mbox{.}(2021a)]%
        {Li2021ConsensusGL}
\bibfield{author}{\bibinfo{person}{Zhenglai Li}, \bibinfo{person}{Chang Tang}, \bibinfo{person}{Xinwang Liu}, \bibinfo{person}{Xiao Zheng}, \bibinfo{person}{Guanghui Yue}, \bibinfo{person}{Wei Zhang}, {and} \bibinfo{person}{En Zhu}.} \bibinfo{year}{2021}\natexlab{a}.
\newblock \showarticletitle{Consensus Graph Learning for Multi-View Clustering}.
\newblock \bibinfo{journal}{\emph{IEEE Transactions on Multimedia}}  \bibinfo{volume}{24} (\bibinfo{year}{2021}), \bibinfo{pages}{2461--2472}.
\newblock


\bibitem[Lidskii(1982)]%
        {lidskii1982spectral}
\bibfield{author}{\bibinfo{person}{Boris~Viktorovich Lidskii}.} \bibinfo{year}{1982}\natexlab{}.
\newblock \showarticletitle{Spectral polyhedron of a sum of two Hermitian matrices}.
\newblock \bibinfo{journal}{\emph{Functional Analysis and Its Applications}} \bibinfo{volume}{16}, \bibinfo{number}{2} (\bibinfo{year}{1982}), \bibinfo{pages}{139--140}.
\newblock


\bibitem[Lin et~al\mbox{.}(2025)]%
        {lin2024spectral}
\bibfield{author}{\bibinfo{person}{Xiaoyang Lin}, \bibinfo{person}{Renchi Yang}, \bibinfo{person}{Haoran Zheng}, {and} \bibinfo{person}{Xiangyu Ke}.} \bibinfo{year}{2025}\natexlab{}.
\newblock \showarticletitle{Spectral Subspace Clustering for Attributed Graphs}. In \bibinfo{booktitle}{\emph{KDD}}. \bibinfo{pages}{789--799}.
\newblock


\bibitem[Lin and Kang(2021)]%
        {Lin2021GraphFM}
\bibfield{author}{\bibinfo{person}{Zhiping Lin} {and} \bibinfo{person}{Zhao Kang}.} \bibinfo{year}{2021}\natexlab{}.
\newblock \showarticletitle{Graph Filter-based Multi-view Attributed Graph Clustering}. In \bibinfo{booktitle}{\emph{IJCAI}}.
\newblock


\bibitem[Lin et~al\mbox{.}(2021)]%
        {lin2021multi}
\bibfield{author}{\bibinfo{person}{Zizheng Lin}, \bibinfo{person}{Haowen Ke}, \bibinfo{person}{Ngo-Yin Wong}, \bibinfo{person}{Jiaxin Bai}, \bibinfo{person}{Yangqiu Song}, \bibinfo{person}{Huan Zhao}, {and} \bibinfo{person}{Junpeng Ye}.} \bibinfo{year}{2021}\natexlab{}.
\newblock \showarticletitle{Multi-relational graph based heterogeneous multi-task learning in community question answering}. In \bibinfo{booktitle}{\emph{CIKM}}. \bibinfo{pages}{1038--1047}.
\newblock


\bibitem[Ling et~al\mbox{.}(2023)]%
        {Ling2023DualLG}
\bibfield{author}{\bibinfo{person}{Yawen Ling}, \bibinfo{person}{Jianpeng Chen}, \bibinfo{person}{Yazhou Ren}, \bibinfo{person}{Xiaorong Pu}, \bibinfo{person}{Jie Xu}, \bibinfo{person}{Xiaofeng Zhu}, {and} \bibinfo{person}{Lifang He}.} \bibinfo{year}{2023}\natexlab{}.
\newblock \showarticletitle{Dual Label-Guided Graph Refinement for Multi-View Graph Clustering}. In \bibinfo{booktitle}{\emph{AAAI}}.
\newblock


\bibitem[Liu et~al\mbox{.}(2021)]%
        {Liu2021MultilayerGC}
\bibfield{author}{\bibinfo{person}{Liang Liu}, \bibinfo{person}{Zhao Kang}, \bibinfo{person}{Ling Tian}, \bibinfo{person}{Wenbo Xu}, {and} \bibinfo{person}{Xixu He}.} \bibinfo{year}{2021}\natexlab{}.
\newblock \showarticletitle{Multilayer Graph Contrastive Clustering Network}.
\newblock \bibinfo{journal}{\emph{ArXiv}}  \bibinfo{volume}{abs/2112.14021} (\bibinfo{year}{2021}).
\newblock


\bibitem[Liu et~al\mbox{.}(2011)]%
        {liu2011kernel}
\bibfield{author}{\bibinfo{person}{Weifeng Liu}, \bibinfo{person}{Jose~C Principe}, {and} \bibinfo{person}{Simon Haykin}.} \bibinfo{year}{2011}\natexlab{}.
\newblock \bibinfo{booktitle}{\emph{Kernel adaptive filtering: a comprehensive introduction}}.
\newblock \bibinfo{publisher}{John Wiley \& Sons}.
\newblock


\bibitem[Liu et~al\mbox{.}(2023)]%
        {liu2023dink}
\bibfield{author}{\bibinfo{person}{Yue Liu}, \bibinfo{person}{Ke Liang}, \bibinfo{person}{Jun Xia}, \bibinfo{person}{Sihang Zhou}, \bibinfo{person}{Xihong Yang}, \bibinfo{person}{Xinwang Liu}, {and} \bibinfo{person}{Stan~Z Li}.} \bibinfo{year}{2023}\natexlab{}.
\newblock \showarticletitle{Dink-net: Neural clustering on large graphs}. In \bibinfo{booktitle}{\emph{International Conference on Machine Learning}}. PMLR, \bibinfo{pages}{21794--21812}.
\newblock


\bibitem[Liu et~al\mbox{.}(2022a)]%
        {DCRN}
\bibfield{author}{\bibinfo{person}{Yue Liu}, \bibinfo{person}{Wenxuan Tu}, \bibinfo{person}{Sihang Zhou}, \bibinfo{person}{Xinwang Liu}, \bibinfo{person}{Linxuan Song}, \bibinfo{person}{Xihong Yang}, {and} \bibinfo{person}{En Zhu}.} \bibinfo{year}{2022}\natexlab{a}.
\newblock \showarticletitle{Deep Graph Clustering via Dual Correlation Reduction}. In \bibinfo{booktitle}{\emph{AAAI}}, Vol.~\bibinfo{volume}{36}. \bibinfo{pages}{7603--7611}.
\newblock


\bibitem[Liu et~al\mbox{.}(2022b)]%
        {liu2022survey}
\bibfield{author}{\bibinfo{person}{Yue Liu}, \bibinfo{person}{Jun Xia}, \bibinfo{person}{Sihang Zhou}, \bibinfo{person}{Xihong Yang}, \bibinfo{person}{Ke Liang}, \bibinfo{person}{Chenchen Fan}, \bibinfo{person}{Yan Zhuang}, \bibinfo{person}{Stan~Z Li}, \bibinfo{person}{Xinwang Liu}, {and} \bibinfo{person}{Kunlun He}.} \bibinfo{year}{2022}\natexlab{b}.
\newblock \showarticletitle{A Survey of Deep Graph Clustering: Taxonomy, Challenge, Application, and Open Resource}.
\newblock \bibinfo{journal}{\emph{arXiv preprint arXiv:2211.12875}} (\bibinfo{year}{2022}).
\newblock


\bibitem[Mo et~al\mbox{.}(2023a)]%
        {Mo2023MultiplexGR}
\bibfield{author}{\bibinfo{person}{Yujie Mo}, \bibinfo{person}{Yuhuan Chen}, \bibinfo{person}{Yajie Lei}, \bibinfo{person}{Liang Peng}, \bibinfo{person}{Xiaoshuang Shi}, \bibinfo{person}{Changan Yuan}, {and} \bibinfo{person}{Xiaofeng Zhu}.} \bibinfo{year}{2023}\natexlab{a}.
\newblock \showarticletitle{Multiplex Graph Representation Learning Via Dual Correlation Reduction}.
\newblock \bibinfo{journal}{\emph{TKDE}}  \bibinfo{volume}{35} (\bibinfo{year}{2023}), \bibinfo{pages}{12814--12827}.
\newblock


\bibitem[Mo et~al\mbox{.}(2023b)]%
        {Mo2023DisentangledMG}
\bibfield{author}{\bibinfo{person}{Yujie Mo}, \bibinfo{person}{Yajie Lei}, \bibinfo{person}{Jialie Shen}, \bibinfo{person}{Xiaoshuang Shi}, \bibinfo{person}{Heng~Tao Shen}, {and} \bibinfo{person}{Xiaofeng Zhu}.} \bibinfo{year}{2023}\natexlab{b}.
\newblock \showarticletitle{Disentangled Multiplex Graph Representation Learning}. In \bibinfo{booktitle}{\emph{International Conference on Machine Learning}}.
\newblock


\bibitem[Newman(2006)]%
        {newman2006finding}
\bibfield{author}{\bibinfo{person}{Mark~EJ Newman}.} \bibinfo{year}{2006}\natexlab{}.
\newblock \showarticletitle{Finding community structure in networks using the eigenvectors of matrices}.
\newblock \bibinfo{journal}{\emph{Physical Review E—Statistical, Nonlinear, and Soft Matter Physics}} \bibinfo{volume}{74}, \bibinfo{number}{3} (\bibinfo{year}{2006}), \bibinfo{pages}{036104}.
\newblock


\bibitem[Ng et~al\mbox{.}(2001)]%
        {Ng2001OnSC}
\bibfield{author}{\bibinfo{person}{A. Ng}, \bibinfo{person}{Michael~I. Jordan}, {and} \bibinfo{person}{Yair Weiss}.} \bibinfo{year}{2001}\natexlab{}.
\newblock \showarticletitle{On Spectral Clustering: Analysis and an algorithm}. In \bibinfo{booktitle}{\emph{Neural Information Processing Systems}}.
\newblock


\bibitem[Nie et~al\mbox{.}(2017)]%
        {Nie2017SelfweightedMC}
\bibfield{author}{\bibinfo{person}{Feiping Nie}, \bibinfo{person}{Jing Li}, {and} \bibinfo{person}{Xuelong Li}.} \bibinfo{year}{2017}\natexlab{}.
\newblock \showarticletitle{Self-weighted Multiview Clustering with Multiple Graphs}. In \bibinfo{booktitle}{\emph{IJCAI}}.
\newblock


\bibitem[Pan and Kang(2021)]%
        {Pan2021MultiviewCG}
\bibfield{author}{\bibinfo{person}{Erlin Pan} {and} \bibinfo{person}{Zhao Kang}.} \bibinfo{year}{2021}\natexlab{}.
\newblock \showarticletitle{Multi-view Contrastive Graph Clustering}. In \bibinfo{booktitle}{\emph{NIPS}}.
\newblock


\bibitem[Pan and Kang(2023)]%
        {Pan2023BeyondHR}
\bibfield{author}{\bibinfo{person}{Erlin Pan} {and} \bibinfo{person}{Zhao Kang}.} \bibinfo{year}{2023}\natexlab{}.
\newblock \showarticletitle{Beyond Homophily: Reconstructing Structure for Graph-agnostic Clustering}. In \bibinfo{booktitle}{\emph{ICML}}.
\newblock


\bibitem[Park et~al\mbox{.}(2019)]%
        {Park2019UnsupervisedAM}
\bibfield{author}{\bibinfo{person}{Chanyoung Park}, \bibinfo{person}{Donghyun Kim}, \bibinfo{person}{Jiawei Han}, {and} \bibinfo{person}{Hwanjo Yu}.} \bibinfo{year}{2019}\natexlab{}.
\newblock \showarticletitle{Unsupervised Attributed Multiplex Network Embedding}. In \bibinfo{booktitle}{\emph{AAAI}}.
\newblock


\bibitem[Peng et~al\mbox{.}(2021)]%
        {PengRe2021}
\bibfield{author}{\bibinfo{person}{Hao Peng}, \bibinfo{person}{Ruitong Zhang}, \bibinfo{person}{Yingtong Dou}, \bibinfo{person}{Renyu Yang}, \bibinfo{person}{Jingyi Zhang}, {and} \bibinfo{person}{Philip~S. Yu}.} \bibinfo{year}{2021}\natexlab{}.
\newblock \showarticletitle{Reinforced Neighborhood Selection Guided Multi-Relational Graph Neural Networks}.
\newblock \bibinfo{journal}{\emph{ACM Trans. Inf. Syst.}}, Article \bibinfo{articleno}{69} (\bibinfo{date}{Dec.} \bibinfo{year}{2021}), \bibinfo{numpages}{46}~pages.
\newblock


\bibitem[Peng et~al\mbox{.}(2023)]%
        {Peng2023UnsupervisedMG}
\bibfield{author}{\bibinfo{person}{Liang Peng}, \bibinfo{person}{Xin Wang}, {and} \bibinfo{person}{Xiaofeng Zhu}.} \bibinfo{year}{2023}\natexlab{}.
\newblock \showarticletitle{Unsupervised Multiplex Graph learning with Complementary and Consistent Information}.
\newblock \bibinfo{journal}{\emph{MM}} (\bibinfo{year}{2023}).
\newblock


\bibitem[Perozzi et~al\mbox{.}(2014)]%
        {perozzi2014deepwalk}
\bibfield{author}{\bibinfo{person}{Bryan Perozzi}, \bibinfo{person}{Rami Al-Rfou}, {and} \bibinfo{person}{Steven Skiena}.} \bibinfo{year}{2014}\natexlab{}.
\newblock \showarticletitle{Deepwalk: Online learning of social representations}. In \bibinfo{booktitle}{\emph{KDD}}. \bibinfo{pages}{701--710}.
\newblock


\bibitem[Qian et~al\mbox{.}(2023)]%
        {Qian2023UpperBB}
\bibfield{author}{\bibinfo{person}{Xiaowei Qian}, \bibinfo{person}{Bingheng Li}, {and} \bibinfo{person}{Zhao Kang}.} \bibinfo{year}{2023}\natexlab{}.
\newblock \showarticletitle{Upper Bounding Barlow Twins: A Novel Filter for Multi-Relational Clustering}.
\newblock \bibinfo{journal}{\emph{ArXiv}}  \bibinfo{volume}{abs/2312.14066} (\bibinfo{year}{2023}).
\newblock


\bibitem[Raghavan et~al\mbox{.}(2007)]%
        {raghavan2007near_LabelProg}
\bibfield{author}{\bibinfo{person}{Usha~Nandini Raghavan}, \bibinfo{person}{R{\'e}ka Albert}, {and} \bibinfo{person}{Soundar Kumara}.} \bibinfo{year}{2007}\natexlab{}.
\newblock \showarticletitle{Near linear time algorithm to detect community structures in large-scale networks}.
\newblock \bibinfo{journal}{\emph{Physical Review E—Statistical, Nonlinear, and Soft Matter Physics}} \bibinfo{volume}{76}, \bibinfo{number}{3} (\bibinfo{year}{2007}), \bibinfo{pages}{036106}.
\newblock


\bibitem[Rahimi and Recht(2007)]%
        {rahimi2007random}
\bibfield{author}{\bibinfo{person}{Ali Rahimi} {and} \bibinfo{person}{Benjamin Recht}.} \bibinfo{year}{2007}\natexlab{}.
\newblock \showarticletitle{Random features for large-scale kernel machines}.
\newblock \bibinfo{journal}{\emph{Advances in neural information processing systems}}  \bibinfo{volume}{20} (\bibinfo{year}{2007}).
\newblock


\bibitem[Scarselli et~al\mbox{.}(2009)]%
        {Scarselli2009TheGN}
\bibfield{author}{\bibinfo{person}{Franco Scarselli}, \bibinfo{person}{Marco Gori}, \bibinfo{person}{Ah~Chung Tsoi}, \bibinfo{person}{Markus Hagenbuchner}, {and} \bibinfo{person}{Gabriele Monfardini}.} \bibinfo{year}{2009}\natexlab{}.
\newblock \showarticletitle{The Graph Neural Network Model}.
\newblock \bibinfo{journal}{\emph{IEEE Transactions on Neural Networks}}  \bibinfo{volume}{20} (\bibinfo{year}{2009}), \bibinfo{pages}{61--80}.
\newblock


\bibitem[Shawe-Taylor and Cristianini(2004)]%
        {shawe2004kernel}
\bibfield{author}{\bibinfo{person}{John Shawe-Taylor} {and} \bibinfo{person}{Nello Cristianini}.} \bibinfo{year}{2004}\natexlab{}.
\newblock \bibinfo{booktitle}{\emph{Kernel methods for pattern analysis}}.
\newblock \bibinfo{publisher}{Cambridge university press}.
\newblock


\bibitem[Shen et~al\mbox{.}(2024a)]%
        {Shen2024BalancedMG}
\bibfield{author}{\bibinfo{person}{Zhixiang Shen}, \bibinfo{person}{Haolan He}, {and} \bibinfo{person}{Zhao Kang}.} \bibinfo{year}{2024}\natexlab{a}.
\newblock \showarticletitle{Balanced Multi-Relational Graph Clustering}.
\newblock \bibinfo{journal}{\emph{ArXiv}}  \bibinfo{volume}{abs/2407.16863} (\bibinfo{year}{2024}).
\newblock


\bibitem[Shen et~al\mbox{.}(2024b)]%
        {Shen2024BeyondRI}
\bibfield{author}{\bibinfo{person}{Zhixiang Shen}, \bibinfo{person}{Shuo Wang}, {and} \bibinfo{person}{Zhao Kang}.} \bibinfo{year}{2024}\natexlab{b}.
\newblock \showarticletitle{Beyond Redundancy: Information-aware Unsupervised Multiplex Graph Structure Learning}.
\newblock \bibinfo{journal}{\emph{ArXiv}}  \bibinfo{volume}{abs/2409.17386} (\bibinfo{year}{2024}).
\newblock


\bibitem[Shi et~al\mbox{.}(2022)]%
        {Shi2022RHINERS}
\bibfield{author}{\bibinfo{person}{Chuan Shi}, \bibinfo{person}{Yuanfu Lu}, \bibinfo{person}{Linmei Hu}, \bibinfo{person}{Zhiyuan Liu}, {and} \bibinfo{person}{Huadong Ma}.} \bibinfo{year}{2022}\natexlab{}.
\newblock \showarticletitle{RHINE: Relation Structure-Aware Heterogeneous Information Network Embedding}.
\newblock \bibinfo{journal}{\emph{IEEE Transactions on Knowledge and Data Engineering}}  \bibinfo{volume}{34} (\bibinfo{year}{2022}), \bibinfo{pages}{433--447}.
\newblock


\bibitem[Shi and Malik(2000)]%
        {shi2000normalized}
\bibfield{author}{\bibinfo{person}{Jianbo Shi} {and} \bibinfo{person}{Jitendra Malik}.} \bibinfo{year}{2000}\natexlab{}.
\newblock \showarticletitle{Normalized cuts and image segmentation}.
\newblock \bibinfo{journal}{\emph{TPAMI}} \bibinfo{volume}{22}, \bibinfo{number}{8} (\bibinfo{year}{2000}), \bibinfo{pages}{888--905}.
\newblock


\bibitem[Sinkhorn and Knopp(1967)]%
        {sinkhorn1967concerning}
\bibfield{author}{\bibinfo{person}{Richard Sinkhorn} {and} \bibinfo{person}{Paul Knopp}.} \bibinfo{year}{1967}\natexlab{}.
\newblock \showarticletitle{Concerning nonnegative matrices and doubly stochastic matrices}.
\newblock \bibinfo{journal}{\emph{Pacific J. Math.}} \bibinfo{volume}{21}, \bibinfo{number}{2} (\bibinfo{year}{1967}), \bibinfo{pages}{343--348}.
\newblock


\bibitem[Son et~al\mbox{.}(2016)]%
        {Son2016AdaptiveSC}
\bibfield{author}{\bibinfo{person}{Jeong-Woo Son}, \bibinfo{person}{Junekey Jeon}, \bibinfo{person}{Sang-Yun Lee}, {and} \bibinfo{person}{Sun-Joong Kim}.} \bibinfo{year}{2016}\natexlab{}.
\newblock \showarticletitle{Adaptive spectral co-clustering for multiview data}.
\newblock \bibinfo{journal}{\emph{2016 18th International Conference on Advanced Communication Technology (ICACT)}} (\bibinfo{year}{2016}), \bibinfo{pages}{447--450}.
\newblock


\bibitem[Strehl and Ghosh(2003)]%
        {Strehl2003ClusterE}
\bibfield{author}{\bibinfo{person}{Alexander Strehl} {and} \bibinfo{person}{Joydeep Ghosh}.} \bibinfo{year}{2003}\natexlab{}.
\newblock \showarticletitle{Cluster ensembles --- a knowledge reuse framework for combining multiple partitions}.
\newblock \bibinfo{journal}{\emph{Journal of Machine Learning Research}}  \bibinfo{volume}{3} (\bibinfo{year}{2003}), \bibinfo{pages}{583--617}.
\newblock


\bibitem[Tan et~al\mbox{.}(2023)]%
        {tan2023metric}
\bibfield{author}{\bibinfo{person}{Yuze Tan}, \bibinfo{person}{Yixi Liu}, \bibinfo{person}{Hongjie Wu}, \bibinfo{person}{Jiancheng Lv}, {and} \bibinfo{person}{Shudong Huang}.} \bibinfo{year}{2023}\natexlab{}.
\newblock \showarticletitle{Metric multi-view graph clustering}. In \bibinfo{booktitle}{\emph{AAAI}}, Vol.~\bibinfo{volume}{37}. \bibinfo{pages}{9962--9970}.
\newblock


\bibitem[Tang et~al\mbox{.}(2023)]%
        {Tang2023UnifiedOM}
\bibfield{author}{\bibinfo{person}{Chang Tang}, \bibinfo{person}{Zhenglai Li}, \bibinfo{person}{J. Wang}, \bibinfo{person}{Xinwang Liu}, \bibinfo{person}{Wei Zhang}, {and} \bibinfo{person}{En Zhu}.} \bibinfo{year}{2023}\natexlab{}.
\newblock \showarticletitle{Unified One-Step Multi-View Spectral Clustering}.
\newblock \bibinfo{journal}{\emph{IEEE Transactions on Knowledge and Data Engineering}}  \bibinfo{volume}{35} (\bibinfo{year}{2023}), \bibinfo{pages}{6449--6460}.
\newblock


\bibitem[Tang et~al\mbox{.}(2025)]%
        {tang2025fraud}
\bibfield{author}{\bibinfo{person}{Jiangnan Tang}, \bibinfo{person}{Huanhuan Gu}, \bibinfo{person}{Darko~B Vukovi{\'c}}, \bibinfo{person}{Guandong Xu}, \bibinfo{person}{Youquan Wang}, \bibinfo{person}{Haicheng Tao}, {and} \bibinfo{person}{Jie Cao}.} \bibinfo{year}{2025}\natexlab{}.
\newblock \showarticletitle{Fraud detection in multi-relation graph: Contrastive Learning on Feature and Structural Levels}.
\newblock \bibinfo{journal}{\emph{Neurocomputing}} (\bibinfo{year}{2025}), \bibinfo{pages}{130063}.
\newblock


\bibitem[Tang et~al\mbox{.}(2012)]%
        {tang2012community}
\bibfield{author}{\bibinfo{person}{Lei Tang}, \bibinfo{person}{Xufei Wang}, {and} \bibinfo{person}{Huan Liu}.} \bibinfo{year}{2012}\natexlab{}.
\newblock \showarticletitle{Community detection via heterogeneous interaction analysis}.
\newblock \bibinfo{journal}{\emph{DMKD}} (\bibinfo{year}{2012}), \bibinfo{pages}{1--33}.
\newblock


\bibitem[Tang et~al\mbox{.}(2009)]%
        {Tang2009ClusteringWM}
\bibfield{author}{\bibinfo{person}{Wei Tang}, \bibinfo{person}{Zhengdong Lu}, {and} \bibinfo{person}{Inderjit~S. Dhillon}.} \bibinfo{year}{2009}\natexlab{}.
\newblock \showarticletitle{Clustering with Multiple Graphs}.
\newblock \bibinfo{journal}{\emph{ICDM}} (\bibinfo{year}{2009}), \bibinfo{pages}{1016--1021}.
\newblock


\bibitem[Tsitsulin et~al\mbox{.}(2023)]%
        {tsitsulin2023graph}
\bibfield{author}{\bibinfo{person}{Anton Tsitsulin}, \bibinfo{person}{John Palowitch}, \bibinfo{person}{Bryan Perozzi}, {and} \bibinfo{person}{Emmanuel M{\"u}ller}.} \bibinfo{year}{2023}\natexlab{}.
\newblock \showarticletitle{Graph clustering with graph neural networks}.
\newblock \bibinfo{journal}{\emph{Journal of Machine Learning Research}} \bibinfo{volume}{24}, \bibinfo{number}{127} (\bibinfo{year}{2023}), \bibinfo{pages}{1--21}.
\newblock


\bibitem[Von~Luxburg(2007)]%
        {von2007tutorial}
\bibfield{author}{\bibinfo{person}{Ulrike Von~Luxburg}.} \bibinfo{year}{2007}\natexlab{}.
\newblock \showarticletitle{A tutorial on spectral clustering}.
\newblock \bibinfo{journal}{\emph{Statistics and computing}}  \bibinfo{volume}{17} (\bibinfo{year}{2007}), \bibinfo{pages}{395--416}.
\newblock


\bibitem[Wagner and Wagner(1993)]%
        {wagner1993between}
\bibfield{author}{\bibinfo{person}{Dorothea Wagner} {and} \bibinfo{person}{Frank Wagner}.} \bibinfo{year}{1993}\natexlab{}.
\newblock \showarticletitle{Between min cut and graph bisection}. In \bibinfo{booktitle}{\emph{Mathematical Foundations of Computer Science 1993: 18th International Symposium, MFCS'93 Gda{\'n}sk, Poland, August 30--September 3, 1993 Proceedings 18}}. Springer, \bibinfo{pages}{744--750}.
\newblock


\bibitem[Wang et~al\mbox{.}(2019b)]%
        {wang2019attributed}
\bibfield{author}{\bibinfo{person}{Chun Wang}, \bibinfo{person}{Shirui Pan}, \bibinfo{person}{Ruiqi Hu}, \bibinfo{person}{Guodong Long}, \bibinfo{person}{Jing Jiang}, {and} \bibinfo{person}{Chengqi Zhang}.} \bibinfo{year}{2019}\natexlab{b}.
\newblock \showarticletitle{Attributed graph clustering: a deep attentional embedding approach}. In \bibinfo{booktitle}{\emph{IJCAI}}. \bibinfo{pages}{3670--3676}.
\newblock


\bibitem[Wang et~al\mbox{.}(2024)]%
        {wang2024multi}
\bibfield{author}{\bibinfo{person}{Chenxu Wang}, \bibinfo{person}{Mengqin Wang}, \bibinfo{person}{Xiaoguang Wang}, \bibinfo{person}{Luyue Zhang}, {and} \bibinfo{person}{Yi Long}.} \bibinfo{year}{2024}\natexlab{}.
\newblock \showarticletitle{Multi-Relational Graph Representation Learning for Financial Statement Fraud Detection}.
\newblock \bibinfo{journal}{\emph{Big Data Mining and Analytics}} \bibinfo{volume}{7}, \bibinfo{number}{3} (\bibinfo{year}{2024}), \bibinfo{pages}{920--941}.
\newblock


\bibitem[Wang et~al\mbox{.}(2017)]%
        {Wang2017ExclusivityConsistencyRM}
\bibfield{author}{\bibinfo{person}{Xiaobo Wang}, \bibinfo{person}{Xiaojie Guo}, \bibinfo{person}{Zhen Lei}, \bibinfo{person}{Changqing Zhang}, {and} \bibinfo{person}{S. Li}.} \bibinfo{year}{2017}\natexlab{}.
\newblock \showarticletitle{Exclusivity-Consistency Regularized Multi-view Subspace Clustering}.
\newblock \bibinfo{journal}{\emph{CVPR}} (\bibinfo{year}{2017}), \bibinfo{pages}{1--9}.
\newblock


\bibitem[Wang et~al\mbox{.}(2019a)]%
        {Wang2019HeterogeneousGA}
\bibfield{author}{\bibinfo{person}{Xiao Wang}, \bibinfo{person}{Houye Ji}, \bibinfo{person}{Chuan Shi}, \bibinfo{person}{Bai Wang}, \bibinfo{person}{Peng Cui}, \bibinfo{person}{Philip~S. Yu}, {and} \bibinfo{person}{Yanfang Ye}.} \bibinfo{year}{2019}\natexlab{a}.
\newblock \showarticletitle{Heterogeneous Graph Attention Network}.
\newblock \bibinfo{journal}{\emph{The World Wide Web Conference}} (\bibinfo{year}{2019}).
\newblock


\bibitem[Wei et~al\mbox{.}(2014)]%
        {Wei2014LearningAM}
\bibfield{author}{\bibinfo{person}{Yunchao Wei}, \bibinfo{person}{Yao Zhao}, \bibinfo{person}{Zhenfeng Zhu}, \bibinfo{person}{Yanhui Xiao}, {and} \bibinfo{person}{Shikui Wei}.} \bibinfo{year}{2014}\natexlab{}.
\newblock \showarticletitle{Learning a mid-level feature space for cross-media regularization}.
\newblock \bibinfo{journal}{\emph{ICME}} (\bibinfo{year}{2014}), \bibinfo{pages}{1--6}.
\newblock


\bibitem[White et~al\mbox{.}(2012)]%
        {White2012ConvexMS}
\bibfield{author}{\bibinfo{person}{Martha White}, \bibinfo{person}{Yaoliang Yu}, \bibinfo{person}{Xinhua Zhang}, {and} \bibinfo{person}{Dale Schuurmans}.} \bibinfo{year}{2012}\natexlab{}.
\newblock \showarticletitle{Convex Multi-view Subspace Learning}. In \bibinfo{booktitle}{\emph{Neural Information Processing Systems}}.
\newblock


\bibitem[Xia et~al\mbox{.}(2023)]%
        {Xia2023RobustCM}
\bibfield{author}{\bibinfo{person}{Hui Xia}, \bibinfo{person}{Shu shu Shao}, \bibinfo{person}{Chun qiang Hu}, \bibinfo{person}{Rui Zhang}, \bibinfo{person}{Tie Qiu}, {and} \bibinfo{person}{Fu Xiao}.} \bibinfo{year}{2023}\natexlab{}.
\newblock \showarticletitle{Robust Clustering Model Based on Attention Mechanism and Graph Convolutional Network}.
\newblock \bibinfo{journal}{\emph{TKDE}}  \bibinfo{volume}{35} (\bibinfo{year}{2023}), \bibinfo{pages}{5203--5215}.
\newblock


\bibitem[Xia et~al\mbox{.}(2021)]%
        {Xia2021MultiviewGE}
\bibfield{author}{\bibinfo{person}{Wei Xia}, \bibinfo{person}{Sen Wang}, \bibinfo{person}{Ming Yang}, \bibinfo{person}{Quanxue Gao}, \bibinfo{person}{Jungong Han}, {and} \bibinfo{person}{Xinbo Gao}.} \bibinfo{year}{2021}\natexlab{}.
\newblock \showarticletitle{Multi-view graph embedding clustering network: Joint self-supervision and block diagonal representation}.
\newblock \bibinfo{journal}{\emph{Neural networks : the official journal of the International Neural Network Society}}  \bibinfo{volume}{145} (\bibinfo{year}{2021}), \bibinfo{pages}{1--9}.
\newblock


\bibitem[Xie et~al\mbox{.}(2025)]%
        {xie2025diffusion}
\bibfield{author}{\bibinfo{person}{Kun Xie}, \bibinfo{person}{Renchi Yang}, {and} \bibinfo{person}{Sibo Wang}.} \bibinfo{year}{2025}\natexlab{}.
\newblock \showarticletitle{Diffusion-based Graph-agnostic Clustering}. In \bibinfo{booktitle}{\emph{TheWebConf}}. \bibinfo{pages}{1353--1364}.
\newblock


\bibitem[Yang and Shi(2024)]%
        {yang2024efficient}
\bibfield{author}{\bibinfo{person}{Renchi Yang} {and} \bibinfo{person}{Jieming Shi}.} \bibinfo{year}{2024}\natexlab{}.
\newblock \showarticletitle{Efficient High-Quality Clustering for Large Bipartite Graphs}.
\newblock \bibinfo{journal}{\emph{SIGMOD}} \bibinfo{volume}{2}, \bibinfo{number}{1} (\bibinfo{year}{2024}), \bibinfo{pages}{1--27}.
\newblock


\bibitem[Yang et~al\mbox{.}(2023b)]%
        {yang2023pane}
\bibfield{author}{\bibinfo{person}{Renchi Yang}, \bibinfo{person}{Jieming Shi}, \bibinfo{person}{Xiaokui Xiao}, \bibinfo{person}{Yin Yang}, \bibinfo{person}{Sourav~S Bhowmick}, {and} \bibinfo{person}{Juncheng Liu}.} \bibinfo{year}{2023}\natexlab{b}.
\newblock \showarticletitle{PANE: scalable and effective attributed network embedding}.
\newblock \bibinfo{journal}{\emph{VLDBJ}} \bibinfo{volume}{32}, \bibinfo{number}{6} (\bibinfo{year}{2023}), \bibinfo{pages}{1237--1262}.
\newblock


\bibitem[Yang et~al\mbox{.}(2020)]%
        {yang2020scaling}
\bibfield{author}{\bibinfo{person}{Renchi Yang}, \bibinfo{person}{Jieming Shi}, \bibinfo{person}{Xiaokui Xiao}, \bibinfo{person}{Yin Yang}, \bibinfo{person}{Juncheng Liu}, \bibinfo{person}{Sourav~S Bhowmick}, {et~al\mbox{.}}} \bibinfo{year}{2020}\natexlab{}.
\newblock \showarticletitle{Scaling attributed network embedding to massive graphs}.
\newblock \bibinfo{journal}{\emph{VLDB}} \bibinfo{volume}{14}, \bibinfo{number}{1} (\bibinfo{year}{2020}), \bibinfo{pages}{37--49}.
\newblock


\bibitem[Yang et~al\mbox{.}(2021)]%
        {yang2021effective}
\bibfield{author}{\bibinfo{person}{Renchi Yang}, \bibinfo{person}{Jieming Shi}, \bibinfo{person}{Yin Yang}, \bibinfo{person}{Keke Huang}, \bibinfo{person}{Shiqi Zhang}, {and} \bibinfo{person}{Xiaokui Xiao}.} \bibinfo{year}{2021}\natexlab{}.
\newblock \showarticletitle{Effective and scalable clustering on massive attributed graphs}. In \bibinfo{booktitle}{\emph{TheWebConf}}. \bibinfo{pages}{3675--3687}.
\newblock


\bibitem[Yang et~al\mbox{.}(2024)]%
        {yang2024effective}
\bibfield{author}{\bibinfo{person}{Renchi Yang}, \bibinfo{person}{Yidu Wu}, \bibinfo{person}{Xiaoyang Lin}, \bibinfo{person}{Qichen Wang}, \bibinfo{person}{Tsz~Nam Chan}, {and} \bibinfo{person}{Jieming Shi}.} \bibinfo{year}{2024}\natexlab{}.
\newblock \showarticletitle{Effective Clustering on Large Attributed Bipartite Graphs}. In \bibinfo{booktitle}{\emph{KDD}}. \bibinfo{pages}{3782--3793}.
\newblock


\bibitem[Yang et~al\mbox{.}(2023a)]%
        {Yang2023ClusterguidedCG}
\bibfield{author}{\bibinfo{person}{Xihong Yang}, \bibinfo{person}{Yue Liu}, \bibinfo{person}{Sihang Zhou}, \bibinfo{person}{Siwei Wang}, \bibinfo{person}{Wenxuan Tu}, \bibinfo{person}{Qun Zheng}, \bibinfo{person}{Xinwang Liu}, \bibinfo{person}{Liming Fang}, {and} \bibinfo{person}{En Zhu}.} \bibinfo{year}{2023}\natexlab{a}.
\newblock \showarticletitle{Cluster-guided Contrastive Graph Clustering Network}.
\newblock \bibinfo{journal}{\emph{ArXiv}}  \bibinfo{volume}{abs/2301.01098} (\bibinfo{year}{2023}).
\newblock


\bibitem[Yang et~al\mbox{.}(2014)]%
        {yang2014recommendation}
\bibfield{author}{\bibinfo{person}{Zaihan Yang}, \bibinfo{person}{Dawei Yin}, {and} \bibinfo{person}{Brian~D Davison}.} \bibinfo{year}{2014}\natexlab{}.
\newblock \showarticletitle{Recommendation in academia: A joint multi-relational model}. In \bibinfo{booktitle}{\emph{ASONAM 2014}}. IEEE, \bibinfo{pages}{566--571}.
\newblock


\bibitem[Yu et~al\mbox{.}(2016)]%
        {yu2016orthogonal}
\bibfield{author}{\bibinfo{person}{Felix Xinnan~X Yu}, \bibinfo{person}{Ananda~Theertha Suresh}, \bibinfo{person}{Krzysztof~M Choromanski}, \bibinfo{person}{Daniel~N Holtmann-Rice}, {and} \bibinfo{person}{Sanjiv Kumar}.} \bibinfo{year}{2016}\natexlab{}.
\newblock \showarticletitle{Orthogonal random features}.
\newblock \bibinfo{journal}{\emph{Advances in neural information processing systems}}  \bibinfo{volume}{29} (\bibinfo{year}{2016}).
\newblock


\bibitem[Yu and Shi(2003)]%
        {shi2003multiclass}
\bibfield{author}{\bibinfo{person}{Stella~X. Yu} {and} \bibinfo{person}{Jianbo Shi}.} \bibinfo{year}{2003}\natexlab{}.
\newblock \showarticletitle{Multiclass Spectral Clustering}. In \bibinfo{booktitle}{\emph{{ICCV}}}. \bibinfo{pages}{313--319}.
\newblock


\bibitem[Yue et~al\mbox{.}(2019)]%
        {Gebnyue2019}
\bibfield{author}{\bibinfo{person}{Xiang Yue}, \bibinfo{person}{Zhen Wang}, \bibinfo{person}{Jingong Huang}, \bibinfo{person}{Srinivasan Parthasarathy}, \bibinfo{person}{Soheil Moosavinasab}, \bibinfo{person}{Yungui Huang}, \bibinfo{person}{Simon~M Lin}, \bibinfo{person}{Wen Zhang}, \bibinfo{person}{Ping Zhang}, {and} \bibinfo{person}{Huan Sun}.} \bibinfo{year}{2019}\natexlab{}.
\newblock \showarticletitle{Graph embedding on biomedical networks: methods, applications and evaluations}.
\newblock \bibinfo{journal}{\emph{Bioinformatics}} \bibinfo{volume}{36}, \bibinfo{number}{4} (\bibinfo{date}{10} \bibinfo{year}{2019}), \bibinfo{pages}{1241--1251}.
\newblock
\showISSN{1367-4803}


\bibitem[Zhang et~al\mbox{.}(2019)]%
        {Zhang2019OAGTL}
\bibfield{author}{\bibinfo{person}{Fanjin Zhang}, \bibinfo{person}{Xiao Liu}, \bibinfo{person}{Jie Tang}, \bibinfo{person}{Yuxiao Dong}, \bibinfo{person}{Peiran Yao}, \bibinfo{person}{Jie Zhang}, \bibinfo{person}{Xiaotao Gu}, \bibinfo{person}{Yan Wang}, \bibinfo{person}{Bin Shao}, \bibinfo{person}{Rui Li}, {and} \bibinfo{person}{Kuansan Wang}.} \bibinfo{year}{2019}\natexlab{}.
\newblock \showarticletitle{OAG: Toward Linking Large-scale Heterogeneous Entity Graphs}.
\newblock \bibinfo{journal}{\emph{KDD}} (\bibinfo{year}{2019}).
\newblock


\bibitem[Zhang et~al\mbox{.}(1996)]%
        {zhangbirch96}
\bibfield{author}{\bibinfo{person}{Tian Zhang}, \bibinfo{person}{Raghu Ramakrishnan}, {and} \bibinfo{person}{Miron Livny}.} \bibinfo{year}{1996}\natexlab{}.
\newblock \showarticletitle{BIRCH: an efficient data clustering method for very large databases}.
\newblock \bibinfo{journal}{\emph{SIGMOD Rec.}} \bibinfo{volume}{25}, \bibinfo{number}{2} (\bibinfo{year}{1996}).
\newblock


\bibitem[Zhao et~al\mbox{.}(2021)]%
        {Zhao2021GraphDC}
\bibfield{author}{\bibinfo{person}{Han Zhao}, \bibinfo{person}{Xu Yang}, \bibinfo{person}{Zhenru Wang}, \bibinfo{person}{Erkun Yang}, {and} \bibinfo{person}{Cheng Deng}.} \bibinfo{year}{2021}\natexlab{}.
\newblock \showarticletitle{Graph Debiased Contrastive Learning with Joint Representation Clustering}. In \bibinfo{booktitle}{\emph{International Joint Conference on Artificial Intelligence}}.
\newblock


\bibitem[Zhao et~al\mbox{.}(2020)]%
        {ZhaoWSLY20}
\bibfield{author}{\bibinfo{person}{Jianan Zhao}, \bibinfo{person}{Xiao Wang}, \bibinfo{person}{Chuan Shi}, \bibinfo{person}{Zekuan Liu}, {and} \bibinfo{person}{Yanfang Ye}.} \bibinfo{year}{2020}\natexlab{}.
\newblock \showarticletitle{Network Schema Preserving Heterogeneous Information Network Embedding.}. In \bibinfo{booktitle}{\emph{IJCAI}}. \bibinfo{pages}{1366--1372}.
\newblock
\newblock
\shownote{Scheduled for July 2020, Yokohama, Japan, postponed due to the Corona pandemic.}.


\bibitem[Zhao et~al\mbox{.}(2022)]%
        {Zhao2022HierarchicalAN}
\bibfield{author}{\bibinfo{person}{Qiqi Zhao}, \bibinfo{person}{Huifang Ma}, \bibinfo{person}{Lijun Guo}, {and} \bibinfo{person}{Zhixin Li}.} \bibinfo{year}{2022}\natexlab{}.
\newblock \showarticletitle{Hierarchical attention network for attributed community detection of joint representation}.
\newblock \bibinfo{journal}{\emph{Neural Computing and Applications}}  \bibinfo{volume}{34} (\bibinfo{year}{2022}), \bibinfo{pages}{5587 -- 5601}.
\newblock


\bibitem[Zheng et~al\mbox{.}(2025)]%
        {zheng2025adaptive}
\bibfield{author}{\bibinfo{person}{Haoran Zheng}, \bibinfo{person}{Renchi Yang}, {and} \bibinfo{person}{Jianliang Xu}.} \bibinfo{year}{2025}\natexlab{}.
\newblock \showarticletitle{Adaptive Local Clustering over Attributed Graphs}. In \bibinfo{booktitle}{\emph{ICDE}}. IEEE Computer Society, \bibinfo{pages}{2052--2065}.
\newblock


\bibitem[Zhou and Burges(2007)]%
        {Zhou2007SpectralCA}
\bibfield{author}{\bibinfo{person}{Dengyong Zhou} {and} \bibinfo{person}{Christopher J.~C. Burges}.} \bibinfo{year}{2007}\natexlab{}.
\newblock \showarticletitle{Spectral clustering and transductive learning with multiple views}. In \bibinfo{booktitle}{\emph{ICML}}.
\newblock


\bibitem[Zhou and Sch{\"o}lkopf(2005)]%
        {zhou2005regularization}
\bibfield{author}{\bibinfo{person}{Dengyong Zhou} {and} \bibinfo{person}{Bernhard Sch{\"o}lkopf}.} \bibinfo{year}{2005}\natexlab{}.
\newblock \showarticletitle{Regularization on discrete spaces}. In \bibinfo{booktitle}{\emph{Joint Pattern Recognition Symposium}}. Springer, \bibinfo{pages}{361--368}.
\newblock


\bibitem[Zhu and Koniusz(2021)]%
        {Zhu2021SimpleSG}
\bibfield{author}{\bibinfo{person}{Hao Zhu} {and} \bibinfo{person}{Piotr Koniusz}.} \bibinfo{year}{2021}\natexlab{}.
\newblock \showarticletitle{Simple Spectral Graph Convolution}. In \bibinfo{booktitle}{\emph{ICLR}}.
\newblock


\bibitem[Zhuang et~al\mbox{.}(2024)]%
        {Zhuang2024EnhancingMG}
\bibfield{author}{\bibinfo{person}{Shuman Zhuang}, \bibinfo{person}{Sujia Huang}, \bibinfo{person}{Wei Huang}, \bibinfo{person}{Yuhong Chen}, \bibinfo{person}{Zhihao Wu}, {and} \bibinfo{person}{Ximeng Liu}.} \bibinfo{year}{2024}\natexlab{}.
\newblock \showarticletitle{Enhancing Multi-view Graph Neural Network with Cross-view Confluent Message Passing}. In \bibinfo{booktitle}{\emph{MM}}.
\newblock


\end{thebibliography}

\end{document}